\documentclass[final,12pt]{clear2025} % Include author names

\usepackage{lmodern}
\usepackage{mathtools} 
\usepackage{relsize}
\usepackage{hyperref}
\usepackage{url}
\usepackage{tikz}
\usepackage{setspace} % setstretch line spacing
\usepackage{dsfont} % /mathds for the indicator function
\usepackage{booktabs} % table separators
\usepackage{enumerate}
\usepackage{csquotes}
\usepackage[noabbrev, capitalise]{cleveref}
\usepackage{comment}
\usepackage{customboxes}  

\newcommand{\R}{\mathbb{R}} % real numbers
\newcommand{\N}{\mathbb{N}} % real numbers
\newcommand{\Parents}{\mathbf{\operatorname{PA}}} % Parents
\newcommand{\Child}{\mathbf{\operatorname{CH}}} % Children 
\newcommand{\Anc}{\operatorname{AN}} % Ancestors
\newcommand{\Desc}{\operatorname{DE}} % Descendants
\newcommand{\Nondesc}{\operatorname{ND}} % Non-Descendants 
\newcommand{\indep}{\,\rotatebox[origin=c]{90}{$\models$}\,} % d-sep 
\newcommand{\set}[1]{\{#1\}}
\newcommand{\notindep}{\not\hspace{-1.8mm}\indep}
\newcommand{\dindep}[1]{\indep^d_{#1}}
\newcommand{\notdindep}[1]{\not\hspace{-1.8mm}\indep^d_{#1}}
\newcommand{\mindep}[1]{\indep^m_{#1}}

\newcommand{\district}{\operatorname{D}}
\newcommand{\mean}{\mathbb{E}}

\usepackage{pifont}% http://ctan.org/pkg/pifont
\newcommand{\cmark}{\ding{51}}%
\newcommand{\xmark}{\ding{55}}%

\newcommand{\cM}{{\cal M}}

\newcommand{\cG}{{\cal G}}

\newcommand{\cP}{{\cal P}}

\newcommand{\bK}{{\textbf K}}
\newcommand{\bG}{{\textbf G}}
\newcommand{\bv}{{\textbf v}}

\newcommand{\lchead}{\circ\hspace{-1.4mm}}
\newcommand{\rchead}{\hspace{-1.4mm}\circ}

\newcommand{\Gx}{{\cal G}}

\newtheorem{assumption}{Assumption}
\newtheorem{condition}{Condition}

\newcommand{\FM}[1]{\textbf{\color{magenta}FM:}{~\color{purple}#1}}
\newcommand{\FL}[1]{}
\newcommand{\PF}[1]{}

\title{Score matching through the roof: linear, nonlinear, and latent variables causal discovery}
\usepackage{times}
\clearauthor{%
 \Name{Francesco Montagna}\thanks{Shared first co-author} \Email{francesco.montagna@edu.unige.it}\\
 \addr MaLGa, University of Genoa 
 \AND
 \Name{Philipp M. Faller}$^{*}\thanks{Part of this work was done while Philipp M. Faller was an intern at Amazon Research Tübingen.}$ \Email{philipp.faller@partner.kit.edu}\\
 \addr Karlsruhe Institute of Technology
\AND
 \Name{Patrick Bl\"{o}baum} \Email{bloebp@amazon.com}\\
 \addr Amazon
 \AND
 \Name{Elke Kirschbaum} \Email{elkeki@amazon.com}\\
 \addr Amazon 
 \AND
 \Name{Francesco Locatello} \Email{francesco.locatello@ista.ac.at}\\
 \addr Institute of Science and Technology Austria (ISTA)%
}

\begin{document}

\maketitle

\begin{abstract}%
Causal discovery from observational data holds great promise, but existing methods rely on strong assumptions about the underlying causal structure, often requiring full observability of all relevant variables. 
We tackle these challenges by leveraging the score function $\nabla \log p(X)$ of observed variables for causal discovery and propose the following contributions. 
First, we fine-tune the existing identifiability results with the score on additive noise models, showing that their assumption of nonlinearity of the causal mechanisms is not necessary. 
Second, we establish conditions for inferring causal relations from the score even in the presence of hidden variables; 
this result is two-faced: we demonstrate the score’s potential to infer the equivalence class of causal graphs with hidden
variables (while previous results are restricted to the fully observable setting), and we provide sufficient conditions for identifying direct causes in latent variable models. Building on these insights, we propose a flexible algorithm suited for causal discovery on linear, nonlinear, and latent variable models, which we empirically validate.%
\begin{comment}
Causal discovery from observational data holds great promise, but existing methods rely on strong assumptions about the underlying causal mechanisms, often requiring full observability of all relevant variables. 
We tackle these challenges by leveraging the score function $\nabla \log p(X)$ of observed variables for causal discovery and propose the following set of contributions. First, we generalize the existing results of identifiability with the score to additive noise models with minimal requirements on the causal mechanisms. Second, we establish the conditions for inferring causal relations from the score even in the presence of hidden variables; this result is two-faced: we demonstrate the score's potential as an alternative to conditional independence tests for inferring the equivalence class of causal graphs with hidden variables, and we provide the necessary conditions for identifying direct causes in latent variable models. Building on these insights, we propose an algorithm for efficient causal discovery across linear, nonlinear, and latent variable models, which we validate through synthetic and real-world data experiments.
\end{comment}
\end{abstract}


\section{Introduction}\label{sec:intro}
The inference of causal effects from observations holds the potential for great impact arguably in any domain of science, where it is crucial to be able to answer interventional and counterfactual queries from observational data \citep{peters2017elements, pearl2009causality, spirtes10_intro}. Existing causal discovery methods can be categorized based on the information they can extract from the data and the assumptions they rely on \citep{glymour2019_review}. Traditional causal discovery methods (e.g. PC, GES \citep{spirtes2000_cauastion, chickering03_ges}) are general in their applicability but limited to the inference of an equivalence class. Additional assumptions on the structural equations generating effects from the cause are, in fact, imposed to ensure the identifiability of a causal order \citep{ shimizu2006_icalingam, hoyer08_anm, peters_2014_identifiability, zhang2009PNL}. As a consequence, existing methods for causal discovery require specialized and often untestable assumptions, preventing their application to real-world scenarios.

\looseness-1Further, the majority of existing approaches are hindered by the assumption that all relevant causes of the measured data are observed, which facilitates the interpretation of associations in the data as causal relationships. Despite the convenience of this hypothesis, it is often not met in practice, and the solutions relaxing this requirement face substantial limitations. The FCI algorithm \citep{spirtes2001_fci} can only return an equivalence class from the data. Appealing to additional restrictions ensures the identifiability of some direct causal effects in the presence of latent variables: RCD \citep{maeda2020_rcd} relies on the linear non-Gaussian additive noise model, whereas CAM-UV \citep{maeda21_causal} requires nonlinear additive mechanisms. Nevertheless, the strict conditions on the structural equations hold back their applicability to more general settings. 

\looseness-1Our paper tackles these challenges and can be put in the context of a recent line of academic research that derives a connection between the score function $\nabla \log p(X)$ and the causal graph underlying the data-generating process \citep{ghoshal2018learning, rolland22_score, montagna2023shortcuts, montagna23_das, montagna23_nogam, sanchez23diffan}. The use of the score for causal discovery is practically appealing, as it yields advantages in terms of scalability to high dimensional graphs \citep{montagna23_das} and guarantees of finite sample complexity bounds \citep{zhu2024sample}. Instead of imposing assumptions that ensure strong, though often impractical, theoretical guarantees, we organically demonstrate different levels of identifiability based on the strength of the modeling hypotheses, always relying on the score function to encode all the causal information in the data. Starting from results of \citet{spantini2018inference} and \citet{lin1997_factorizing}, we show how constraints on the Jacobian of the score $\nabla^2 \log p(X)$ can be used to characterize the Markov equivalence class of causal models with hidden variables. 
% \PF{I know we already discussed this last time. But reviewers complained about it and I'm also still not convinced that you can say this. If we talk about the population setting, (what we do in the theory part) you don't use a CI \emph{test} in the first place, only CI-criteria/characterizations. If we talk about finite samples, we also use a test (in fact, we use the word \enquote{test} w.r.t. to our method multiple times in the paper). Worse, we use the score like a test but have no guarantees about false-positive control. Therefore I would be cautious. How about \enquote{we show how constraints on the Jacobian of the score $\nabla^2 \log p(X)$ can be used as a criterion for conditional independence... }} 
Previous works exploit this connection in case of fully observable causal models \citep{montagna23_das, liu2024mixed}. Further, we prove that the score function identifies the causal direction of additive noise models, with minimal assumptions on the causal mechanisms. This extends the previous findings of \citet{montagna23_nogam}, limited by the assumption of nonlinearity of the causal effects. On these results, we build the main contributions of our work, enabling the identification of direct causal effects in hidden variable models. 

\textbf{Our main contributions} are as follows: \textit{(i)} We present conditions for the identifiability of direct causal effects with the score in the case of latent variables models. \textit{(ii)} We propose AdaScore (Adaptive Score-matching-based causal discovery), a flexible algorithm for causal discovery based on score matching estimation of $\nabla \log p(X)$ \citep{hyvrinen2005_scorematching}. Based on the user's belief about the plausibility of several modeling assumptions on the data, AdaScore can output a Markov equivalence class, a directed acyclic graph, or a mixed graph, accounting for the presence of unobserved variables. To the best of our knowledge, the broad class of causal models handled by our method is unmatched by other approaches in the literature. Our main contributions are presented in \cref{sec:unobservable_identifiability,sec:algo,sec:experiments}. The preliminary \cref{sec:score_ci,sec:observable_identifiability} summarise and generalize the existing theory connecting causal discovery and the score function, and Section \ref{sec:model} introduces the formalism of structural causal models, with and without latent variables.

%%%%%%%%%%%%%%%%%%%%%%%%%%%%%%%%%%%%%%
\section{Model definition}\label{sec:model}
In this section, we introduce the formalism of structural causal models (SCMs), for the cases with and without hidden variables.

\subsection{Causal model with observed variables}
Let $X$ be a set of random variables in $\R$ defined according to the set of structural equations
\begin{equation}\label{eq:scm}
    X_i \coloneqq f_i(X_{\Parents^\cG_i}, N_i), \hspace{2mm} \forall i=1,\ldots,k.
\end{equation}
$N_i \in \R$ are mutually independent random variables with strictly positive density, known as \textit{noise} or error terms. The function $f_i$ is the \textit{causal mechanism} mapping the set of \textit{direct causes} $X_{\Parents^\cG_i}$ of $X_i$ and the noise term $N_i$, to $X_i$'s value. A structural causal model (SCM) is defined as the tuple $(X, N, \mathcal{F}, \mathbb P_N)$, where $\mathcal{F} = (f_i)_{i=1}^k$ is the set of causal mechanisms, and $\mathbb P_N$ is the joint distribution with the density $p_N$ over the noise terms $N \in \R^k$.
 We define the \textit{causal graph} $\Gx$ as a directed acyclic graph (DAG) with nodes $ X = \{X_1,\ldots,X_k\}$, and the set of edges defined as $\{X_j \rightarrow X_i: X_j \in X_{\Parents^\cG_i} \}$, such that $\Parents^\cG_i$ are the indices of the parent nodes of $X_i$ in the graph $\Gx$. (In the remainder of the paper, we adopt the following notation: given a set of random variables $Y=\set{Y_1, \ldots, Y_n}$ and a set of indices $Z \subset \N$, then $Y_Z = \set{Y_i | i \in Z, Y_i \in Y}$.)
% where $\Parents^\cG_i\subseteq \{X_1, \dots, X_k\}$, $N_i$ are mutually independent random variables and $f_i: \R^{|\Parents^\cG_i| + 1} \rightarrow \R$ are measureable functions.  
% The functions $f_i$ are called \emph{causal mechanism} and map the \emph{direct causes } $\Parents^\cG_i$ and the noise $N_i$ to $X_i$'s value. A structural causal model (SCM) is defined as the tuple $(X, N, \mathcal{F}, P_N)$, where $\mathcal{F} = (f_i)_{i=1}^d$ is the set of causal mechanisms, and $P_N$ is the joint distribution function over the noise terms. 
% Then the distribution $P_N$ entails a distribution $P$ over all random variables.
% We assume further that there is always a density $p$ w.r.t. $P$ and that $p$ is is positive and its second-order partial derivatives are defined everywhere.
% We define the causal graph $\Gx$ with the set of nodes $ X = \{X_1,\ldots,X_k\}$, each node $i$ representing a random variable, and with the set of directed edges $\{X_j \rightarrow X_i: X_j \in X_{\Parents^\cG_i} \}$ (note that we overload $X$ to denote both random variables and set of random variables, as the distinction is clear from context) 

Under this model, the probability density of $X$ satisfies the \emph{Markov factorization} (e.g. \citet{peters2017elements} Proposition 6.31):
\begin{equation}\label{eq:markov_factorization}
    p(x) = \prod_{i=1}^k p(x_i | x_{\Parents^\cG_i}),
\end{equation}
where we adopt the convention of lowercase letters referring to realized random variables, and use $p$ to denote the density of different random objects, when the distinction is clear from the argument.
In the remainder of the paper, we assume that \textit{faithfulness} is satisfied \citep{pearl2009causality, uhler12_faithfulness} (Definition \ref{def:faithfulness} in the appendix). Together with the \textit{global Markov condition} (implied by \cref{eq:markov_factorization}, see e.g. \citet{peters2017elements} Proposition 6.22), this means that probabilistic and graphical statements of conditional independence are equivalent, such that for $\{X_i, X_j\}\subseteq X$ and $X_Z\subseteq X\setminus\{X_i, X_j\}$
\begin{equation}\label{eq:dsep_criterion}
X_i \indep X_j | X_Z \Longleftrightarrow X_i \dindep{\Gx} X_j | X_Z,
\end{equation} 
where $(\cdot \indep \cdot | \cdot)$ denotes probabilistic conditional independence of $X_i, X_j$ given $X_Z$, and $(\cdot \indep^d_{\mathcal{G}} \cdot | \cdot)$ is the notation for \textit{d-separation}, a criterion of conditional independence defined on the graph $\Gx$ (Definition \ref{def:msep} of the appendix).

% In general, several DAGs may entail the same set of d-separations: graphs sharing such common structure form a \textit{Markov equivalence class} (see Definition \ref{def:mec} in the appendix).
% By Equation \eqref{eq:dsep_criterion}, the Markov equivalence class of the DAG $\Gx$ can be identified with tests of conditional independence. 
The above model assumes that there aren't any unobserved causes of variables in $X$, other than the noise terms in $N$. As we are interested in distributions with potential hidden variables, we will now generalize our model to represent data-generating processes that may involve latent causes. 

% \paragraph{Definitions on graphs.} 
% \looseness-1As graphs play a central role in our work, \cref{app:useful_graph} provides a detailed overview of the fundamental notation and definitions that we rely on in the remainder of the paper. For the next section, we advise the reader to be comfortable with the notions of \textit{ancestors} (Definition \ref{def:ancestor}) and \textit{inducing paths} (Definition \ref{def:inducing_path}) in DAGs. 

\subsection{Causal model with unobserved variables}
Under the model \eqref{eq:scm}, we consider the case where the set of variables $X$ is partitioned into the disjoint subsets of \textit{observed} random variables $V = \set{V_1, \ldots, V_d}$ and \textit{unobserved} (or \textit{latent}) random variables $U = \set{U_{d+1}, \ldots, U_p}$. 
 % Without loss of generality, let $V_i = X_i, \forall i = 1, \ldots, d$ \FM{I think it's not necessary}. 
We assume that the following set of structural equations is satisfied:

\begin{align}
    V_i &\coloneqq f_i(V_{\Parents^\cG_i}, U^i, N_i), \hspace{2mm} &\forall i=1,\ldots,d, \label{eq:latent_scm_v}\\
    U_i &\coloneqq f_i(X_{\Parents^\cG_i}, N_i), \hspace{2mm} &\forall i=d+1,\ldots,p, \label{eq:latent_scm_u}
\end{align}
where $U^i$ stands for the set of unobserved parents of $V_i$, and $V_{\Parents^\cG_i} = \set{V_k| k \in \Parents_i^\Gx, V_k \in V}$ are the observed parents of $V_i$.
Some of the causal relations and the conditional independencies implied by the set of equations \eqref{eq:latent_scm_v} can be summarized in a graph obtained as a \textit{marginalization} of the DAG $\Gx$ onto the observable nodes $V$. For the next definition, we advise the reader to be comfortable with the notions of \textit{ancestors} (Definition \ref{def:ancestor}) and \textit{inducing paths} (Definition \ref{def:inducing_path}) in DAGs. 
\begin{definition}[Marginal graph, \citet{zhang2008causal}]\label{def:marginal_g}
    Let $X=V \cup U$, $V$ and $U$ disjoint, and $\Gx$ be a DAG over $X$. The following construction gives the \textit{marginal} graph $\cM^\Gx_V$, with nodes $V$ and edges found as follows:
    \begin{itemize}
        \item pair of nodes $V_i ,V_j$ are adjacent in the graph $\cM^\Gx_V$ if and only if there is an inducing path between them relative to $U$ in $\Gx$;
        \item for each pair of adjacent nodes $V_i, V_j$ in $\cM^\Gx_V$, orient the edge as $V_i \rightarrow V_j$ if $V_i$ is an ancestor of $V_j$ in $\Gx$, else orient it as $V_i \leftrightarrow V_j$.
    \end{itemize}
\looseness-1We define the map $\Gx \mapsto \cM^\Gx_V$  as the \textit{marginalization} of the DAG $\Gx$ onto $V$, the observable nodes.
\end{definition}
The graph resulting from the above construction is a maximal ancestral graph (MAG, Definition \ref{def:mag}), hence we will often refer to it as the \textit{marginal MAG} of $\cG$. Intuitively,  edges denote dependencies that cannot be removed by conditioning on any of the observed variables; in particular, if the edge is directed it denotes an ancestorship relation.

In the case of DAGs, d-separation encodes the probabilistic conditional independence relations between the variables of $X$ in the graph $\Gx$, as explicit by Equation \eqref{eq:dsep_criterion}. Such notion of graphical separation has a natural generalization to maximal ancestral graphs, known as \textit{m-separation} (Definition \ref{def:msep} of the appendix). \citet{zhang2008causal} shows that \textit{m-separation} and \textit{d-separation} are in fact equivalent (see Lemma \ref{lem:msep_dsep_equivalence} of the appendix), such that given $\set{V_i, V_j} \subset V$ and $V_Z \subset V\setminus\{V_i, V_j\}$ , the following holds:
\begin{equation}\label{eq:msep_dsep_equivalence}
    V_i \indep_{\Gx}^d V_j | V_Z \iff V_i \indep_{\cM^\Gx_V}^m V_j | V_Z,
\end{equation}
where $(\cdot \indep_{\cM^\Gx_V}^m \cdot \mid \cdot)$ denotes \textit{m-separation} relative to the graph $\cM^\Gx_V$. Just like with DAGs, MAGs that imply the same set of conditional independencies define an equivalence class.
Usually, the common structure of these graphs is represented by partial ancestral graphs (PAGs, Definition \ref{def:pag} of the appendix). We use $\cP_{\cM^\Gx_V}$ to denote the PAG relative to $\cM^\Gx_V$.
%
%
% \citet{evans2016_graphs} study the Markov properties of mDAGs. They introduce the notion of \textit{districts}, which allows obtaining a Markov factorization of the joint density of $V$ similar to \eqref{eq:markov_factorization}, in the case of a latent variables model.
% \begin{definition}[District]
%     A set $C \subseteq V$ is said to be \textit{bidirected-connected} if:
%     \begin{itemize}
%         \item $C = \set{V_i}$ for some $i=1,\ldots, d$; or
%         \item for each pair of nodes $V_i, V_j$ in $C$, there is a bidirected path $V_i \leftrightarrow V_{i+1} \leftrightarrow \ldots \leftrightarrow V_{j-1} \leftrightarrow V_j$.
%     \end{itemize}
%     A maximal bidirected-connected set of nodes is called a \textit{district} of $\cM_V^\Gx$.
% \end{definition}
% Given that districts form a partition of the nodes in $V$ (as shown in \cref{app:district}), we can define the map $D: V_i \mapsto \district_k$ such that $V_i \in \district_k$, mapping a node to the district it belongs to. Assuming that $V = \bigcup_{i=1}^m \district_i$ where $\district_i$ denotes a district of the mDAG, the following factorization holds (Proposition 5.4 \citet{evans2016_graphs}):
% \begin{equation}
%     p(v) = \prod_{k=1}^m q_k,
% \end{equation}
% where $\Parents^\cG_{\district_i} \coloneqq \bigcup_{V_k \in \district_i} \Parents^\cG_k$, and the factors are defined as:
% $$
% q_k = \sum
% $$
%

\begin{emptyroundbox}
\paragraph{Problem definition.}
In this work, our goal is to provide theoretical guarantees for the identifiability of the Markov equivalence class of the marginal graph $\cM^{\Gx}_V$ and its direct causes with the score, where variables $V_i$ are defined according to Equation \eqref{eq:latent_scm_v}.
\end{emptyroundbox}
 Without further assumptions on the data-generating process, we can identify the graph $\cM_V^\Gx$ only up to its partial ancestral graph. This information is encoded in the Jacobian of the score, as discussed in the next section.

\section{A score-matching-based criterion for m-separation}\label{sec:score_ci}
In this section, we show that for $V \subseteq X$ generated according to Equation \eqref{eq:latent_scm_v} the Hessian matrix of $\log p(V)$  identifies the equivalence class of the marginal MAG $\cM^\Gx_V$. It has already been proven that cross-partial derivatives of the log-likelihood are informative about a set of conditional independence relationships between random variables: \citet{spantini2018inference} (Lemma 4.1) and previously \citet{lin1997_factorizing} show that, given $V_Z \subseteq X$ such that $\set{V_i, V_j} \subseteq V_Z$, then
\begin{equation}\label{eq:lem_spantini}
    \frac{\partial^2}{\partial V_i \partial V_j} \log p(V_Z) = 0 \iff V_i \indep V_j | V_{Z} \setminus \set{V_i, V_j}.
\end{equation}
\looseness-1Equation \eqref{eq:dsep_criterion} resulting from faithfulness and the directed global Markov property implies that this expression can be used as a test of graphical separation to identify the Markov equivalence class of the graph $\cM_{V}^\Gx$, as commonly done in constraint-based causal discovery (for reference, see e.g. Section 3 in \citet{glymour2019_review}). This result generalizes Lemma 1 of \citet{montagna23_das}, where it is used to define constraints to infer edges in the causal structure without latent variables, under the assumption of nonlinear models with additive Gaussian noise.

\begin{proposition}[Corollary of \citet{spantini2018inference}]\label{prop:score_msep}\hspace{-1.5mm}\footnote{In their Lemma 4.1 \citet{spantini2018inference} provides the connection between vanishing cross-partial derivatives of the log-likelihood and conditional independence of random variables. Note that this result does not depend on the assumption of a generative model, thus holding beyond the set of structural equations \eqref{eq:latent_scm_v} and \eqref{eq:latent_scm_u}. Our result exploits their finding to the case when observations are generated according to a causal model with potentially latent variables.}
    Let $V$ be a set of random variables with strictly positive density generated according to the structural equations \eqref{eq:latent_scm_v}. For each set $V_Z \subseteq V$ of nodes in $\cM_{V}^\Gx$ such that $\set{V_i, V_j} \subseteq V_Z$, then the following holds:
$$
    \frac{\partial^2}{\partial V_i \partial V_j} \log p(V_Z) = 0 \iff V_i \mindep{\cM_{V}^\Gx} V_j | V_Z\setminus \set{V_i, V_j}.
$$
\end{proposition}

% \FM{Maybe cut this and move it to appendix - we just say that we can use this as CI test in constraint-based approach like FCI.}We build on this result to extend the findings of \citet{spantini2018inference}, showing that the equivalence class of the graph $\cM_V^\Gx$ can be identified using the cross partial derivatives of the log-likelihood as a test of conditional independence between variables, much in the spirit of the Fast Causal Inference algorithm \citep{spirtes2001_fci}.
The result of Proposition \ref{prop:score_msep} presents an alternative way to assess graphical separation in \textit{constraint-based} approaches to causal discovery: the equivalence class of the graph $\cM_V^\Gx$ can be identified using the cross-partial derivatives of the log-likelihood to characterize conditional independencies between variables, much in the spirit of the Fast Causal Inference (FCI) algorithm \citep{spirtes2001_fci}.
% \begin{proposition}
% \label{prop:pag_identifiable} 
%     Let $X$ be a set of random variables with strictly positive density generated according to model \eqref{eq:scm}, and $\cM^\Gx_V$ the marginalization onto $V \subseteq X$. Given 
%     \begin{displaymath}
%         \frac{\partial^2}{\partial V_i \partial V_j} \log p(V_Z)
%     \end{displaymath}
%     for all $V_i, V_j\in V, V_Z\subseteq V\setminus\{V_i, V_j\}$, we can identify the equivalence class $\cP_{\cM^\Gx_V}$. \FM{\textbf{TODO (Philipp)}: write the proof}
% \end{proposition}
Identifying the Markov equivalence class is the most we can hope to achieve without further restrictions on the hypothesis class.
As we will see in the next section, the score function can also help leverage additional restrictive assumptions on the causal mechanisms of Equation \eqref{eq:latent_scm_v} to identify direct causal effects.

\section{A theory of identifiability from the score}\label{sec:identifiability}

In this section, we show that, under additional assumptions on the data-generating process, we can identify some direct causal relations with the score even in the case of latent variable models.

As a preliminary step before diving into causal discovery with hidden nodes, we show how the properties of the score function identify edges in directed acyclic graphs, that is in the absence of unmeasured variables (when $U = \emptyset$ and $\Gx = \cM^\Gx_V$). The goal of the next section is two-sided: first, it introduces the fundamental ideas connecting the score function to causal discovery that also apply to hidden variable models, second, it extends the existing theory of causal discovery with score matching to additive noise models with both linear and nonlinear mechanisms. 

\subsection{Warm up: identifiability without latent confounders}\label{sec:observable_identifiability}
\looseness-1In this section, we summarise and extend the theoretical findings presented in \citet{montagna23_nogam}, where the authors show how to derive constraints on the score function that identify the causal order of the DAG $\Gx$ where all the variables in the set $X$ are observed. Define the structural relations of \eqref{eq:scm} as:
\begin{equation}\label{eq:anm}
    X_i \coloneqq h_i(X_{\Parents^\Gx_i}) + N_i, i = 1, \ldots, k,
\end{equation}
with three times continuously differentiable mechanisms $h_i$, noise terms centered at zero, and strictly positive density $p_X$. Further, we assume that the SCM on $X$ is a \textit{restricted} additive noise model (Definition \ref{def:restricted_scm}), which is necessary to ensure identifiability. Given the Markov factorization of Equation \eqref{eq:markov_factorization}, and the change of variable formula for densities, the components of the score function $\nabla \log p(x)$ are:
\begin{equation}\label{eq:score_general}
    \begin{split}
    \partial_{X_i} \log p(x) 
    % &= \partial_{X_i} \log p(x_i | x_{\Parents^\Gx_i}) + \sum_{j \in \Child^\Gx_i} \partial_{X_i}\log p(x_j | x_{\Parents^\Gx_j}) \\
    = \partial_{N_i} \log p(n_i) - \sum_{j \in \Child^\Gx_i} \partial_{X_i}h_j(x_{\Parents^\Gx_j})\partial_{N_j} \log p(n_j),
    \end{split}
\end{equation}
where $\Child^\Gx_i$ denotes the set of children of node $X_i$. We observe that if a node $X_s$ is a \textit{sink}, i.e. a node satisfying $\Child^\Gx_s = \emptyset$, then the summation over the children vanishes, implying $\partial_{X_s} \log p(x) = \partial_{N_s} \log p(n_s)$.
% implying that:
% \begin{equation}\label{eq:score_leaf_noise}
% \partial_{X_s} \log p(x) = \partial_{N_s} \log p(n_s).
% \end{equation}
\looseness-1The key point is that the score component of a sink node is a function of its structural equation noise term, such that one could learn a consistent estimator of $\partial_{X_s} \log p_X$
from a set of observations of the noise term $N_s$. Given that, in general, one has access to $X$ samples rather than observations of the noise random variables, authors in \citet{montagna23_nogam} show that $N_s$ of a sink node can be consistently estimated from i.i.d. realizations of  $X$. For each node  $X_1, \ldots, X_k$, we define the quantity:
\begin{equation}\label{eq:observed_residual}
    R_i \coloneqq X_i - \mean[X_i|X_{\setminus {X_i}}],
\end{equation}
where $X_{\setminus {X_i}}$ are the random variables in the set $X \setminus \set{X_i}$. $\mean[X_i|X_{\setminus {X_i}}]$ is the optimal least squares predictor of $X_i$ from all the remaining nodes in the graph, and $R_i$ is the regression residual. For a sink node $X_s$, the residual satisfies:
\begin{equation}\label{eq:observed_sink_residual}
R_s = N_s,
\end{equation}
which can be seen by rewriting $\mean[X_s|X_{\setminus {X_s}}] = h_s(X_{\Parents^\Gx_s}) +\mean[N_s|X_{\operatorname{DE}^\Gx_s}, X_{\operatorname{ND}^\Gx_s}] = h_s(X_{\Parents^\Gx_s}) + \mean[N_s]$, where $X_{\operatorname{DE}^\Gx_s}$ and $X_{\operatorname{ND}^\Gx_s}$ denotes the descendants and non-descendants of $X_s$, respectively. Equations \eqref{eq:score_general} and \eqref{eq:observed_sink_residual} together imply that the score $\partial_{N_s} \log p(N_s)$ is a function of $R_s$, such that it is possible to find a consistent approximator of the score of a sink from observations of $R_s$.
\begin{proposition}[Generalization of Lemma 1 in \citet{montagna23_nogam}]\label{prop:nogam}
    Let $X$ be a set of random variables, generated by a restricted additive noise model (Definition \ref{def:restricted_scm}) with structural equations \eqref{eq:anm}, and let $X_j \in X$. Then:
\begin{equation}
X_j \textnormal{ is a sink} \Longleftrightarrow \mean\left[\left(\mean\left[\partial_{X_j} \log p(X) \mid R_j\right] - \partial_{X_j} \log p(X)\right)^2\right] = 0.
\label{eq:lem_nogam}
\end{equation}  
\end{proposition}
Our result generalizes Lemma 1 in \citet{montagna23_nogam}, as they assume  $X$ generated by an identifiable additive noise model with nonlinear mechanisms, which we can replace by the weaker hypothesis of \textit{restricted} additive noise model. These findings are not surprising, in light of previous literature connecting the precision matrix and the causal structure of linear SCMs \citep{ghoshal2018learning}, as we discuss in Appendix \ref{app:nogam_linear}. 
% Instead, we remove the nonlinearity assumption and make the weaker hypothesis of a \textit{restricted} additive noise model, which is provably identifiable \citep{peters_2014_identifiability}, in the formal sense defined in the appendix (Definition \ref{def:identifiable}). This result doesn't come as a surprise, given the previous findings of \citet{ghoshal2018learning} showing that the score infers linear non-Gaussian additive noise models: Proposition \ref{prop:nogam} provides a general theory for the identifiability via the score of models with potentially mixed linear and nonlinear mechanisms.
%
Inspired by these results, in the next section we demonstrate the identifiability via the score of direct causal effects between a pair of variables in the marginal MAG $\cM^\Gx_V$ when $U \neq \emptyset$.

\subsection{Identifiability in the presence of latent confounders}\label{sec:unobservable_identifiability}
\looseness-1We now introduce our main theoretical result, that is: given a pair of nodes $V_i$, $V_j$ that are adjacent in the graph $\cM^\Gx_V$ with $U \neq \emptyset$, we can use the score function to identify the presence of a direct causal effect between $V_i$ and $V_j$ which can not be detected when inference is limited to the PAG $\mathcal{P}_{\cM^\Gx_V}$ (representing the Markov equivalence class). Given that the causal model of Equation \eqref{eq:latent_scm_v} ensures identifiability only up to the equivalence class, we need additional restrictive assumptions. In particular, we enforce an additive noise model with respect to both the observed and unobserved noise variables. This corresponds to an additive noise model on the observed variables with the noise terms recentered by the latent causal effects.

\begin{assumption}[SCM assumptions]\label{hp:scm}
    The set of structural equations of the observable variables specified in \eqref{eq:latent_scm_v} is now defined as:
\begin{equation}\label{eq:latent_identifiable_scm}
    V_i \coloneqq f_i(V_{\Parents^\Gx_i}) + g_i(U^i) + N_i, \forall i = 1, \ldots, d,
\end{equation}
assuming the mechanisms $f_i$ to be of class $\mathcal{C}^3(\R^{|V_{\Parents^\Gx_i}|})$, and mutually independent noise terms with strictly positive density function. The $N_i$'s are assumed to be non-Gaussian when $f_i$ is linear in some of its arguments.
\end{assumption}

% \FL{This assumption is enforcing an additive noise model with respect to both the observed and non-observed variables. This corresponds to an additive noise model on the observed variables with non-indipendent noise terms (because of the confounding, cf. Reichenbach principles) which themselves have additive noise.}

\looseness-1Crucially, our hypothesis is weaker than those required by two state-of-the-art approaches, CAM-UV \citep{maeda21_causal} and RCD \citep{maeda2020_rcd}: CAM-UV assumes a Causal Additive Model (CAM) with structural equations with nonlinear mechanisms in the form $V_i \coloneqq \sum_{k \in \Parents^\Gx_i} f_{ik}(V_k) + \sum_{U^i_k}g_{ik}(U^i_k) + N_i$, and RCD requires an additive noise model with linear effects of both the latent and observed causes. Thus, our model encompasses and extends the nonlinear and linear settings of CAM-UV and RCD, such that the theory developed in the remainder of the section is valid for a broader class of causal models.

Our first step is rewriting the structural relations in \eqref{eq:latent_identifiable_scm} as:
\begin{equation}\label{eq:tilde_anm}
    \begin{split}
            &V_i \coloneqq f_i(V_{\Parents^\Gx_i}) + \tilde{N}_i, \\
            & \tilde{N}_i \coloneqq g_i(U^i) + N_i,
            \forall i = 1, \ldots, d,
    \end{split}
\end{equation}
which provides an additive noise model in the form of \eqref{eq:anm}. 
% Now, we consider a pair of nodes $V_i, V_j$ such that 
% $$
% V_i \hspace{1em} \not\hspace{-1.05em}\mindep{\cM^\Gx_V} V_j | Z \setminus \set{V_i, V_j}, \forall Z \subseteq V,
% $$ 
% which implies that there exists an edge between $V_i$ and $V_j$ in the graph $\cM^\Gx_V$. 
Next, we define the following regression residuals for any node $V_k$ in the graph $\cM_V^\Gx$:
\begin{equation}\label{eq:latent_residual}
        R_k(V_Z) \coloneqq V_k -\mean[V_k \mid V_{Z \setminus \set{k}}],
\end{equation}
where $V_{Z \setminus \set{k}}$ denotes the set of random variables $V_Z \setminus \set{V_k}$.
% Finally, we recall the definition of \textit{inducing path} relative to a set of variables.
% \begin{definition}[Inducing path]
% Consider the set of random variables $X$ nodes of a DAG $\cG$, and $Y,Z$ disjoint subsets such that $X = Y \cup Z$. We define an \textit{inducing path relative to } $Z$ between nodes $Y_i, Y_j$ if every node in the path that is not in $Z$ is a collider on the path, i.e. for each $Y_k$ in the path the sequence $Y_i \ldots \rightarrow Y_k \leftarrow \ldots Y_j$ appears.
% \end{definition}
% Finally, we define two possible paths between variables in $V$, involving unobserved random variables in $U$. 

%  \begin{definition}[Unobserved causal path, UCP]
% A direct causal path between $V_i$ and $V_j$ is called an \textit{unobserved causal path} (UCP) if there is at least one node in the path that is in $U$ and that is a direct cause of $V_j$, as in $V_i \rightarrow \ldots \rightarrow U_a \rightarrow V_j$.
%  \end{definition}
%  \begin{definition}[Unobserved backdoor path]
% A path between $V_i$ and $V_j$ is called an \textit{unobserved backdoor path} (UBP) if it consists of two UCP to $V_i$ and $V_j$ from a common ancestor, as in $V_i \leftarrow U_a \ldots \leftarrow V_k  \rightarrow \ldots \rightarrow U_b \rightarrow V_j$.
%  \end{definition}
% Intuitively, an unobserved backdoor path between $V_i$ and $V_j$ is associated with a latent confounding effect between the two variables, and an unobserved causal path is concerned with the presence of unobserved mediators.
Given these definitions, we are ready to show that the score can identify the presence of direct causal effects between pairs of observed variables in $V$ which can not be detected by conditional independence testing (e.g., as in FCI). In particular, if the mechanisms are nonlinear, direct parents in the DAG can be identified, while in the linear case identifiability is limited to ancestorship relations. 

\subsubsection{Identifiability of directed edges}
\looseness-1Let $V_i, V_j$ be adjacent nodes in the PAG $\mathcal{P}_{\cM^\Gx_V}$: we want to investigate when a direct causal effect $V_i \in V_{\Parents^\Gx_j}$ can be identified from the score. Consider the set of variables $V_Z = V_{\Parents^\Gx_j} \cup \set{V_j}$ and the graph $\mathcal M_{V_Z}^\Gx$ resulting from the marginalization of $\Gx$ on $V_Z$. By Equation \eqref{eq:tilde_anm}, we have that:
\begin{equation*}
    \begin{split}
        V_j \coloneqq f_j(V_{\Parents^\Gx_j}) + \tilde{N}_j, \hspace{1em} \tilde{N}_j \coloneqq g_j(U^j) + N_j.
    \end{split}
\end{equation*}
The key observation is that for $V_{\Parents^\Gx_j} \indep_d^\Gx U^j$, then $V_{\Parents^\Gx_j} \indep \tilde N_j$ (e.g. \citet{blitzstein2019probability}, Theorem 3.8.5): this allows considering $\tilde N_j$ as an \textit{exogenous} noise term independent of other variables in the structural equation of $V_j$. By Equation \eqref{eq:latent_residual} this implies that 
\begin{equation}\label{eq:latent_residual_sink}
R_j(V_Z) = \tilde{N_j} - \mean[\tilde{N_j}],
\end{equation}
where we use $V_{\Parents^\Gx_j} \indep \tilde N_j$ to write $\mean[\tilde N_j | V_{\Parents_j^\Gx}] = \mean[\tilde N_j]$. Moreover, $V_{\Parents^\Gx_j} \indep \tilde N_j$ implies that $p(V_j| V_{\Parents_j^\Gx}) = p(\tilde N_j)$, and from simple manipulations we can show that $\partial_{V_j} \log p(V_Z) = \partial_{N_j}\log p(\tilde N_j)$. We conclude that in analogy to the case without latent variables the score $\partial_{V_j} \log p(V_Z)$ is a function of $\tilde{N_j}$, the error term in the additive noise model of Equation \eqref{eq:tilde_anm}. Then the score of $V_j$ can be consistently predicted from observations of the residual $R_j(V_Z)$, which is the sufficient and necessary condition to discover the direct causes (parents and ancestors) of $V_j$ with the score.

% \begin{proposition}\label{prop:causal_dir_1}
% Let X be generated by a restricted additive noise model with structural equations \textcolor{red}{\eqref{eq:latent_identifiable_scm}, and causal graph $\mathcal{G}$, where $g_i$ be non-linear for all $i=1, \dots, d$}. Consider $V_i, V_j$ adjacent in $\cP_{\cM^\Gx_V}$, the PAG relative to marginalization $\cM^\Gx_V$. Further, assume that the score component $\partial_{V_j} \log p(V_Z)$ is non-constant for uncountable values of $V_Z$.
% \begin{enumerate}[(i)]
%     \item Let $V_Z = V_{\Parents^\Gx_j} \cup \set{V_i, V_j}$.
%     Then:
%     \begin{equation*}
%         V_{\Parents^\Gx_j} \indep^d_{\Gx} U^j \land V_i \in V_{\Parents^\Gx_j} \Longleftrightarrow  \mean[\partial_{V_j} \log p(V_Z) - \mean[\partial_{V_j} \log p(V_Z) | R_j(V_Z)]]^2 = 0.
%     \end{equation*}
%     \item Let $V_Z \subseteq V$, such that $\set{V_i, V_j} \subseteq V_Z$. Then:
%     \begin{equation*}
%     V_{\Parents^\Gx_j} \notdindep{\Gx} U^j \lor V_i \not\in V_{\Parents^\Gx_j} \Longleftrightarrow 
%         \mean[\partial_{V_j} \log p(V_Z) - \mean[\partial_{V_j} \log p(V_Z) | R_j(V_Z)]]^2 \neq  0.
%     \end{equation*}
%     %
% \end{enumerate}
%     \textcolor{red}{If $f_i$ and $g_i$ are linear for $i=1, \dots, d$ we have
%     \begin{equation*}
%     V_{\Parents^{\Gx}_j} \notdindep{\Gx} U^j \lor V_i \not\in V_{\Anc^\Gx_j} \Longleftrightarrow 
%         \mean[\partial_{V_j} \log p(V) - \mean[\partial_{V_j} \log p(V) | R_j(V)]]^2 \neq  0.
%     \end{equation*}
%     }
% \end{proposition}

\begin{proposition}\label{prop:causal_dir_1}
Let $X= V \cup U$, $V$ and $U$ disjoint, be generated by a restricted additive noise model with causal graph $\mathcal{G}$. Let $V$ satisfying the set of structural equations \eqref{eq:latent_identifiable_scm}, and $f_i$ nonlinear for each $i = 1, \ldots, d$. Consider $V_i, V_j$ adjacent in $\cP_{\cM^\Gx_V}$, the PAG relative to marginalization $\cM^\Gx_V$. Further, assume that for each subset $V_Z \subseteq V$ the score component $\partial_{V_j} \log p(V_Z)$ is non-constant for uncountable values of $V_Z$. Then, 
\begin{equation}\label{eq:main_prop_nonlinear}
\begin{split}
    &\exists V_Z \subseteq V, \set{V_i, V_j} \in V_Z, \textnormal{s.t. } \mean\left[\left(\partial_{V_j} \log p(V_Z) - \mean[\partial_{V_j} \log p(V_Z) | R_j(V_Z)]\right)^2\right] = 0 \\
    &\Longleftrightarrow
    V_{\Parents^\Gx_j} \indep^d_{\Gx} U^j  \land V_i \in V_{\Parents^\Gx_j}.
\end{split}
\end{equation}
If $f_i$ linear for each $i = 1, \ldots, d$ we have
\begin{equation}\label{eq:main_prop_linear}
\begin{split}
   &\exists V_Z \subseteq V, \set{V_i, V_j} \in V_Z, \textnormal{s.t. } \mean\left[\left(\partial_{V_j} \log p(V_Z) - \mean[\partial_{V_j} \log p(V_Z) | R_j(V_Z)]\right)^2\right] = 0 \\ 
   &\Longleftrightarrow
   V_{\Parents^\Gx_j} \indep^d_{\Gx} U^j \land V_i \in V_{\Anc^\Gx_j}.
\end{split}
\end{equation}
\end{proposition}

The proof is found in \cref{app:main_prop_proof}. Here, we present the intuition about the content of the proposition. Given two adjacent nodes $V_i, V_j$ in the PAG, they must be graphically connected by one edge between $\circ \hspace{-1.9mm}\rightarrow, \circ \hspace{-1.60mm}- \hspace{-1.60mm}\circ, \leftrightarrow$ or $\rightarrow$, the latter denoting ancestral relation in the DAG (which may not be a direct parent). \cref{eq:main_prop_nonlinear} provides the condition under which a parent-child relation can be identified in place of the less informative PAG edges: in particular, for nonlinear additive noise models, that is when
the independence $V_{\Parents^\Gx_j} \indep \tilde N_j$ implied by $V_{\Parents^\Gx_j} \dindep{\Gx} U^j$ allows to interpret $\tilde N_j$ as an \textit{exogenous} noise term independent of other variables in the structural \cref{eq:tilde_anm} of $V_j$.
% no path between $V_j$ and its observed parents includes hidden variables. 
This condition is necessary: given an active path such that $V_{\Parents^\Gx_j} \not\hspace{-2.345mm}\dindep{\Gx} U^j$,  the score could not identify a direct causal effect $V_i \in  V_{\Parents_j^\Gx}$.
We remark on the novelty of our theory compared to \citet{maeda21_causal}, the nearest neighbor to our work in the literature: they demonstrate identifiability of direct parents under the assumptions of additive and nonlinear mechanisms (i.e. $ V_i \coloneqq \sum_{k \in \Parents^\Gx_i} f_{ik}(V_k) + \sum_{U^i_k}g_{ik}(U^i_k) + N_i$), which is more restrictive than our modeling hypothesis. In the case of linear mechanisms, we can only identify $V_i \in V_{\Anc_j^\cG}$, which subsumes the findings of \citet{maeda2020_rcd}, as they do not consider the case of unobserved mediators. See Appendix \ref{sec:difference_adascore_camuv} for more details on the relation of these results.

\paragraph{Connections to FCI.} In \cref{app:difference_adascore_fci} we present a thorough analysis of the relation between the identifiability guarantees of Proposition \ref{prop:causal_dir_1} and those provided by the FCI algorithm. Intuitively, under standard assumptions, FCI discovers all aspects of the causal structure that are uniquely determined by facts of probabilistic dependence and independence \citep{zhang2008onthecompleteness}. Adding the assumption of an additive noise model, as in our case, allows further identifying causal directions that are not distinguishable purely by probabilistic (in)dependence facts.

\paragraph{Connections to local causal discovery.} \looseness-1We note that the result of Proposition \ref{prop:causal_dir_1} has an important application in the domain of \textit{local causal discovery}, which is based on the idea that the identification of causal effects of the covariates on the response under interventions only requires knowledge of the local causal structure around the treatment \citep{maathuis2009localcd}. Then, the discovery of parental relations between a targeted subset of nodes, rather than on the entire set of observed variables, may be sufficient for some downstream tasks in causal inference. The benefits of this approach are clear in terms of computational efficiency. Proposition \ref{prop:causal_dir_1} shows that for any pair of nodes, the score function can discover a parent-child (or ancestral) relationship while being agnostic of the structure of the other nodes in the graph. This is in contrast with existing approaches to causal discovery with score matching that are concerned with the inference of the global topological order, which might be unnecessarily computationally expensive. While we do not elaborate on this connection in the following algorithmic section, we believe this to be an important potential application building on our results.

\vspace{1em}
\looseness-1We have established theoretical guarantees of identifiability for additive noise models, even in the presence of hidden variables: we find that, for nonlinear models, the score function is a means for the identifiability of all direct parental relations that are not influenced by unobserved variables; all the remaining arrowheads of the edges in the graph $\cM_{V}^\Gx$ are identified no better than in the equivalence class. For linear models, identifiability is guaranteed for ancestral relations. Based on these insights, we propose AdaScore, a score matching-based algorithm for the inference of Markov equivalence classes and direct causal effects, in the presence of latent variables.

\subsection{A score-matching-based algorithm for causal discovery}\label{sec:algo}
Building on our theory, we propose AdaScore, a generalization of NoGAM to linear and nonlinear additive noise models with latent variables. The main strength of our approach is the adaptivity of its theoretical guarantees for
     this broad class of structural assumptions, as illustrated in Table \ref{tab:assumptions}. In practice, we design our method to be flexible in its output: based on the user's belief about the plausibility of several modeling assumptions on the data, AdaScore can output an equivalence class (using the condition of Proposition \ref{prop:score_msep} to infer conditional independence in an FCI-like algorithm), a directed acyclic graph (as in NoGAM), or a mixed graph, accounting for the presence of unobserved variables. We now describe the version of our algorithm whose output is a mixed graph, where we rely on score matching estimation of the score and its Jacobian (\cref{app:finite_adascore}). 
     
     At an intuitive level, we find unoriented edges using Proposition \ref{prop:score_msep}, i.e. checking for dependencies in the form of non-zero entries in the Jacobian of the score via hypothesis testing on the mean, and find the edges' directions via the condition of Proposition \ref{prop:causal_dir_1}, i.e. by estimating residuals of each node $X_i$ (via kernel ridge-regression, as we motivate in Appendix \ref{app:krr_residuals}) and checking whether they can correctly predict the $i$-th score entry (the vanishing mean squared errors are verified by hypothesis test of independent residuals, see \cref{app:finite_adascore}). It would be tempting to simply find the skeleton (i.e. the graphical representation of the constraints of an equivalence class) first via the well-known adjacency search of the FCI algorithm and then iterate through all neighborhoods of all nodes to orient edges using Proposition \ref{prop:causal_dir_1}.
This would be prohibitively expensive. 

Instead, we propose an alternative solution: exploiting the fact that some nodes may not be influenced by latent variables, we first use Proposition \ref{prop:nogam} to find sink nodes that are not affected by latents (using hypothesis testing to find vanishing mean squared error in the score predictions from the residuals), in the spirit of the NoGAM algorithm. If there is such a sink, we search all its adjacent nodes via Proposition \ref{prop:score_msep} (plus an optional pruning step for better accuracy, \cref{app:finite_adascore}),
% (combining Proposition \ref{prop:pag_identifiable} score-matching-based independence condition and CAM-pruning \citep{buhlmann14_cam}, an edge selection procedure on additive noise models, see \cref{app:cam})
and orient the inferred edges towards the sink. Else, if no sink can be found, we pick a node in the graph and find its neighbors by Proposition \ref{prop:score_msep}, orienting its edges using the condition in Proposition \ref{prop:causal_dir_1} (score estimation by residuals under latent effects). This way, we get an algorithm that is polynomial in the best case (\cref{app:complexity}). Details on AdaScore are provided in \cref{app:algorithm}, while a pseudo-code summary is provided in the \cref{algo:simplified_scam} box.

\begin{algorithm}[ht]
\caption{Simplified pseudo-code of AdaScore}
\label{algo:simplified_scam}
\While{nodes remain}{
    Find sink candidate using Proposition \ref{prop:causal_dir_1}\;\\
    \eIf{Proposition \ref{prop:causal_dir_1} finds a sink \textbf{\emph{and}} output is mixed graph\\
    \textbf{\emph{or}} output is DAG}{
        Add edges from adjacent nodes to sink
    }{
        Pick some remaining node $V_i \in V$\;\\
        Prune neighbourhood of $V_i$ using Proposition \ref{prop:score_msep}\;\\
        \If{output is not PAG}{
            Orient edges adjacent to $V_i$ using Proposition \ref{prop:causal_dir_1}\;\\
        }
        \eIf{$V_i$ has outgoing directed edge to some $V_j \in V$}{
            \textbf{continue with} $V_j$\;
        }{
            Remove $V_i$ from remaining nodes\;
        }
    }
}
Prune remaining potential edges using Proposition \ref{prop:score_msep}\;\\
\If{output is PAG}{Do PAG orientations using Proposition \ref{prop:score_msep}\;}
\end{algorithm}

%%%%%%%%%%%%%%%%%%%%%%%%%%%%%%%%%%%%%%%%%%%%%%%%%%%%%%%%%%%%
\section{Experiments}
\label{sec:experiments}
\looseness-1The code for all experiments is available under \url{https://github.com/amazon-science/causal-score-matching}. We use the \texttt{causally}\footnote{\url{https://causally.readthedocs.io/en/latest/}} Python library \citet{montagna2023_assumptions}  to generate synthetic data with known ground truths, created as Erd\"os-R\'enyi sparse and dense graphs, respectively with probability of edge between pair of nodes equals $0.3$ and $0.5$.
We sample the data according to linear and nonlinear mechanisms with additive noise,
% (both with additive and nonadditive noise terms),
where the nonlinear functions are parametrized by a neural network with random weights, a common approach in the literature \citep{montagna2023shortcuts, montagna2023_assumptions, lippe2022efficient, ke2023learning, brouillard2020intervention}. Noise terms are sampled from a uniform distribution in the $[-2, 2]$ range. We introduce hidden variables by randomly picking two nodes and dropping the corresponding column from the data matrix (datasets with no confounding effect are discarded and re-sampled to ensure that experiments are meaningful). See \cref{app:data} for further details on the data generation.
As metric, we consider the structural Hamming distance (SHD) \citep{tsamardinos2006max,triantafillou2016score}, a simple count of the number of incorrect edges, where missing and wrongly directed edges count as one error.
\begin{comment}
Further, we use the $F_1$ score to evaluate predictions of the equivalence class. I.e. we apply the $F_1$ score to the binary decision whether there is a directed (or bidirected or any respectively) edge between two nodes over all pairs of nodes.
\end{comment} 
We fix the level of the hypothesis tests of AdaScore to $0.05$, which is a common choice in the absence of prior knowledge. We compare AdaScore to NoGAM, CAM-UV, RCD, and DirectLiNGAM, whose assumptions are detailed in \cref{tab:assumptions}. We also adopt a random baseline, described in Appendix \ref{app:random_baseline}. In the main manuscript, we consider inference of \textit{sparse} graphs, where each dataset contains $1000$ observations (boxplots are obtained sampling datasets with $20$ different random seeds). Additional experiments including those on dense networks are presented in \cref{app:experiments}. Our synthetic data are standardized by their empirical variance to remove shortcuts in the data \citep{montagna2023shortcuts, reisach2021beware}. We limit our synthetic experiments to graphs with $9$ nodes, as we empirically observed that AdaScore and CAM-UV struggle to scale with the number of variables. A thorough analysis of the elapsed computational time is provided in Appendix \ref{app:experiments_time}. 

\looseness-1Further, we show results for three real and pseudo-real benchmark datasets. Namely, a biological dataset on cell signaling \citep{sachs2005causal}, the \texttt{AutoMPG} dataset\footnote{Some of the features in this dataset are not continuous. For simplicity, we still treat them as if they were.} concerning fuel consumption in cars \citep{bache2013uci}, where we use the causal ground truth given by \citep{wang2017visual}, and the synthetic FMRI dataset \texttt{Sim2} \citep{smith2011network}. For each dataset, we randomly pick two variables and drop them to introduce hidden variables. For \texttt{Sachs} and \texttt{Sim2} we also randomly\footnote{\texttt{Sim2} is subdivided into samples from different (synthetic) subjects. We picked our subsample across all subjects.} select 1000 samples.
The experiments are repeated 20 times.
\begin{figure}
\centering
%\subfigure[]{%
    \includegraphics[width=0.65\textwidth]{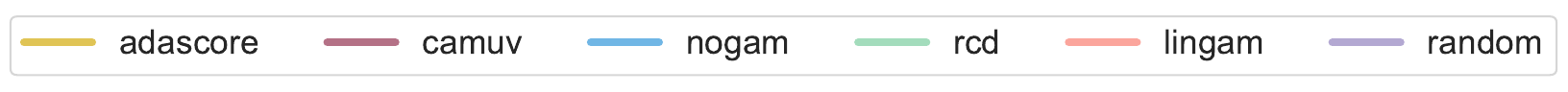}\\
%}
\vspace{1.2em}
\subfigure[Fully observable model\label{fig:sparse_observable}]{%
    \includegraphics[width=0.8\textwidth]{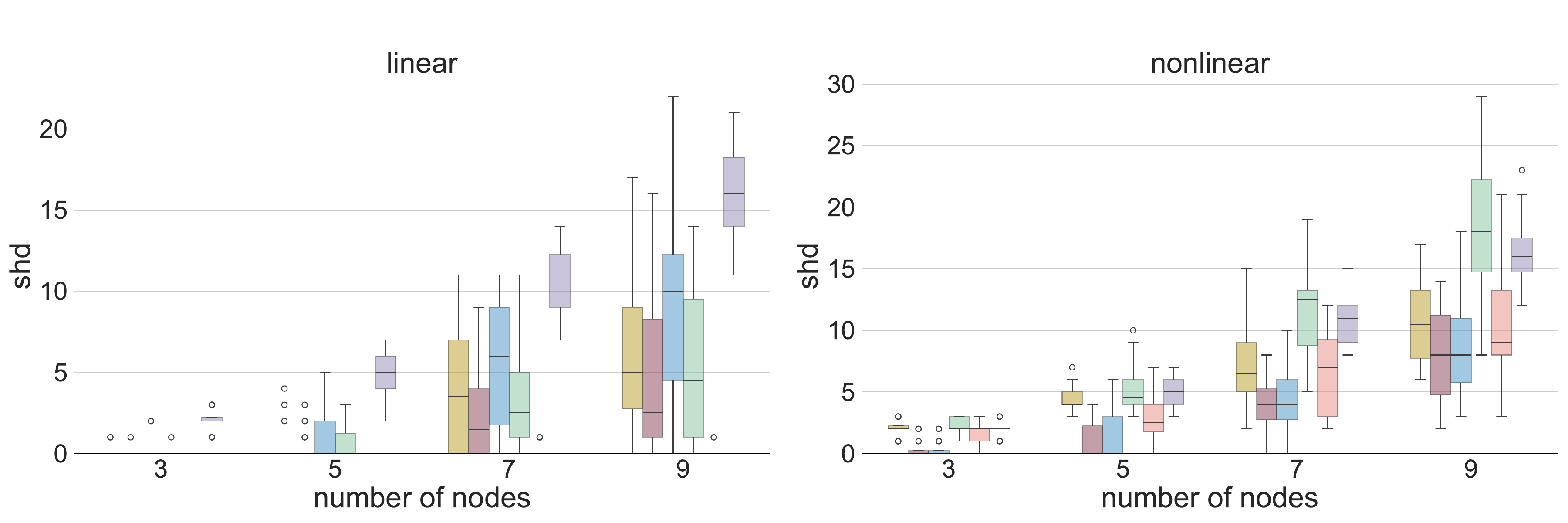}
}
\vspace{1.2em}
\subfigure[Latent variables model\label{fig:sparse_latent}]{%
    \includegraphics[width=0.8\textwidth]{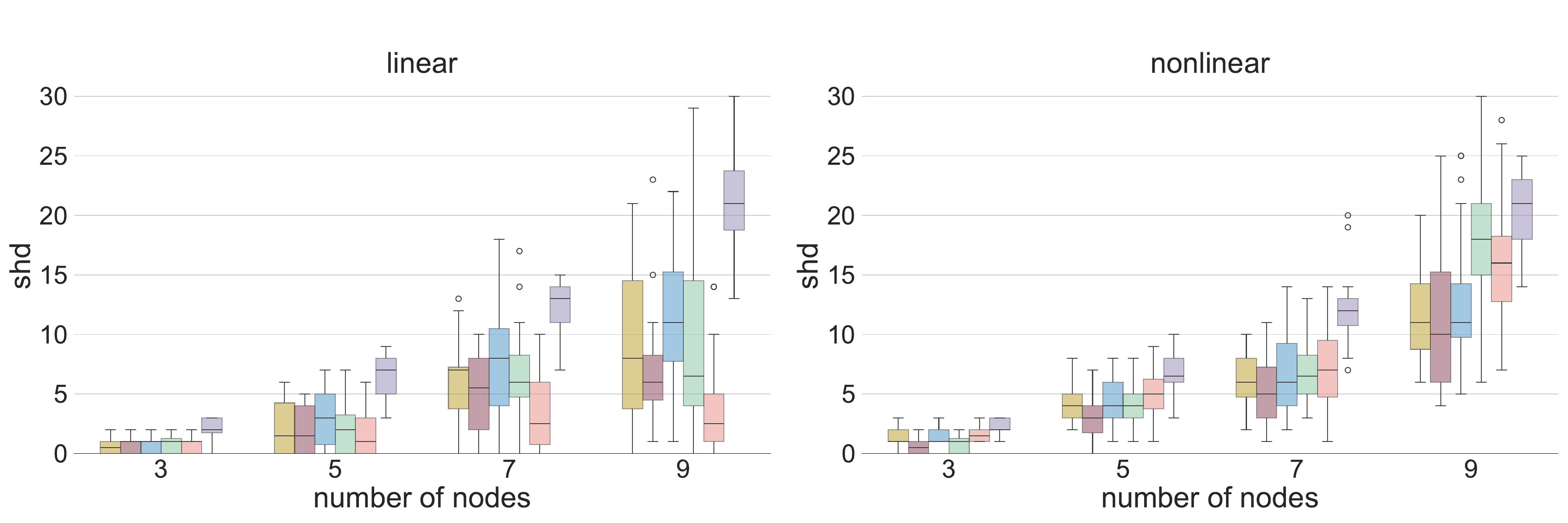}
}
\vspace{-1.2em}
\caption{\looseness-1\looseness-1Empirical results on sparse graphs with different numbers of nodes, on fully observable (no hidden variables) and latent variable models. We report the SHD accuracy (lower is better). We note that Adascore is comparable to the other methods in all settings (except for DirectLiNGAM on linear data), and always significantly better than random.}
\label{fig:experiments-sparse}
\end{figure}

\begin{figure}
\centering
    %\subfigure{%
    \includegraphics[width=0.55\textwidth]{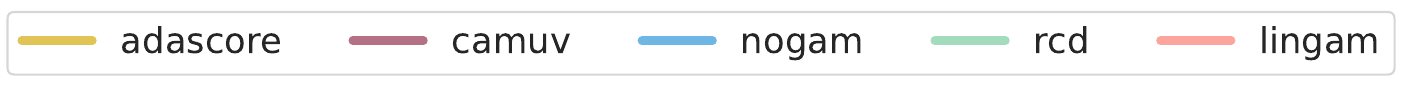}\\
    %}
    \vspace{1.2em} 
    \subfigure[Cell signaling data\label{fig:sachs}]{%
        \includegraphics[width=0.3\textwidth]{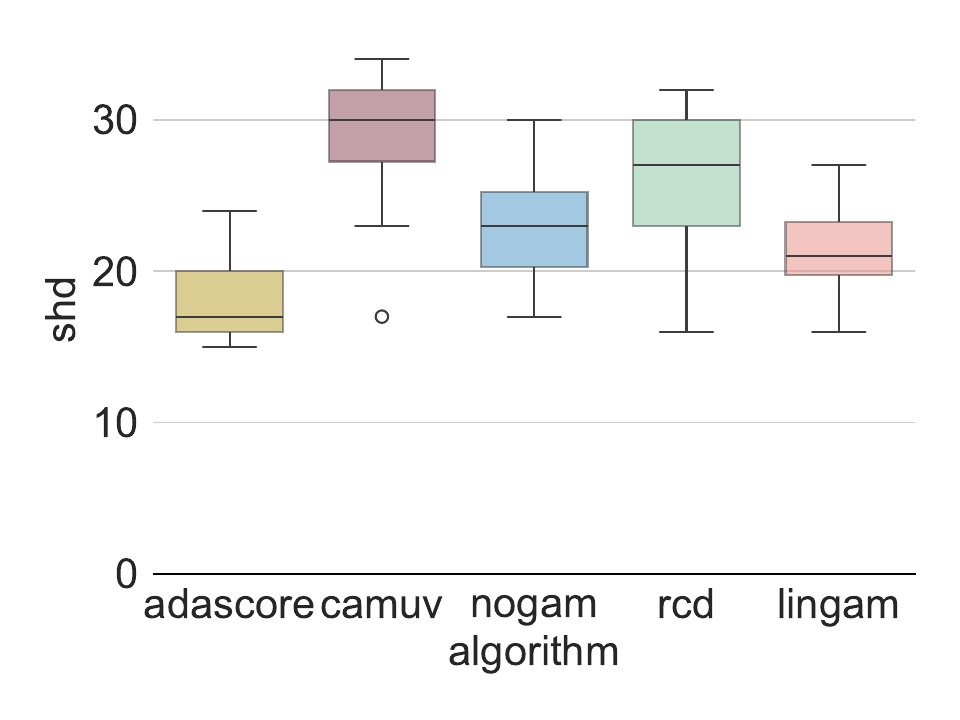}
    }
    \vspace{1.2em}
    \subfigure[Fuel consumption data\label{fig:auto_mpg}]{%
        \includegraphics[width=0.3\textwidth]{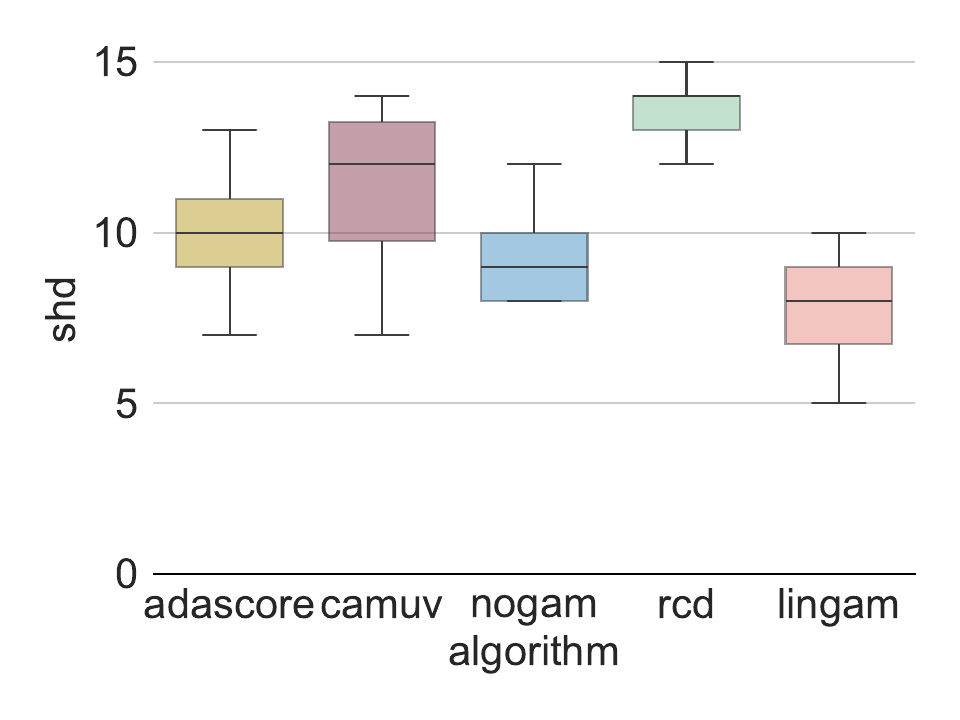}
    }
    \vspace{1.2em}
    \subfigure[FMRI data\label{fig:fmri}]{%
        \includegraphics[width=0.3\textwidth]{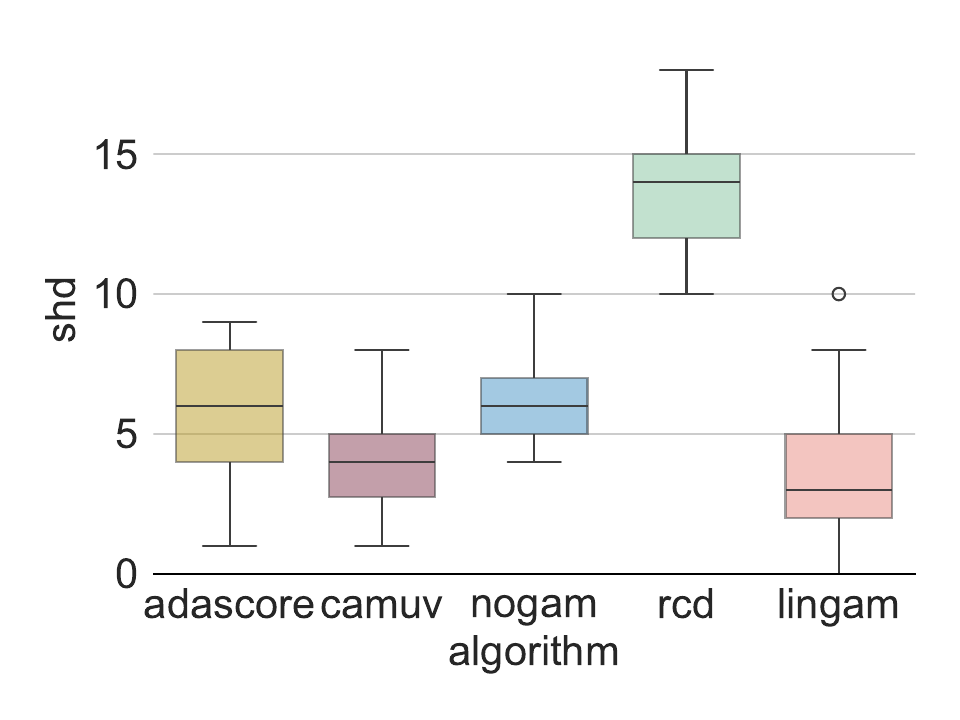}
    }
    \vspace{-2.4em}
\caption{\looseness-1Empirical results on real and pseudo-real datasets from \citet{sachs2005causal}, \citet{bache2013uci} , and \citet{smith2011network}. We report the SHD accuracy (lower is better). AdaScore has the lowest SHD among all tested methods on the gene dataset and appears to be competitive compared to other methods on the fuel consumption and FMRI data.}
\label{fig:experiments-real}
\end{figure}
\paragraph{Discussion.} Our experimental results on models without latent variables of \cref{fig:sparse_observable} show that when causal relations are linear, AdaScore can recover the causal graph with accuracy that is comparable with all the other benchmarks, with the exception of DirectLiNGAM. 
%On nonlinear data, AdaScore presents better performance than CAM-UV, RCD, and DirectLiNGAM while being comparable to NoGAM in accuracy. This is in line with our expectations: in the absence of finite sample errors and in the fully observable setting, NoGAM and AdaScore are indeed the same algorithms. 
On nonlinear data, AdaScore outperforms RCD accuracy, while being slightly worse than CAM-UV and NoGAM. 
When inferring under latent causal effects, \cref{fig:sparse_latent}, our method is comparable to CAM-UV and NoGAM, and appears to be preferrable to RCD, which degrade its performance with scale and is no better than random with $9$ nodes. % while slightly degrading on nine nodes. %Additionally, AdaScore outperforms NoGAM in this setting, as we would expect according to our theory. 
Finally, on the real benchmarks of  Figure \ref{fig:experiments-real}, AdaScore presents promising results compared to the other methods. On cell signaling data, our algorithm emerges as the best option, while it retains competitive performance on fuel consumption and FMRI data. 
Overall, we observe that our method is robust to a variety of structural assumptions, with accuracy that is comparable and sometimes better than competitors. We remark that although AdaScore does not clearly outperform the other baselines, its broad theoretical guarantees of identifiability are not matched by any available method in the literature; this makes it an appealing option for inference in realistic scenarios that are hard to investigate with synthetic data, where the structural assumptions of the causal model underlying the observations are unknown. 

\begin{table}
\caption{\looseness-1Algorithms of our experiments. Each cell denotes
denotes whether the method has (\cmark) or has not (\xmark) guarantees of identifiability under the condition specified in the related row.}
\label{tab:assumptions}
\centering
    \setlength{\tabcolsep}{5.pt}
    \begin{tabular}{lccccc}
    \toprule
         & CAM-UV & RCD & 
         NoGAM & DirectLiNGAM & AdaScore \\
         \midrule
         \scalebox{0.92}{Linear additive noise model}  & \xmark & \cmark & \xmark & \cmark & \cmark \\
         \scalebox{0.92}{Nonlinear additive noise model}  & \xmark & \xmark & \cmark & \xmark & \cmark\\
         \scalebox{0.92}{Nonlinear CAM}  & \cmark & \xmark & \cmark & \xmark & \cmark\\
         \scalebox{0.92}{Latent variables effects}  & \cmark & \cmark & \xmark & \xmark & \cmark\\
         % \scalebox{0.92}{Unfaithful distribution}  & \xmark & \xmark & \xmark & \xmark & \xmark\\
         \midrule
         \scalebox{0.92}{Output} & Mixed & Mixed & DAG & DAG & Mixed \\
         \bottomrule 
    \end{tabular}
\end{table}
% \begin{tabular}{lccccc}
% \toprule
% Algorithm          & Faithfulness & Additive Noise & CAM & Causal Sufficiency & Output \\ 
% \midrule
% CAM-UV             & \cmark & \cmark   & \cmark                & \xmark                  & Mixed Graph   \\ 
% FCI                & \cmark           & \xmark              & \xmark                          & \xmark       & PAG   \\ 
% NoGAM              & \cmark& \cmark             & \xmark                          & \cmark      & DAG   \\ 
% Direct LiNGAM      & \xmark & \cmark  & \xmark                          & \cmark      & DAG   \\ 
% %DAS                & \cmark           & \cmark             & \xmark                          & \cmark      & DAG   \\ 
% RCD         & \cmark            & \cmark  & \xmark                &    \cmark   & DAG   \\ 
% \bottomrule
% \end{tabular}
% \end{table}
\section{Conclusion}
\looseness-1The existing literature on causal discovery shows a connection between score matching and structure learning in the context of nonlinear ANMs: in this paper, (i) we formalize and extend these results to linear SCMs, and (ii) we show that the score retains information on the causal structure even in the presence of unobserved variables. While previous works posit the accent on finding the causal order through the score, we study its potential to identify the Markov equivalence class with a \textit{constraint-based} strategy, as well as to identify direct causal effects. Our theoretical insights result in AdaScore: unlike existing approaches for the estimation of causal directions, our algorithm provides theoretical guarantees for a broad class of identifiable models, namely linear and nonlinear, with additive noise, in the presence of latent variables. Even though AdaScore does not clearly outperform the existing baselines on our synthetic benchmark, it appears promising on realistic datasets, and its adaptivity to different structural hypotheses is a step towards causal discovery that is less reliant on prior assumptions, which are often untestable and thus hindering reliable inference in real-world problems. A promising research direction in relation to our work involves extending and applying our theory to algorithms for local causal discovery.
% While we do not touch on the task of causal representation learning \citep{scholkopf2021towards}, where causal variables are learned from data, we believe this is a promising research direction in relation to our work due to the specific interplay between score-matching estimation and generative models. 

% Acknowledgments---Will not appear in anonymized version
\acks{Philipp M.
Faller was supported by a doctoral scholarship of the
Studienstiftung
des deutschen Volkes (German Academic Scholarship
Foundation). This work
has been supported by AFOSR, grant n. FA8655-20-1-7035. FM is supported by Programma
Operativo Nazionale ricerca e innovazione 2014-2020. We thank Atalanti A. Mastakouri, Kun Zhang and Haoyue Dai for the insightful discussions.}

\bibliography{biblio}

\newpage
\appendix

\section{Related works}
In this section we discuss works closely related to ours, in the context of observational causal discovery with and without latent variables. 
\paragraph{Causal discovery with score-matching.} Several methods for the causal discovery of fully observable models using the score have been recently proposed. \citet{ghoshal2018learning} demonstrates the identifiability of the linear non-Gaussian model from the score, and it is complemented by \citet{rolland22_score}, which shows the connection between score matching estimation of $\nabla \log p(X)$ and the inference of causal graphs underlying nonlinear additive noise models with Gaussian noise terms, also allowing for sample complexity bounds~\citep{zhu2024sample}. \citet{montagna23_nogam} provides identifiability results in the nonlinear setting, without posing any restriction on the distribution of the noise terms. \citet{montagna23_das} is the first to show that the Jacobian of the score provides information equivalent to conditional independence in the context of causal discovery, limited to the case of additive noise models. All of these studies make specialized assumptions to find theoretical guarantees of identifiability, whereas our paper provides a unifying view of causal discovery with the score function, which generalizes and expands the existing results.

\paragraph{Causal discovery with latent variables.} Causal discovery with latent variables has been studied first in the context of \textit{constraint-based} approaches with the FCI algorithm \citep{spirtes2001_fci}, which shows the identifiability of the equivalence class of a marginalized graph via conditional independence testing. 
There are several methods that have been proposed for specific structural or functional assumptions.
E.g. assuming linearity and restrictions on possible graphs \citet{silva2006learning} present a method based on Tetrad-constraints, while \citet{chandrasekaran2010latent} use a maximum-likelihood-based approach for linear models and sparse graphs.
A wide class of approaches builds on the assumption of non-Gaussian additive noise, going back to the work of \citet{shimizu2006_icalingam} and \citet{hoyer08_anm} on cases without latent variables.
Some examples include \citet{janzing2009identifying}, who show how confounders can be identified in a bivariate setting with non-linear causal relationship, \citet{adams2021identification}, who use conditions on the structural coefficients, \citet{wang2023causal}, who recover the causal structure from statistical moments or \citet{dong2024a}, who impose constraints on the rank of  cross-covariance matrices.
The RCD and CAM-UV \citep{maeda2020_rcd, maeda21_causal} approaches  demonstrate the inferrability of causal edges via testing for independent regression residuals. Like the aforementioned methods, both rely on strong assumptions on the causal mechanisms: their theoretical guarantees apply to models where the effects are generated by a linear (RCD) or nonlinear (CAM-UV) additive contribution of each cause. 

Our work demonstrates that using the score function for causal discovery, one can unify and generalize several of these results, 
%presenting an alternative to conditional independence testing for constraint-based methods, and 
while being agnostic about the class of causal mechanisms of the observed variables, under the weaker requirement of additivity of the noise terms.
Further, we show how the score can be utilized for causal discovery with latent variables in a non-parametric setting.

\section{Useful results}\label{app:useful}
In this section, we provide a collection of results and definitions relevant to the theory of this paper.

% --------------------------------------
% --------------------------------------
\subsection{Definitions over graphs}\label{app:useful_graph}
Let $X = {X_1, \ldots, X_d}$ a set of random variables. A graph $\cG = (X, E)$ consists of finitely many nodes or vertices $X$ and edges $E$. We now provide additional definitions, separately for directed acyclic and mixed graphs.

\paragraph{Directed acyclic graph.}
In a \textit{directed graph}, nodes can be connected by a \textit{directed edge} ($\to$), and between each pair of nodes there is at most one directed edge. We say that $X_1$ is a \textit{parent} of $X_j$ if $X_i \rightarrow X_j \in E$, in which case we also say that $X_j$ is a \textit{child} of $X_i$. Two nodes are \textit{adjacent} if they are connected by an edge. A \textit{path} in $\cG$ is a sequence of at least two distinct vertices $\pi = X_{i_1}, \ldots, X_{i_m}$ such that there is an edge between $X_{i_k}$ and $X_{i_{k+1}}$ for each $k = 1, \ldots, m$. If $X_{i_k} \to X_{i_{k+1}}$ for every node in the path, we speak of a \textit{directed path}, and call $X_{i_1}$ an \textit{ancestor} of $X_{i_{m}}$, $X_{i_{m}}$ a \textit{descendant} of $X_{i_1}$. Given the set $\operatorname{DE}^\Gx_i$ of descendants of a node $X_i$, we define the set of \textit{non-descendants} of $X_i$ as $\Nondesc^\cG_i = X \setminus (\operatorname{DE}^\Gx_i \cup \set {X_i})$. Given the path $\pi = X_{i_1}, \ldots, X_{i_m}$, we say that $X_{i_k}$, $k = 2, \ldots, m-1$, is a \textit{collider on} $\pi$ if $X_{i_{k-1}}, X_{i_{k+1}}$ are both parents of $X_{i_k}$, and we call the triplet $X_{i_{k-1}} \rightarrow X_{i_k} \leftarrow X_{i_{k+1}}$ a \textit{v-structure}. A node without parents is called a \textit{source node}. A node without children is called a \textit{sink node}. A \textit{directed acyclic graph} is a directed graph with no cycles.

\paragraph{Mixed graph.} In a  \textit{mixed graph}  nodes can be connected by a \textit{directed edge} ($\to$) or a \textit{bidirected edge} ($\leftrightarrow$), and between each pair of nodes there is at most one directed edge. Two vertices are said to be \textit{adjacent} in a graph if there is an edge (of any kind) between them. The definitions of \textit{parent}, \textit{child}, \textit{ancestor}, \textit{descendant}, \textit{path} provided for directed acyclic graph also apply in the case of mixed graphs. Additionally, $X_i$ is a \textit{spouse} of $X_j$ (and vice-versa) if $X_i \leftrightarrow X_j \in E$. An \textit{almost directed cycle} occurs
when $X_i \leftrightarrow X_j \in E$ and $X_i$ is an ancestor of $X_j$ in $\cG$. In the context of mixed graphs,  given the path $\pi = X_{i_1}, \ldots, X_{i_m}$, we say that $X_{i_k}$, $k = 2, \ldots, m-1$, is a \textit{collider on} $\pi$ if the edges between $X_{i_{k-1}}, X_{i_{k}}$ and  $X_{i_{k}}, X_{i_{k+1}}$ both have an arrowhead towards $X_{i_k}$. The triplet $X_{i_{k-1}}, X_{i_k}, X_{i_{k+1}}$ is a \textit{v-structure}

For ease of reference from the main text, we separately define inducing paths and ancestors in directed acyclic graphs.
\begin{definition}[Ancestor]\label{def:ancestor}
Consider a DAG $\cG$ with set of nodes $X$, and $X_i, X_j$ elements of $X$. We say that $X_i$ is an \textit{ancestor} of $X_j$ if there is a directed path from $X_i$ to $X_j$ in the graph, as in $X_i \rightarrow \ldots \rightarrow X_j$.
\end{definition}

\begin{definition}[Inducing path]\label{def:inducing_path}
Consider a DAG $\cG$ with set of nodes $X = V \cup U$, $V,U$ disjoint subsets. We say that a path $\pi$ with endpoints $V_i, V_j$ is an \textit{inducing path relative to} $U$ if every non-endpoint $V_k$ in the path and not in $U$ is both a collider on $\pi$ (i.e. $V_i \ldots \rightarrow V_k \leftarrow \ldots V_j$ appears) and an ancestor of $V_i$ or $V_j$.
% Consider a DAG $\cG$ with set of nodes $X$, and $Y,Z$ disjoint subsets such that $X = Y \cup Z$. We say that a path $\pi$ with endpoints $Y_i, Y_j$ and nodes $Y \cup Z$ in the path, is an \textit{inducing path relative to} $Z$ if every non-endpoint $Y_k$ is both a collider on $\pi$ (i.e. for $Y_k \in Y$ on the path the sequence $Y_i \ldots \rightarrow Y_k \leftarrow \ldots Y_j$ appears) and an ancestor of $Y_i$ or $Y_j$.
\end{definition}

Intuitively, an inducing path relative to $U$ is a path between two variables in $V$ and that cannot be separated by conditioning on any other node in $V$. This makes them natural candidates to encode dependencies between observable random variables $V$ that can not be eliminated by conditioning on subsets of $V$, as they are due to the presence of inducing paths relative to $U$, with $U$ unobserved random variables. 

\begin{example}[Examples of inducing paths] Trivially, a single edge path is an inducing path relative to any set of vertices (as there are no colliders). As another example, let $U = \set{U_1}$, and $V = \set{V_1, V_2, V_3, V_4}$. Let there be a direct path $V_1 \rightarrow V_2 \rightarrow V_3 \rightarrow V_4$, and the path $\pi = V_1 \rightarrow V_2 \leftarrow U_1 \rightarrow V_4$. The path $\pi$ is an inducing path relative to $U$ (as $V_2$ is a collider on $\pi$ and also an ancestor of $V_4$). As final example, let $U = \set{U_1}$, $V = \set{V_1, V_2}$: $V_1 \rightarrow U_1 \rightarrow V_2$ is an inducing path relative to $U$.
\end{example}

One natural way to encode inducing paths and ancestral relationships between variables is represented by maximal ancestral graphs.
\begin{definition}[MAG]\label{def:mag}
    A \emph{maximal ancestral graph} (MAG) is a mixed graph such that:
    \begin{enumerate}
        \item there are no directed cycles and no almost directed cycles;
        \item there are no inducing paths between two non-adjacent nodes.
    \end{enumerate}
\end{definition}

Next, we define conditional independence in the context of graphs.
\begin{definition}[active paths and m-separation]\label{def:msep}Let $\cM$ be a mixed graph with nodes $X$. A path $\pi$ in $\cM$ between $X_i, X_j$ elements of $X$ is \emph{active w.r.t.  $Z\subseteq X\setminus\set{X_i, X_j}$} if:
    \begin{enumerate}
        \item every non-collider on $\pi$ is not in $Z$
        \item every collider on $\pi$ is an ancestors of a node in $Z$.
    \end{enumerate}
    $X_i$ and $X_j$ are said to be \emph{m-separated} by $Z$ if there is no active path between $X_i$ and $X_j$ relative to $Z$. Two disjoint sets of variables $W$ and $Y$ are \textit{m-separated} by $Z$ if every variable in $W$ is m-separated from every variable in $Y$ by $Z$. 
    
    If m-separation is applied to DAGs, it is called \textit{d-separation}. An active path w.r.t. the empty set is simply called active.
\end{definition}
The set of directed acyclic graphs that satisfy the same set of conditional independencies form an equivalence class, known as the \textit{Markov equivalence class}.
\begin{definition}[Markov equivalence class of a DAG]\label{def:mec}
    Let $\cG$ be a DAG with nodes $X$. We denote with $[\cG]$ the \textit{Markov equivalence class} of $\cG$. A DAG $\tilde \cG$ with nodes $X$ is in $[\cG]$ if the following conditions are satisfied for each pair $X_i, X_j$ of distinct nodes in $X$:
    \begin{itemize}
        \item there is an edge between $X_i$, $X_j$ in $\cG$ if and only if there is an edge between $X_i$, $X_j$ in $\tilde \cG$;
        \item let $Z \subseteq X \setminus \set{X_i, X_j}$. Then $X_i \indep^d_\cG X_j | Z \iff X_i \indep^d_{\tilde \cG} X_j | Z$;
        \item let $\pi$ be a path between $X_i$ and $X_j$. $X_k$ is a collider on $\pi$ in $\cG$ if and only if it is a collider on $\pi$ in $\tilde \cG$.
    \end{itemize}
\end{definition}
In summary, graphs in the same equivalence class share the edges up to direction, the set of d-separations, and the set of colliders.

Just as for DAGs, there may be several MAGs that imply the same conditional independence statements.
Denote the \emph{Markov-equivalence class} of a MAG $\cM$ with $[\cM]$: this is represented by a partial mixed graph, the class of graphs that can contain four kinds of edges: $\rightarrow$, $\leftrightarrow$, $\circ \hspace{-1mm} -\hspace{-1.53mm}- \hspace{-.2mm}  \circ$ and $\circ\hspace{-1mm}\rightarrow$, and
hence three kinds of end marks for edges: arrowhead ($>$), tail ($-$) and circle ($\circ$).
\begin{definition}[PAG, Definition 3 of \citet{zhang2008causal}]\label{def:pag}
    Let $[\cM]$ be the Markov equivalence class of an arbitrary MAG $\cM$ . The partial
ancestral graph (PAG) for $[\cM]$, $P_{\cM}$
, is a partial mixed graph such that:
\begin{itemize}
    \item $P_{\cM}$ has the same adjacencies as $\cM$ (and any member of $[\cM]$) does;
    \item A mark of arrowhead is in $P_{\cM}$ if and only if it is shared by all MAGs in $[\cM]$; and
    \item A mark of tail is in $P_\cM$ if and only if it is shared by all MAGs in $[\cM]$.
\end{itemize}
\end{definition}
Intuitively, a PAG represents an equivalence class of MAGs by displaying all common edge marks
shared by all members of the class and displaying circles for those marks that are not in common.

\subsection{Equivalence between m-separation and d-separation}
In this section, we provide a proof for \cref{eq:msep_dsep_equivalence}, stating the equivalence between m-separation and d-separation in a formal sense.
\begin{lemma}[Adapted from \citet{zhang2008causal}]
\label{lem:msep_dsep_equivalence}
    Let $\cG$ be a DAG with nodes $X=V \cup U$, with $V$ and $U$ disjoint sets, and $\cM^\Gx_{V}$ the marginalization of $\Gx$ onto $V$. For any $\set{V_i, V_j} \subseteq V$ and
$V_Z\subseteq V \setminus\{V_i, V_j\}$, the following equivalence holds: 
$$
V_i \indep^d_{\Gx} V_j | V_Z \iff V_i \indep^m_{\cM^\Gx_{V}} V_j | V_Z.
$$
\end{lemma}
\begin{proof}
    The implication $V_i \indep^d_{\Gx} V_j | V_Z \implies V_i \indep^m_{\cM^\Gx_{V}} V_j | V_Z$ is a direct consequence of Lemma 18 from \citet{spirtes1996polynomial}, where we set $S=\emptyset$, since we do not consider selection bias.
    The implication $V_i \indep^d_{\Gx} V_j | V_Z \impliedby V_i \indep^m_{\cM^\Gx_{V}} V_j | V_Z$ follows from Lemma 17 by \citet{spirtes1996polynomial}, again with $S=\emptyset$.
    Note, that in their terminology \enquote{d-separation in MAGs} is what we call m-separation.
\end{proof}

Next, we define the \textit{faithfulness assumption}, a bridge between d-separation and probabilistic conditional independence.
\begin{definition}[Faithfulness]\label{def:faithfulness}
Let $X$ be generated according to the structural causal model \eqref{eq:scm}, with causal graph $\Gx$. 
We say that the density  $p$ of $X$ entailed by the generative SCM is \textit{faithful to the graph $\Gx$} if $X_i \indep X_j | X_Z \implies X_i \indep^d_{\mathcal{G}} X_j | X_Z$ for all $i, j$ and $X_Z \subseteq X$.
\end{definition}

% --------------------------------------
% --------------------------------------
\subsection{Additive noise model identifiability}
We study the identifiability of the additive noise model, reporting results from \citet{peters_2014_identifiability}. We start with a formal definition of identifiability in the context of causal discovery.

\begin{definition}[Identifiable causal model]\label{def:identifiable}
    \looseness-1Let $(X, N, \mathcal{F}, p_N)$ be an SCM with underlying graph $\Gx$ and $p_X$ joint density function of the variables of $X$. We say that the model is \textit{identifiable} from observational data if the distribution $p_X$ can not be generated by a structural causal model with graph $\mathcal{\tilde{G}} \neq \Gx$.
\end{definition}

First, we consider the case of models of two random variables 
\begin{equation}\label{eq:bi_anm_appendix}
X_2 \coloneqq f(X_1) + N, \hspace{1em} X_1 \indep N.
\end{equation}
\begin{condition}[Condition 19 of \citet{peters_2014_identifiability}]\label{cond:peters}
    Consider an additive noise model with structural equations \eqref{eq:bi_anm_appendix}. The triple $(f, p_{X_1}, p_{N})$ does not solve the following differential equation for all pairs $x_1, x_2$ with $f'(x_2)\nu''(x_2 - f(x_1)) \neq 0$:
    \begin{equation}\label{eq:condition}
        \xi''' = \xi''\left(\frac{f''}{f'} - \frac{\nu'''f'}{\nu''} \right) + \frac{\nu''' \nu' f'' f'}{\nu''} - \frac{\nu'(f'')^2}{f'} - 
2\nu'' f'' f' + \nu' f''',
    \end{equation}
    Here, $\xi \coloneqq \log p_{X_1}$, $\nu \coloneqq \log p_{N}$, the logarithms of the strictly positive densities. The arguments $x_2 - f(x_1)$, $x_1$, and $x_1$ of $\nu$, $\xi$ and $f$ respectively, have been removed to improve readability.
\end{condition}

Next, we show that a structural causal model satisfying Condition \ref{cond:peters} is identifiable, as in Definition \ref{def:identifiable}
\begin{theorem}[Theorem 20 of \citet{peters_2014_identifiability}]\label{thm:peters}
    Let $p_{X_1, X_2}$ the joint distribution of a pair of random variables generated according to the model of \cref{eq:bi_anm_appendix} that satisfies Condition \ref{cond:peters}, with graph $\mathcal{G}$. Then, $\mathcal{G}$ is identifiable from the joint distribution.
\end{theorem}

Finally, we show an important fact, holding for identifiable bivariate models, which is that the score $\frac{\partial}{\partial{X_1}} \log p(x_1, x_2)$ is non-constant in $x_1$.
\begin{lemma}[Sufficient variability of the score]\label{lem:score_variability}
Let $p_{X_1, X_2}$ the joint distribution of a pair of random variables generated according to a structural causal model that satisfies Condition \ref{cond:peters}, with graph $\mathcal{G}$. Then:
$$
\frac{\partial}{\partial X_1} (\xi'(x_1) - f'(x_1)\nu'(x_2 - f(x_1))) \neq 0,
$$
for all pairs $(x_1, x_2)$. 
\end{lemma}
\begin{proof}
    By contradiction, assume that there exists $(x_1, x_2)$ such that $\frac{\partial}{\partial X_1} (\xi'(x_1) - f'(x_1)\nu'(x_2 - f(x_1))) = 0$. Then:
    $$
    \frac{\partial}{\partial X_1}\left( 
    \frac{\frac{\partial^2}{\partial X_1^2} \pi(x_1, x_2)}{\frac{\partial^2}{\partial X_1 \partial X_2} \pi(x_1, x_2)} \right) = 0,
    $$
    where $\pi(x_1, x_2) = \log p(x_1, x_2)$. By explicitly computing all the partial derivatives of the above equation, we obtain that equation \ref{eq:condition} is satisfied, which violates Condition \ref{cond:peters}.
\end{proof}

These results guaranteeing the identifiability of the bivariate additive noise model can be generalized to the multivariable case, with a set of random variables $X = \set{X_1, \ldots, X_k}$ that satisfy:
\begin{equation}\label{eq:multi_anm_appendix}
    X_i \coloneqq f_i(X_{\Parents^\Gx_i}) + N_i, i = 1, \ldots, k,
\end{equation}
where $\Gx$ is the resulting causal graph directed and acyclic.
The intuition is that, rather than studying the multivariate model as a whole, we need to ensure that Condition \ref{cond:peters} is satisfied for each pair of nodes, adding restrictions on their marginal conditional distribution.
\begin{definition}[Definition 27 of \citet{peters_2014_identifiability}]\label{def:restricted_scm}
Consider an additive noise model with structural equations \eqref{eq:multi_anm_appendix}. We call this SCM a \textit{restricted additive noise model} if for all $X_j \in X$, $X_i \in X_{\Parents^\Gx_j}$, and all sets $X_S \subseteq X$, $S \subset \N$, with $X_{\Parents^\Gx_j} \setminus \set{X_i} \subseteq X_S \subseteq X_{\Nondesc_j}^\Gx \setminus \set{X_i, X_j}$, there is a value $x_S$ with $p(x_S) > 0$, such that the triplet 
$$
(f_j(x_{\Parents^\Gx_j \setminus \set{i}}, \cdot), p_{X_i|X_S=x_S}, p_{N_j})
$$
satisfies Condition \ref{cond:peters}. Here, $f_j(x_{\Parents^\Gx_j \setminus \set{i}}, \cdot)$ denotes the mechanism function $x_i \mapsto f_j(x_{\Parents^\Gx_j})$. Additionally, we require the noise variables to have positive densities and the functions $f_j$ to be continuous and three times continuously differentiable.
\end{definition}
Then, for a restricted additive noise model, we can identify the graph from the distribution.

\begin{theorem}[Theorem 28 of \citet{peters_2014_identifiability}]
    Let $X$ be generated by a restricted additive noise model with graph $\Gx$, and assume that the causal mechanisms $f_j$ are not constant in any of the input arguments, i.e. for $X_i \in X_{\Parents^\Gx_j}$, there exist $x_i \neq x'_i$ such that $f_j(x_{\Parents^\Gx_j \setminus \set{i}}, x_i) \neq f_j(x_{\Parents^\Gx_j \setminus \set{i}}, x'_i)$. Then, $\Gx$ is identifiable.
\end{theorem}

% --------------------------------------
% --------------------------------------
\subsection{Other auxiliary results}
We state one crucial result that we require for the proof of Proposition \ref{prop:causal_dir_1}.
\begin{lemma}\label{lem:score_lucky_leaf}
    Let $V_j \in V$, and $Z \subset \N$ such that $V_Z = V_{\Parents^\Gx_j} \cup \set{V_j}$. Assume that $V_{\Parents^\Gx_j} \indep \tilde N_j$, $\tilde N_j$ as defined in Equation \eqref{eq:tilde_anm}. Then, the score of $V_j$ with respect to density $p(V_Z)$ satisfies:
    $$
    \partial_{V_j}\log p(V_Z) = \partial_{\tilde N_j}\log p(\tilde N_j).
    $$
\end{lemma}
\begin{proof}
    By Bayes' rule, we have that $p(V_Z) = p(V_j|V_{\Parents_j^\Gx})p(V_{\Parents_j^\Gx})$, such that the log-likelihood can be written as:
    $$
    \log p(V_Z) = \log p(V_j|V_{\Parents_j^\Gx}) + \log p(V_{\Parents_j^\Gx}).
    $$
    Taking the partial derivative w.r.t $V_j$, Equation \eqref{eq:tilde_anm} implies
    $$
    \partial_{V_j} \log p(V_Z) = \partial_{V_j} \log p(V_j|V_{\Parents_j^\Gx}).
    $$
    Note that given $V_j \coloneqq f_j(V_{\Parents^\Gx_j}) + \tilde{N}_j$ and the independence $V_{\Parents^\Gx_j} \indep \tilde{N}_j$, using the change of variable formula for invertible transforms on the density $p(V_j|V_{\Parents^\Gx_j}) $, we find that $p(V_j|V_{\Parents^\Gx_j}) = p(\tilde N_j)$. Moreover, the chain rule of derivatives and $\partial_{V_j} \tilde N_j = 1$ imply $\partial_{V_j} \log p(\tilde N_j) = \partial_{\tilde N_j} \log p(\tilde N_j)$, such that the claim follows.
\end{proof}

\section{Proofs of theoretical results}\label{app:main_proofs}

\subsection{Proof of Proposition \ref{prop:score_msep}}

\begin{proof}\textit{(Proof of Proposition \ref{prop:score_msep})}
Observe that
$$
\frac{\partial^2}{\partial V_i \partial V_j} \log p(v_Z) = 0 \iff V_i \dindep{\Gx} V_j | V_Z \setminus \set{V_i, V_j} \iff V_i \mindep{\cM_V^\Gx} V_j | V_Z \setminus \set{V_i, V_j},
$$
where the first equivalence holds by a combination of the faithfulness assumption with the global Markov property, as explicit in \cref{eq:dsep_criterion}, and the second due to Lemma \ref{lem:msep_dsep_equivalence}. Then, the claim is proven.
\end{proof}

\subsection{Proof of Proposition \ref{prop:nogam}}
% The proof follows the one of Lemma 1 in \citet{montagna23_nogam} for nonlinear additive noise models.
\begin{proof}\textit{(Proof of Proposition \ref{prop:nogam})}
    The forward direction is immediate from \cref{eq:score_general} and $R_j = N_j$, when $X_j$ is a sink (\cref{eq:observed_sink_residual}). Thus, we focus on the backward direction. Given
    $$\mean\left[\left(\mean\left[\partial_{X_j} \log p(X) \mid R_j\right] - \partial_{X_j} \log p(X)\right)^2\right] = 0,
    $$
     we want to show that $X_j$ has no children, which we  prove by contradiction.
     
     Let us introduce a function $q: \R \rightarrow \R$ such that:
     $$
     \mean\left[\partial_{X_j} \log p(X) \mid R_j = r_j\right] = q(r_j),
     $$
     and $s_j: \R^{|X|} \rightarrow \R$,
    $$  
    s_j(x) = \partial_{X_j} \log p(x).
    $$
     
    The mean squared error equal to zero implies that $s_j(X)$ is a constant, once $R_j$ is observed. Formally, under the assumption of $p(x) > 0$ for each $x \in \R^k$, this implies that
    $$
    p(s_j(x) \neq q(R_j)| R_j=r_j) = 0, \hspace{.2mm} \forall x \in \R^k.
    $$
    By contradiction, we assume that  $X_j$ is not a sink, and want to show that $s_j(X)$ is not constant in $X$, given $R_j$ fixed. Let $X_i$ such that $X_j \in X_{\Parents^\Gx_i}$. Being the structural causal model identifiable, there is no model with distribution $p_X$ whose graph has a backward edge $X_i \rightarrow X_j$: thus, the Markov factorization of \cref{eq:markov_factorization} is unique and implies:
    $$
    \partial_{X_j}\log p(X) = \partial_{N_j} \log p(N_j) - \sum_{k \in \Child^\Gx_j}\partial_{X_j}h_k(X_{\Parents_k}) \partial N_k\log p(N_k).
    $$
    We note that, by definition of residual in \cref{eq:observed_residual}, $R_j=r_j$ fixes the following distance:
    $$
    R_j = N_j - \mean[N_j | X_{\setminus X_j}].
    $$
    Hence, conditioning on $R_j$ doesn't restrict the support of $X$: given $R_j=r_j$, for any $x_{\setminus X_j}$ (value of the vector of elements in $X \setminus \set{X_j}$), $\exists n_j$ with $p(n_j>0)$ (by the hypothesis of strictly positive densities of the noise terms) that satisfies
    $$
    r_j = n_j - \mean[N_j | x_{\setminus X_j}].
    $$
    Next, we condition on all the parents of $X_i$, except for $X_j$, to reduce our problem to the simpler bivariate case. Let $S \subset \N$ and $X_S \subseteq {X}$ such that $X_{\Parents^\Gx_i} \setminus \set{X_j} \subseteq X_S \subseteq X_{\Nondesc^\Gx_i} \setminus \set{X_i, X_j}$, and consider $x_S$ such that $p(x_S > 0)$. Let $X_{\Parents^\Gx_i} = x_{\Parents^\Gx_i}$ hold under $X_S = x_S$.  We define $X_{j _\mathlarger{|_{\mathlarger{x_s}}}} \coloneqq X_j | (X_S=x_S)$, and similarly $X_{|_{\mathlarger{x_s}}}\coloneqq X | (X_S=x_S)$. Being the SCM a restricted additive noise model, by Definition \ref{def:restricted_scm}, the triplet $(g_i, p_{X_{j _\mathlarger{|_{\mathlarger{x_s}}}}}, p_{N_i})$ satisfies Condition \ref{cond:peters}, where $g_i(x_j) = h_i(x_{\Parents^\Gx_i \setminus \set{X_j}}, x_j)$.
    Consider $X_i = x_i$, and the pair of values $(x_j, x_j^*)$ such that $x_j \neq x_j^*$ and 
    \begin{align*}
    &\nu''_{N_i}(x_i - g_i(x_j))g'_i(x_j) \neq 0,\\
    &\nu''_{N_i}(x_i - g_i(x^*_j))g'_i(x^*_j) \neq 0,
    \end{align*}
    where we resort to the usual notation $\nu_{N_i} \coloneqq \log p_{N_i}$. By Lemma \ref{lem:score_variability}, $(x_i, x_j)$ and $(x_i, x^*_j)$ satisfy:
    \begin{align*}
    &\partial_{X_j}(\xi'(x_j) - \nu'_{N_i}(x_i - g_i(x_j))g'_i(x_j)) \neq 0,\\
    &\partial_{X_j} (\xi'(x^*_j) - \nu'_{N_i}(x_i - g_i(x^*_j))g'_i(x^*_j)) \neq 0,
    \end{align*}
    where $\xi \coloneqq \log p_{X_{j _\mathlarger{|_{\mathlarger{x_s}}}}}$. Thus, we can fix $x_j$ and $x_j^*$ (which are arbitrarily chosen) such that
    \begin{equation}\label{eq:score_diff}
        \partial_{X_j}(\xi'(x_j) - \nu'_{N_i}(x_i - g_i(x_j))g'_i(x_j)) - \partial_{X_j} (\xi'(x^*_j) - \nu'_{N_i}(x_i - g_i(x^*_j))g'_i(x^*_j)) \neq 0.
    \end{equation}
    Fixing $X_{\mid_{\mathlarger{x_S}, \mathlarger{x_j}}} = x$ and $X_{\mid_{\mathlarger{x_S}, \mathlarger{x^*_j}}} = x^*$, where the two values differ only in their j-$th$ component, we find the following difference:
    $$
    s_j(x) - s_j(x^*) = \partial_{X_j}(\xi'(x_j) - \nu'_{N_i}(x_i - g_i(x_j))g'_i(x_j)) - \partial_{X_j} (\xi'(x^*_j) - \nu'_{N_i}(x_i - g_i(x^*_j))g'_i(x^*_j)),
    $$
    which is different from $0$ by \cref{eq:score_diff}. This contradicts the fact that the score $s_j$ is constant once $R_j$ is fixed, which proves our claim.
\end{proof} 

\subsection{Proof of Proposition \ref{prop:causal_dir_1}}\label{app:main_prop_proof}
We concentrate on the statement of \cref{eq:main_prop_nonlinear}, given that the proof of \cref{eq:main_prop_linear} follows a similar template. We separately analyze their backward and forward directions. The following proofs use several ideas from the demonstration of Proposition \ref{prop:nogam}.

\begin{proof}[Proof of \cref{eq:main_prop_nonlinear}, backward direction] 
 Given $ V_{\Parents^\Gx_j} \indep^d_{\Gx} U^j  \land V_i \in V_{\Parents^\Gx_j}$, we want to show that there exists $V_Z \subseteq V$, $\set{V_i, V_j} \subseteq V_Z$, such that:
    \begin{equation*}
        \mean[\partial_{V_j} \log p(V_Z) - \mean[\partial_{V_j} \log p(V_Z) | R_j(V_Z)]]^2 = 0.
    \end{equation*}
Let $V_Z = V_{\Parents_j^\Gx} \cup \set{V_i, V_j}$. We will show that this is indeed the right choice for $V_Z$. By Lemma \ref{lem:score_lucky_leaf}, the score of $V_j$ is:
\begin{equation*}
        \partial_{V_j} \log p(V_Z) = \partial_{\tilde N_j} \log p(\tilde N_j).
\end{equation*}
Further, by \cref{eq:latent_residual_sink} we know that
$$
R_j(V_Z) = \tilde N_j + c,
$$
where $c = - \mean[\tilde N_j]$ is a constant. It follows that the least square estimator of the score of $V_j$ from $R_j(V_Z)$ satisfies the following equation:
$$
\mean[\partial_{V_j} \log p(V_Z) | R_j(V_Z)] = \mean[\partial_{V_j} \log p(\tilde N_j) | \tilde N_j] = \partial_{V_j} \log p(\tilde N_j),
$$
where the first equality holds because $\mean[\cdot | \tilde N_j] = \mean[\cdot | \tilde N_j + c]$. Then, we find
$$
\mean[\partial_{V_j} \log p(V_Z) - \mean[\partial_{V_j} \log p(V_Z) | R_j(V_Z)]]^2 = \mean[\partial_{V_j} \log p(\tilde N_j) - \partial_{V_j} \log p(\tilde N_j)]^2 = 0,
$$
which is exactly our claim. 
\end{proof}

\begin{proof}[Proof of \cref{eq:main_prop_nonlinear}, forward direction] Given that there is $V_Z \subseteq V$, $\set{V_i,V_j} \subseteq V_Z$, such that $\mean[\partial_{V_j} \log p(V_Z) - \mean[\partial_{V_j} \log p(V_Z) | R_j(V_Z)]]^2 = 0$, we want to show that $V_{\Parents^\Gx_j} \indep^d_{\Gx} U^j \land V_i \in V_{\Parents^\Gx_j}$. We prove the contrapositive statement, hence, assuming that $V_i$ is connected to $V_j$ in the marginal MAG and that $V_{\Parents^\Gx_j} \notdindep{\Gx} U^j \lor V_i \not\in V_{\Parents^\Gx_j}$, we want to show that for each $V_Z \subseteq V$ with $\set{V_i, V_j} \subseteq V_Z$, the following holds:
    \begin{equation}\label{eq:proof_mse}
        \mean[\partial_{V_j} \log p(V_Z) - \mean[\partial_{V_j} \log p(V_Z) | R_j(V_Z)]]^2 \neq 0.
    \end{equation}
    Let us introduce $q : \R \rightarrow \R$ such that:
    $$
    \mean[\partial_{V_j} \log p(V_Z) | R_j(V_Z) = r_j] = q(r_j),
    $$
    and further define:
    $$
    s_j(V_Z) = \partial_{V_j} \log p(V_Z).
    $$
    Having the mean squared error in \cref{eq:proof_mse} equals zero implies that $s_j(V_Z)$ is a constant, once $R_j(V_Z)$ is observed, meaning that $p(s_j(V_Z) \neq q(R_j(V_Z))| R_j(V_Z)) = 0$.
    Thus, the goal of the proof is to show that there are uncountable values of $V_Z$ such that the score is not a constant once $R_j(V_Z)$ is fixed. To do that, we assume 
    that conditioning on $R_j(V_Z)$ doesn't restrict the support of $V_Z$, meaning that the support of $V_Z$ and the support of $V_Z | R_j(v_z)$ are the same. We derive the consequences of this assumption (and later prove that it holds): consider $v_Z$, $v^*_Z \in \operatorname{supp}(V_Z|R_j(V_Z) = r_j)$ (support of $V_Z|R_j(V_Z)=r_j$), taken from the set of uncountable values such that the score $s_j$ function is not a constant, meaning that $s_j(v_Z) \neq s_j(v_Z^*)$. Given that different $v_Z$ and $v_Z^*$ are selected from an uncountable subset of the support, we conclude that the conditional score $s_j | (R_j(V_Z) = r_j) := \partial_{V_j}\log p(V_Z|R_j(V_Z) = r_j)$ is not a constant for at least an uncountable set of points, such that the claim follows.

    \vspace{1em}
    \noindent To conclude the proof, we need to show that for each $r_j$, $ \operatorname{supp}(V_Z|R_j(V_Z) = r_j) =  \operatorname{supp}(V_Z)$. First, we consider the case where $V_{\Parents_j^\mathcal G} \not \subset V_Z$. By \cref{eq:latent_residual}, $R_j(V_Z) = r_j$  fixes the distance 
    \begin{equation}\label{eq:proof_util}
    r_j = \tilde N_j + \left( f_j(V_{\Parents_j^\mathcal G}) - \mean[f(V_{\Parents_j^\mathcal G}) + \tilde N_j|V_{Z\setminus \set{j}}] \right) =: \tilde N_j + h(V_{Z\setminus \set{j}}, V_{\Parents_j^\cG}),
    \end{equation}
    $h$ newly defined function.
    By assumption of positive density of the noise $N_j$ on the support $\R$, for each $v_Z \in \operatorname{supp}(V_Z)$, there is $\tilde n_j \in \R$ such that $p_{\tilde N_j}(n)>0$ and 
    $$
        r_j = \tilde n_j + \left( f_j(V_{\Parents_j^\mathcal G}) - \mean[V_{\Parents_j^\mathcal G} + \tilde N_j|v_{Z\setminus \set{j}}] \right)
    $$
    is true (in fact, note that by fixing $r_j$ we have a system of one equation and two unknowns, which has a solution for each value taken by $h(V_Z, V_{\Parents_j^\Gx})$). Then, the claim is proven for $V_{\Parents_j^\mathcal G} \not \subset V_Z$. We now consider $V_{\Parents_j^\mathcal G} \subset V_Z$, and further separate the proof in subcases. Note that \cref{eq:proof_util} becomes:
    \begin{equation}\label{eq:proof_util_pt2}
        r_j = \tilde N_j - \mean[ \tilde N_j|V_{Z\setminus \set{j}}]
    \end{equation}
    To finalize our proof, it sufficient to show that $\mean[\tilde N_j | V_{Z \setminus \set{j}}] \neq \mean[\tilde N_j]$: in fact, if this is the case, we have that for each $v_Z \in \operatorname{supp}(V_Z)$ there is $\tilde n_j$ such that $p(\tilde n_j > 0)$ and \cref{eq:proof_util_pt2} is satisfied - which in turns implies our claim that conditioning on $R_j(V_Z)$ does not limit the support of $V_Z$.

    \noindent \paragraph{Case 1.} Consider $U^j \notdindep{\cG} V_{\Parents_j^\cG}$. Together with $V_{\Parents_j^\cG} \in V_Z$, it implies that $\mean[\tilde N_j | V_{Z \setminus \set{j}}] \neq \mean[\tilde N_j]$. proving our claim.

    \noindent \paragraph{Case 2.}  Consider $V_i \not\in V_{\Parents_j^\cG}$. We know that $V_i$ adjacent to $V_j$ in the MAG $\mathcal M_{V_Z}^\Gx$, meaning there is an inducing path through $U$ between the two variables. As we shall see, this implies that $V_i \notdindep{\cG} U^j | V_{Z \setminus \set{i,j}}$, which gives $\mean[\tilde N_j | V_{Z \setminus \set{j}}] \neq \mean[\tilde N_j]$. 
    \begin{enumerate}[(i)]
        \item For $V_i \in V_{\Anc_j^\cG}$, there must be a direct path $\pi = V_i \rightarrow \ldots \rightarrow X_{j-1} \rightarrow V_j$ where either $X_{j-1} \in U^j$, or $X_{j-1} = V_k \in V$ and is descendant of some $\tilde U \in U^j$, i.e. there is a path $V_i \rightarrow \ldots \rightarrow V_k \leftarrow \tilde U$: if this was not the case, then $V_i \dindep{\cG} V_j | V_k$, but this contradicts the definition of inducing path.
        \item For $V_i \in V_{\Desc_j^\cG}$, given $V_i \notdindep{\cG} V_j | V_{Z \setminus \set{i,j}}$, clearly $V_i$, $U^j$ not d-separated.
        \item For $V_i \leftrightarrow V_j$, there is at least one latent common cause $\tilde U$ such that there exists $\pi = V_i \leftarrow \ldots \leftarrow \tilde U \rightarrow \ldots \rightarrow V_j$. Given that no observable variable can block every active path between $V_i, V_j$, then no observed variables can block all paths of the form of $\pi$ (i.e. with a latent common cause), hence, there is one such path where all variables are hidden, meaning that $V_i \notdindep{\cG} U^j | V_{Z \setminus \set{i,j}}$.
    \end{enumerate} 
\end{proof}

\noindent The proof of \cref{eq:main_prop_linear} follows from minor adjustments of the above demonstration, hence we omit it.

\section{Miscellanea}
\subsection{The connection between NoGAM on linear models and existing literature}\label{app:nogam_linear}
In this section, we elaborate on our comment that the findings of Proposition \ref{prop:nogam}, relative to the identifiability of the causal order with the score even with linear mechanisms, \textit{are not surprising, in light of previous
literature}. While \citet{montagna23_nogam} limits themselves to nonlinear additive noise models, previously \citet{ghoshal2018learning} showed that the causal order can be identified by the \textit{precision matrix} of the data, namely the inverse of the covariance matrix. In case of $X$ generated by a linear additive noise model, the score of $X=x$ satisfies $\nabla \log p(x) = \Theta x $, where $\Theta$ denotes the precision matrix. Despite \citet{montagna23_nogam} and \citet{ghoshal2018learning} differ in how they exploit the score for estimation of the causal order, the novelty of our result in Proposition \ref{prop:nogam} is showing that findings in \citet{montagna23_nogam} are less restrictive than they prove in their original paper, limiting their theoretical guarantees with the assumption of nonlinear mechanisms. Instead, an immediate generalization of the estimation of the causal order from the precision matrix in the setting of nonlinear ANMs is beyond the scope of this paper. 

\subsection{On the difference of identifiability guarantees between Proposition \ref{prop:causal_dir_1} and FCI}
\label{app:difference_adascore_fci}
As we have discussed before, the semantics of the edges that Proposition \ref{prop:causal_dir_1} can orient is different from edges in a MAG or PAG.
The condition $V_{\Parents^\Gx_j} \indep^d_{\Gx} U^j \land V_i \in V_{\Parents^\Gx_j}$ in Proposition \ref{prop:causal_dir_1} means that the node $V_i$ is a direct parent of $V_j$ in the original underlying DAG $\cG$.
On the other hand, in a MAG or PAG, a directed egde $V_i\to V_j$ indicates an ancestral relationship, meaning that in the underlying DAG there is a directed path from $V_i$ to $V_j$.
So in this sense, the information conveyed by a directed edge differs.
But further, since Proposition \ref{prop:causal_dir_1} and FCI rely on different assumptions, also the kinds of structure that can be identified differ.
We will illustrate this in the following examples.

First, it is important to note that all orientation rules given by \citet{zhang2008completeness} boil down to the presence of colliders in the resulting PAG.
Therefore, in the absence of these, FCI cannot orient any edges in contrast to the orientation rule presented in Proposition \ref{prop:causal_dir_1}.
\begin{example}[Direct edges that FCI cannot detect]
\label{ex:direct_edge_fci}
    Suppose the underlying distribution is generated by the DAG in Figure \ref{fig:x_causing_y}.
    Since there is no unshielded triplet and thus, no collider structure that could be detected by FCI, the PAG that FCI will output reads $X\lchead - \rchead  Y \lchead - \rchead Z$.
    Since $X$ is a direct parent of $Y$ and $Y$ has no unobserved parents, Proposition \ref{prop:causal_dir_1} can orient this edge and thus, Proposition \ref{prop:causal_dir_1} will indicate the directed edge.
    Similarly, for $Y\to Z$.
\end{example}

\begin{example}[Partially directed edges]
\label{ex:partial_edges_fci}
    Now suppose the underlying distribution is generated by the DAG in Figure \ref{fig:single_collider} and no variable is unobserved.
    In this case there is a collider that can be detected by FCI.
    Still, the edges can only be oriented partially, i.e. the output reads $X\lchead\to Y \leftarrow\rchead Z$.
    Particularly, this means there could still be arrowheads towards $X$ and $Z$ indicating, that they are connected by a confounder and $X$ (or $Z$) is not an ancestor of $Y$.
    Again, Proposition \ref{prop:causal_dir_1} can orient these edges.
    This entails the assertion that $X$ and $Z$ are indeed direct parents of $Y$.
\end{example}
\begin{figure}
    \centering
    \subfigure[Edges that FCI cannot orient]{
    \label{fig:x_causing_y}\qquad
        \begin{tikzpicture}[node distance=2cm, baseline=(current bounding box.north), auto]
				\node(x){$X$};
				%\node[draw, circle, dotted](u)[above right of=x]{$U$};
				\node(y)[right of=x]{$Y$};
                \node(z)[right of=y]{$Z$};
                \draw[->] (x) edge node{} (y);
				\draw[->] (y) edge node{} (z);
			\end{tikzpicture}\qquad
    }
    \subfigure[Edges that FCI can only orient partially]{
    \label{fig:single_collider}\qquad
        \begin{tikzpicture}[node distance=2cm, baseline=(current bounding box.north), auto]
				\node(x){$X$};
				%\node[draw, circle, dotted](u)[above right of=x]{$U$};
				\node(y)[right of=x]{$Y$};	
                \node(z)[right of=y]{$Z$};	
				\draw[->] (x) edge node{} (y);
                \draw[<-] (y) edge node{} (z);
			\end{tikzpicture}\qquad\quad
    }

    \caption{Examples where the orientations of Proposition \ref{prop:causal_dir_1} is more informative than FCI.}
    \label{fig:enter-label}
\end{figure}
On the other hand, the following examples show how the orientation rules of FCI can still be applied in the presence of hidden confounders and mediators, which is not the case for Proposition \ref{prop:causal_dir_1}.
\begin{example}[Unoriented edges]
    Suppose the underlying distribution is generated by the DAG in Figure \ref{fig:x_causes_y_confounded} and the variable $U_1$ and $U_2$ are unobserved.
    Like in Example \ref{ex:direct_edge_fci}, FCI outputs 
    \begin{displaymath}
    X\lchead - \rchead Y \lchead - \rchead Z.
    \end{displaymath}
    But here Proposition \ref{prop:causal_dir_1} cannot orient the given edges either, due to the unobserved mediators.
\end{example}

\begin{example}[Collider with unobserved nodes]
    Suppose the underlying distribution is generated by the DAG in Figure \ref{fig:single_collider_confounded} and the variable $U_1$ and $U_2$ are unobserved.
    Similarly to Example \ref{ex:partial_edges_fci}, FCI outputs the partialy oriented collider $X\lchead\to Y \leftarrow\rchead Z$.
    But here Proposition \ref{prop:causal_dir_1} cannot orient the given edges, due to the unobserved mediators.
\end{example}

\begin{figure}
    \centering
    \subfigure[Edges that cannot be oriented by FCI or Proposition \ref{prop:causal_dir_1}.]{
    \label{fig:x_causes_y_confounded}
        \begin{tikzpicture}[node distance=2cm, baseline=(current bounding box.north), auto]
				\node(x){$X$};
				\node[draw, circle, dotted](u1)[right of=x]{$U_1$};
				\node(y)[right of=u1]{$Y$};
				\node[draw, circle, dotted](u2)[right of=y]{$U_2$};
                \node(z)[right of=u2]{$Z$};
                \draw[->] (x) edge node{} (u1);
                \draw[->] (u1) edge node{} (y);
				\draw[<-] (u2) edge node{} (y);
                \draw[->] (u2) edge node{} (z);
			\end{tikzpicture}
    }\qquad
    \subfigure[Edges that Proposition \ref{prop:causal_dir_1} cannot orient but FCI partially.]{
    \label{fig:single_collider_confounded}\qquad
        \begin{tikzpicture}[node distance=2cm, baseline=(current bounding box.north), auto]
				\node(x){$X$};
				\node[draw, circle, dotted](u1)[right of=x]{$U_1$};
				\node(y)[right of=u1]{$Y$};
				\node[draw, circle, dotted](u2)[above right of=y]{$U_2$};
                \node(z)[below right of=u2]{$Z$};
                \draw[->] (x) edge node{} (u1);
                \draw[->] (u1) edge node{} (y);
				\draw[->] (u2) edge node{} (y);
                \draw[->] (u2) edge node{} (z);
			\end{tikzpicture}\qquad
    }

    \caption{Examples where the output of FCI is more informative than Proposition \ref{prop:causal_dir_1} due to hidden variables.}
    \label{fig:enter-label}
\end{figure}

The relevance of hidden confounders for our orientation criterion begs the question how the edges detected by Proposition \ref{prop:causal_dir_1} are related to \emph{visible} edges \citep{zhang2008causal}.
%Recall, that an edge $X\to Y$ in a MAG is visible if there is a nodes $Z$ that is not adjacent to $Y$ and either there is and edge $Z\to X$ or a path from $C$ to $A$ where every node on the path is a collider and a parent of $Y$ and the last edge on the path has an arrowhead towards $X$.
Intuitively, a visible edge is an edge that cannot be confounded.
But again, the rules of FCI and our new rule are in no strict hierarchy.
First note, that e.g. Example \ref{ex:direct_edge_fci} already shows that not every edge that is orientable by Proposition \ref{prop:causal_dir_1} is visible.
But there is also no inclusion the other way around, as the following examples show.

\begin{example}[Collider with unobserved nodes]
\label{ex:visible_edge_fci}
    Assume the underlying distribution is generated by the DAG in Figure \ref{fig:visible_edge_fci}.
    The edges $X_1\lchead\to Y$ and $X_2\lchead\to Y$ and $Y\to Z$ can be oriented by FCI.
    Further, $Y\to Z$ is visible, due to the missing link between e.g. $X_1$ and $Z$ together with the arrowhead into $Y$ on the edge between $X_1$ and $Y$.
    And again, Proposition \ref{prop:causal_dir_1} can orient $Y\to Z$ as it is a direct, unconfounded link.
\end{example}

\begin{example}[Collider with unobserved nodes]
\label{ex:visible_edge_mediator}
    Assume the underlying distribution is generated by the DAG in Figure \ref{fig:visible_edge_fci_mediator} with unobserved node $U$.
    Like in Example \ref{ex:visible_edge_fci}, $Y\to Z$ is orientable by FCI and is visible.
    But due to the unobserved mediator, the link is not direct anymore and Proposition \ref{prop:causal_dir_1} cannot orient $Y\to Z$.
\end{example}
Summing up these examples, there is no strict inclusion or hierarchy between edges identified by FCI and by Proposition \ref{prop:causal_dir_1}.

\begin{figure}
    \centering
    \subfigure[$Y\to Z$ is visible and orientable by Proposition \ref{prop:causal_dir_1}.]{
    \label{fig:visible_edge_fci}\qquad
        \begin{tikzpicture}[node distance=2cm, baseline=(current bounding box.north), auto]
				\node(x1){$X_1$};
				\node(y)[below right of=x1]{$Y$};
				\node(x2)[below left of=y]{$X_2$};
                \node(z)[right of=y]{$Z$};
                \draw[->] (x1) edge node{} (y);
                \draw[->] (x2) edge node{} (y);
				\draw[->] (y) edge node{} (z);
			\end{tikzpicture}\qquad
    }\qquad
    \subfigure[The edge $Y\to Z$ in the respective PAG is visible and cannot be oriented by Proposition \ref{prop:causal_dir_1}.]{
    \label{fig:visible_edge_fci_mediator}\qquad
    \begin{tikzpicture}[node distance=2cm, baseline=(current bounding box.north), auto]
				\node(x1){$X_1$};
				\node(y)[below right of=x1]{$Y$};
				\node(x2)[below left of=y]{$X_2$};
                \node[draw, circle, dotted](u1)[right of=y]{$U$};
                \node(z)[right of=u1]{$Z$};
                \draw[->] (x1) edge node{} (y);
                \draw[->] (x2) edge node{} (y);
				\draw[->] (y) edge node{} (u1);
                \draw[->] (u1) edge node{} (z);
			\end{tikzpicture}\qquad
    }

    \caption{Visible edges are different from the edges that can be oriented via Proposition \ref{prop:causal_dir_1}.}
    \label{fig:enter-label}
\end{figure}

\subsection{On the difference of identifiability guarantees between Proposition \ref{prop:causal_dir_1} and CAM-UV}
\label{sec:difference_adascore_camuv}

As we have noted before, the CAM-UV relies on stronger structural assumptions than our proposed criterion in Proposition \ref{prop:causal_dir_1}.
In this section we want to elaborate on the subtle differences in the graphical structures that can be recovered via these criteria.
In Section \ref{subsec:adascore_vs_camuv} we highlight further differences that are not due to Proposition \ref{prop:causal_dir_1}.
First note, that Proposition \ref{prop:causal_dir_1} cannot orient an edge whenever CAM-UV also cannot.
To this end, recall that \citet{maeda21_causal} define an \emph{unobserved backdoor path} between nodes $V_i\in V$ and $V_j\in V$ as a path $V_i \leftarrow U_k \leftarrow \cdots \leftarrow X_l \rightarrow \cdots, \rightarrow U_m \rightarrow V_j$, where $U_i, U_m\in U$ and all other nodes in $X$.
Similarly, they define an \emph{unobserved causal path} (UCP) between $V_i$ and $V_j$ to be a path $V_i\to \cdots \to U_k\to V_j$, where $U_k\in U$ and all other nodes are in $X$.
In their work, they can orient an edge iff there is no UCP or UBP between $V_i$ and $V_j$.
\begin{lemma}
    Let $V_i, V_j\in V$ such that there is an edge between them in $\cM^\cG_V$ and there is a UBP or UCP between them.
    Then we have $\mean\left[\left(\partial_{V_j} \log p(V_Z) - \mean[\partial_{V_j} \log p(V_Z) | R_j(V_Z)]\right)^2\right] \neq 0 $ for all $V_Z \subseteq V$ with $\set{V_i, V_j} \in V_Z$.
\end{lemma}
\begin{proof}
    Suppose $V_i\in \Parents^\cG_j$ and there is a UBP between $V_i$ and $V_j$. 
    Then there is an unobserved parent $U_m$ of $V_j$ that is connected to $V_i$ via this backdoor path, so we get $V_{\Parents^\cG_j} \notdindep{\cG} U^j$. 
    Now suppose there is an UCP from $V_i$ to $V_j$.
    If $V_i\in \Parents^\cG_j$, we have again $V_{\Parents^\cG_j} \notdindep{\cG} U^j$, since there is an unobserved parent $U_k$ of $V_j$ that is a descedant of $V_i$.
    The rest follows from Proposition \ref{prop:causal_dir_1}.
\end{proof}
On the other hand, there can be cases where CAM-UV recovers a direct edge, while Proposition \ref{prop:causal_dir_1} does not indicate one.
\begin{example}
    Suppose the underlying distribution is generated by the DAG in Figure \ref{fig:prop4_fails} with $U$ being unobserved.
    There is no UCPs, as the path from $X$ to $Y$ contains no unobserved nodes and no UBP, since $X$ has an observed parent on the only backdoor path between $X$ and $Y$, namely $Z$.
    Therefore, CAM-UV can orient this edge.
    Yet, Proposition \ref{prop:causal_dir_1} cannot, as we have $U\notdindep{\cG} X$.
    \begin{figure}
        \centering
        \begin{tikzpicture}[node distance=2cm, baseline=(current bounding box.north), auto]
				\node(x){$X$};
				%\node[draw, circle, dotted](u)[above right of=x]{$U$};
				\node(y)[right of=x]{$Y$};
                \node(z)[above of=x]{$Z$};
                \node(l)[above of=y, draw, circle, dotted]{$U$};
                \draw[->] (x) edge node{} (y);
				\draw[->] (z) edge node{} (x);
                \draw[->] (l) edge node{} (y);
                \draw[->] (l) edge node{} (z);
			\end{tikzpicture}
        \caption{In this graph we have $U\notdindep{\cG} X$ and Proposition \ref{prop:causal_dir_1} cannot orient the edge $X\to Y$. But CAM-UV can orient this edge, as there is no unobserved confounding path or unobserved backdoor path between $X$ and $Y$.}
        \label{fig:prop4_fails}
    \end{figure}
\end{example}

\subsection{Kernel ridge regression estimation of the residuals in AdaScore}\label{app:krr_residuals}
In AdaScore, the residuals of Equation \eqref{eq:latent_residual}, which we recall to be defined as 
$$
R_k(V_Z) \coloneqq V_k -\mean[V_k \mid V_{Z \setminus \set{k}}]
$$
for $V_Z \subseteq V$ and $V_k \in V$,
are estimated via kernel-ridge regression (KRR). In particular, notice that for any pair of random variables $X,Y$, the least squares estimator is $\mean[Y|X]$. It is thus immediate to see that, for $V_Z = V_{\Parents_j^\Gx} \cup \set{V_j}$, $V_{\Parents_j^\Gx} \dindep{\Gx} U^j$ (the unobserved parents of $V_j$), given $\hat V_k \coloneqq \mean[V_k | V_Z]$ (the unregularized least squares estimator), $R_j(V_Z) = \tilde N_j - \mean[\tilde N_j] = V_k - \hat V_k$. The least squares algorithm is thus the right choice to estimate the required residuals. The use of feature maps (as in KRR) is justified by the potential nonlinear mechanisms in structural equations, while the use of regularization is necessary to avoid overfitting and vanishing residuals. 

\section{Algorithm}
\label{app:algorithm}
\subsection{Detailed description of our algorithm}
%\FM{For Philipp: careful to notation: at the beginning, I setup the following
%\begin{itemize}
%    \item $X$ be a set of random variables nodes of the graph. So now $X = \set{X_i}_i^{\ldots}$.
%    \item In the remainder of the paper, we adopt the following notation: given a set of random variables $Y=\set{Y_1, \ldots, Y_n}$ and a set of indices $Z \subset \N$, then $Y_Z = \set{Y_i | i \in Z, Y_i \in Y}$.
%\end{itemize}
%So, you don't want to have things like $S \in X$, but instead $X_S \in X$. Additionally, if I say e.g. "In Proposition \ref{prop:pag_identifiable}" I use uppercase for the first letter, as if it is a proper noun. Same for Algorithm, Lemma, ...
%}
In Proposition \ref{prop:score_msep} we have seen that score-matching can detect $m$-separations and therefore the skeleton of the PAG describing the data.
If one is willing to make the assumptions required for Proposition \ref{prop:causal_dir_1}  it could be desirable to use this to orient edges, as Proposition \ref{prop:causal_dir_1} offers additional causal insights that we have discussed in Appendix \ref{app:difference_adascore_fci}.
Therefore, one could simply find the skeleton of the PAG using the fast adjacency search \citep{spirtes2000_cauastion} and then orient the edges by applying Proposition \ref{prop:causal_dir_1} on every subset of the neighbourhood of every node.
%As we will show later, this will output a subgraph of the output of the CAM-UV algorithm, containing the same directed edges.
This would yield a very costly algorithm.
%Yet, this na\"ive algorithm has a worse asymptotic runtime than CAM-UV.
But if we make the assumptions required to orient edges with Proposition \ref{prop:causal_dir_1} we can do a bit better.
In \cref{algo:scam} we present an algorithm that still has the same worst case runtime but runs polynomially in the best case.
The main intuition is that we iteratively remove irrelevant nodes in the spirit of the original SCORE algorithm \citep{rolland22_score}.
To this end, we first check if there is any unconfounded sink (i.e. a sink that is not connected to any edge which cannot be oriented using Proposition \ref{prop:causal_dir_1}) if we consider the set of all remaining variables.
If there is one, we can orient its parents and ignore it afterwards.
If there is no such node, we need to fall back to the procedure proposed above, i.e. we need to check the condition of Proposition \ref{prop:causal_dir_1} on all subsets of the neighbourhood of a node, until we find no node with a direct outgoing edge.
In Proposition \ref{prop:algo_correct} we show that this way we do not fail to orient an edge or fail to remove any adjacency.
In the following discussion, we will use the  notation
\[
\delta_i(V_Z) \coloneqq 
\mean[\partial_{V_j} \log p(V_Z) - \mean[\partial_{V_j} \log p(V_Z) | R_j(V_Z)]]^2,
\]
for the second residual from Proposition \ref{prop:causal_dir_1} and also 
\[
\delta_{i, j}(V_Z) \coloneqq\frac{\partial^2}{\partial V_i \partial V_j} \log p(V_Z)
\]
for the cross-partial derivative, where $V_i, V_j\in V$ and $Z\subseteq V$.
\begin{algorithm}[p]
\caption{AdaScore Algorithm}
\label{algo:scam}
\SetKwFunction{AdaScore}{AdaScore}
\SetKwProg{Fn}{Procedure}{:}{}
\SetKwFor{For}{for}{do}{end for}
\SetKwFor{While}{while}{do}{end while}
\SetKwComment{Comment}{$\triangleright$\ }{}

\Fn{\AdaScore{$p, V_1, \dots, V_d$}}{
    $S \gets \{V_1, \dots, V_d\}, E \gets \{\}$ \Comment*[r]{Init remaining nodes and edges}
    
    \For{$V_i \in S$}{
        $B_i \gets \{V_1, \dots, V_d\}$ \Comment*[r]{Neighbourhoods}
    }

    \While(\Comment*[f]{While nodes remain}){$S \neq \emptyset$}{ 
        \If(\Comment*[f]{Unconfounded sink}){$\exists V_i \in S : \delta_{i}(V_S)=0$}{ 
            $S \gets S \setminus \{V_i\}$\;
            $E \gets E \cup \{V_j \to V_i : \delta_{i,j}(V_S) \neq 0\}$ \Comment*[r]{Add edges like DAS}
        }
        \Else{
                $V_i \gets \min_{V_j\in S} \delta_j(V_S)$\;\\
                \For{$V_j \in \{V_k\in B_i: \min_{S'\subseteq B_i} \delta_{i,k}(V_{S'\cup \{V_i, V_k\}}) = 0\}$}{ 
                        $B_i \gets B_i \setminus \{V_j\}$\;\Comment*[f]{Prune neighbourhoods}\\
                        $B_j \gets B_j \setminus \{V_i\}$\;
                }
                \For(\Comment*[f]{Orient edges in $B_i$}){$V_j \in B_i$}{ 
                    $m_i = \min_{S' \subseteq B_i} \delta_{i}(V_{S' \cup \{V_i, V_j\}})$\;\\
                    $m_j = \min_{S' \subseteq B_j} \delta_{j}(V_{S' \cup \{V_i, V_j\}})$\;
                    
                    \If{$m_i = 0 \land m_j \neq 0$}{
                        $E \gets E \cup \{V_j \to V_i\}$\;
                    }
                    \ElseIf{$m_i \neq 0 \land m_j = 0$}{
                        $E \gets E \cup \{V_i \to V_j\}$\;
                    }
                    \Else{
                        $E \gets E \cup \{V_i - V_j\}$\;
                    }
                }
                \If{$\exists V_j \in B_i : (V_i \to V_j) \in E$}{
                    \textbf{continue with } $V_j$\;
                }
                \Else{
                    $S \gets S \setminus \{V_i\}$ \Comment*[r]{Remove $V_i$}
                    \textbf{break}\;
                }
        }
    } 

    \For(\Comment*[f]{Prune undirected edges}){$V_i - V_j \in E$}{ 
        \If{$\min_{S' \subseteq B_i} \delta_{i,j}(V_{S' \cup \{V_i, V_j\}}) = 0 \lor \min_{S' \subseteq B_j} \delta_{i,j}(V_{S' \cup \{V_i, V_j\}}) = 0$}{
            $E \gets E \setminus \{V_i - V_j\}$\;
        }
    }
    \Return{$E$}\;
}
\end{algorithm}

\begin{proposition}[Correctness of algorithm]
\label{prop:algo_correct}
Let $X=V\cup U$ with $V\cap U=\emptyset$ be generated by a restricted additive noise model of the form from \cref{eq:latent_scm_v}.
    Let $\cG$ be the causal DAG of $X$ and $\cM^\cG_{V}$ be the marginal MAG of $\cG$.
    Then \cref{algo:scam} outputs a directed edge from $V_i\in V$ to $V_j\in V$ iff there is a direct edge in $\cG_X$ between them and $\Parents_j^\cG \dindep{\cG} U^j$.
    Further, the output of \cref{algo:scam} has the same skeleton as $\cM^\cG_{V}$.
\end{proposition}
\begin{proof}
    Note, that in this proof we will refer to edges that can be oriented w.r.t. a set $S\subseteq V$, where we mean applying Proposition \ref{prop:causal_dir_1} to the observed and unobserved nodes implied by the partitioning $X = S \cup (X\setminus S)$, instead of $X=V\cup U$. We also implicitly change the definitions of observed and unobserved parents $\Parents_i$ and $U^i$ for a node $V_i$ w.r.t. $S$.
    We define that a node is an \emph{unconfounded sink} w.r.t. to a set $S$ iff it has no children in $S$ and is not incident to an edge that cannot be oriented via Proposition \ref{prop:causal_dir_1} w.r.t. $S$.

    We prove the statement by induction over the steps of the algorithm.
    Let $S$ be the set of remaining nodes in an arbitrary step of the algorithm.
    Our induction hypothesis is that for $V_i, V_j\in S$ and $V_k\in B_i$ (where $B_i$ is defined as in Algorithm \ref{algo:scam}) we have
    \begin{enumerate}
        \item $V_i$ is an unconfounded sink w.r.t. to some set $S'\subseteq S$ iff $V_i$ is an unconfounded sink w.r.t. some $S''\subseteq V$
        \item if there is no $S'\subseteq V\setminus\{V_i, V_j\}$ such that $V_i\indep V_j\mid S'$ then $V_j\in B_i$
        %\item $X_i\indep X_k \mid S'$  for some set $S'\subseteq B_X$ iff $X_i\indep X_k\mid S''$ for some $S''\subseteq V\setminus\{X_i, X_k\}$
    \end{enumerate}
    Clearly, this holds in the initial step as $S=V$.
    
    Suppose we find $\delta_i(V_S) = 0$ for $V_i\in S$ and $|S| > 1$.
    By Proposition \ref{prop:causal_dir_1}, we know that all  nodes that are adjacent in the underlying MAG and are in $S$ must be parents of $V_i$.
    This means, all nodes that are not separable from $V_i$ must be direct parents of $V_i$, which are, by our induction hypothesis 2), the nodes in $B_i$.
    Since $V_i$ does not have children in $S$, it also suffices to check $V_i\indep V_j| S\setminus \{V_i, V_j\}$ for $V_j\in S$ (instead of conditioning on all subsets of $S$) to determine whether $V_j$ is in $B_i$.
    So we can already add these direct edges to the output.
    If, on the other hand, $V_i$ is not adjacent to a node in $S$, we have $V_i\indep V_j| S\setminus \{V_i, V_j\}$ for $V_j\in B_i$, so we add precisely the correct set of parents.
    Since $V_i$ is not a parent of any of the nodes in $S\setminus\{V_i\}$, $V_i$ cannot be in the set of unobserved parents of a node in $S\setminus\{V_i\}$ after it's removal and conditioning on $V_i$ cannot block an open path.
    Thus, the induction hypothesis still holds in the next step.

    Suppose now there is no unconfounded sink and we explore $V_i$.
    By our induction hypothesis 2), $B_i$ contains the parents of $V_i$ and by Proposition \ref{prop:causal_dir_1} it suffices to only look at subsets of $B_i$ to orient direct edges, as we only orient edges that also exist in the MAG.
    And also due to the induction hypothesis 2) $B_i$ contains all nodes that are not separable from $V_i$.
    So by adding undirected edges to all nodes in $B_i$ can only add too many edges but not miss some.

    Now it remains to show that the induction hypothesis holds if we set $S$ to $S\setminus\{V_i\}$ in the previous case.
    For 1) we need to show that $V_i$ cannot prevent the orientation of an edge w.r.t. $S\setminus\{V_i\}$.
    Suppose there are $V_k, V_l\in S\setminus\{V_i\}$ such that they are not separable and $\Parents_l\notdindep{\cG} U^l$.
    If $V_i\not\in U^l$, then the edge between $V_k$ and $V_l$ could not have been oriented w.r.t. to $S$ already.
    So suppose $V_i\in U^l$, i.e. there is a direct edge from $V_i$ to $V_l$ in the original DAG.
    %Suppose $V_i$ is on a unobserved causal path $V_k\to \dots \to U^m\to V_l$ with $V_k, V_l\in S\setminus\{V_i\}$ and $U^m\in V\setminus (S\setminus\{V_i\})$.
    %This path must have been a unobserved causal path before, unless $V_i=U^m$.
    %But then there is a direct edge $V_i\to V_l$.
    We would not remove $V_i$ from $S$ if this edge was orientable, so there must a hidden confounder or mediator between $V_i$ and $V_l$, i.e. we have $\Parents_l\notdindep{\cG} U^l$ w.r.t. to $S$.
    %\PF{I think we cannot orient edges between $X_k$ and $X_l$ in this case anyway.}
    But then we also have $V_k\notdindep{\cG} U^l$ w.r.t. to $S$, due to the edge between $V_i$ and $V_k$.
    So in this case, Proposition \ref{prop:causal_dir_1} wouldn't allow us to direct the edge w.r.t. $S\setminus\{V_i\}$ as we again have $\Parents_l\notdindep{\cG} U^l$.
    So by removing $V_i$ we do not render any edges unorientable.
    %Suppose there is confounding path $V_k\leftarrow \dots \to U^m\to V_l$ with $V_k, V_l\in S\setminus\{V_i\}$ and $U^m\in V\setminus (S\setminus\{V_i\})$.
    %If $V_i\neq U^m$ the path was already been a confounding path without $V_i$ being unobserved.
    %So again, there must be a confounder between $V_i$ and $V_l$, as otherwise we would not remove $V_i$.
    %\PF{Analogously to UCP we cannot orient $X_k\to X_l$.}
    %And analogously to before, we could not have oriented the edge even with $V_i\in S$ since  $V_{\Parents_l} \notdindep{\cG} U_l$.
    For 2) it suffices to note that we only remove nodes from $B_i$ if we found an independence.
    %For 3), note that if $X_j, X_k\in S\setminus\{X_i\}$ are separable by a set $S'\subseteq S\setminus\{X_i\}$ there is nothing to show.
    %So if $X_j\indep X_k | S'$ only for $X_i\in S'$, we need to check $X_j$ and $X_l$ in the neighbourhood pruning, i.e. we need to argue that $X_j, X_l\in B_i$.
    %If $X_j\indep X_k | S'$ only for $X_i\in S'$, we know that the paths between $X_k$ and $X_i$ (and between $X_i$ and $X_l$ analogously) are not separated by any set $S''\subseteq S\setminus\{X_k, X_i\}$ (because otherwise $S''$ would separate $X_k$ and $X_l$).
    %By our induction hypothesis this means $X_k$ and $X_i$ and also $X_i$ and $X_l$ are not separated by an $S''\subseteq B_i$.
    %But this means $X_k, X_l\in B_i$.
    %Therefore, we would check $X_k\indep X_l\mid S''\subseteq B_i \cup \{X_i\}$ in the pruning step, and remove $X_k$ from $B_i$.
    %So 3) also holds in the next step.
    
    For $|S| < 2$, the algorithm enters the final pruning stage.
    From the discussion above it is clear, that we already have the correct result, up to potentially too many undirected edges.
    In the final step we certainly remove all these edges $V_i-  V_j$, as we check $m$-separation for all subsets of the neighbourhoods $\operatorname{Adj}(V_i)$ and $\operatorname{Adj}(V_j)$, which are supersets of the true neighbourhoods.
    
\end{proof}

\subsection{On the output of AdaScore and CAM-UV}
\label{subsec:adascore_vs_camuv}
The algorithm we described in \cref{algo:scam} outputs (if desired by the user) a mixed graph with the skeleton of the underlying MAG and direct edges that indicate a direct causal influence.
The CAM-UV algorithm similarly outputs direct parental relationships and a set of pairs of nodes with what \citet{maeda21_causal} call \emph{unobserved backdoor paths} (UBP) and \emph{unobserved causal paths} (UCP) between them.
Yet, they do not investigate whether their algorithm \emph{only} outputs said pairs.
In fact, the following example shows that CAM-UV can also add pairs to this set that have neither a UBP nor a UCP between them.

\begin{example}
\label{ex:edge_without_ucp_ubp}
    In the following we will extensively reference Algorithm 1 and Algorithm 2 from \citet{maeda21_causal}.
    Suppose we apply CAM-UV to a sufficiently large sample from a distribution that fulfills CAM-UV's assumptions w.r.t. to the graph $G$ shown in \cref{fig:instrument_variable}, where $U$ is unobserved.
    First note, that there is neither a UBP nor a UCP w.r.t. to $U$ between $I$ and $Y$.
    Further, CAM-UV will not add $X$ to the set $M_Y$ of non-descendants of $Y$, due to the UBP between them.
    %\PF{Change $\notindep{G}$ to statistical independence for clarity.}
    Similarly, $I$ is not added to $M_Y$, as
    \begin{displaymath}
        Y - f(X) - g(I) \notindep \{X, I\} \quad \text{ and } \quad Y - h(I) \notindep I
    \end{displaymath}
    for any functions $f, g$ and $h$.
    Hence, Algorithm 2 in \citep{maeda21_causal} then finds
    \begin{displaymath}
        I\notindep Y
    \end{displaymath}
    and adds them to the set of pairs with UBP and UCP between them.

    On the other hand, the path $I\to X\leftarrow U\to Y$ is an inducing path w.r.t. to $U$. Therefore it is well-defined that AdaScore adds an (undirected) edge here.
		\begin{figure}[h] 
			\centering
					\begin{tikzpicture}[node distance=2cm, baseline=(current bounding box.north), auto]
				\node(i){$I$};
				\node(x)[right of=i]{$X$};
				\node[draw, circle, dotted](u)[above right of=x]{$U$};
				\node(y)[below right of=u]{$Y$};
				
				\draw[->] (x) edge node{} (y);
				\draw[->] (i) edge node{} (x);
				\draw[->] (u) edge node{} (x);
				\draw[->] (u) edge node{} (y);
				
			\end{tikzpicture}
    \caption{The CAM-UV algorithm will add $I$ and $Y$ to the set of nodes with unobserved backdoor path or unobserved causal path between them, though there is no such path between them.}
    \label{fig:instrument_variable}
		\end{figure}
    
\end{example}

This begs the question whether CAM-UV then also outputs the skeleton of the underlying MAG like AdaScore. 
  The answer is negative, as the following example shows.

\begin{example}
\PF{I think this also happens for RCD. Double check and add.}
    Suppose we apply CAM-UV to a sufficiently large sample from a distribution that fulfills CAM-UV's assumptions w.r.t. to the graph $G$ shown in \cref{fig:instrument_variable}, where now $U_1$ and $U_2$ are unobserved.
    Again, there is neither a UBP nor a UCP w.r.t. to $U_1, U_2$ between $Z$ and $Y$.
    Similarly to \cref{ex:edge_without_ucp_ubp}, CAM-UV will not add $X$ to the parent sets of $Z$ and $Y$ due to the UBP between $X$ and $Y$ and the UCP between $Z$ and $X$.
    Therefore, Algorithm 2 in \citep{maeda21_causal} will find
    \begin{displaymath}
        Z \notindep Y,
    \end{displaymath}
    and add the pair $(Z, Y)$ to the nodes with UBP or UCP between them.

    In contrast, $Z$ and $Y$ are not connected by an inducing path, as $X$ separates them.
    Therefore, the underlying MAG (and thus the output of AdaScore) does not contain an edge between $Z$ and $Y$.
		\begin{figure}[h] 
			\centering
					\begin{tikzpicture}[node distance=2cm, baseline=(current bounding box.north), auto]
				\node[draw, circle, dotted](i){$U_2$};
				\node(x)[right of=i]{$X$};
				\node[draw, circle, dotted](u)[above right of=x]{$U_1$};
				\node(y)[below right of=u]{$Y$};
                \node(z)[left of=i]{$Z$};
				
				\draw[->] (x) edge node{} (y);
				\draw[->] (x) edge node{} (i);
				\draw[->] (u) edge node{} (x);
				\draw[->] (u) edge node{} (y);
                \draw[->] (i) edge node{} (z);
				
			\end{tikzpicture}
    \caption{CAM-UV will add $Z$ and $Y$ to the set of nodes with unobserved backdoor path or unobserved causal path between them, though there is neither such path nor an inducing between them. }
    \label{fig:modified_insturment}
		\end{figure}
    
\end{example}

\subsection{Finite sample version of AdaScore}\label{app:finite_adascore}
All theoretical results in the paper have assumed that we know the density of our data.
Obviously, in practice, we have to deal with a finite sample instead.
Especially, in Proposition \ref{prop:score_msep} and Proposition \ref{prop:causal_dir_1} we derived criteria that compare random variables with zero.
Clearly, this condition is almost never met in practice.
Therefore, we need to find ways to reasonably set thresholds for these random quantities.

First note, that we use the Stein gradient estimator \citep{li2017gradient} to estimate the score function, given by
\begin{equation}
\hat{\bG}^{\text{Stein}} = -( \bK + \eta \textbf{I} )^{-1} \langle \nabla, \bK \rangle,
\label{eq:stein_estimator}
\end{equation}
where $\bv_Z^i$ is the $i$-th of $m\in \N$ samples of the observable variables in the set $Z$, $\bK$ is the matrix with $\bK_{ij} = \mathcal{K}(\bv^i, \bv^j)$, we use $\langle \nabla, \bK \rangle$ to denote a matrix with
$\langle \nabla, \bK \rangle_{ij} = \sum_{k=1}^{m} \nabla_{v^k_j} \mathcal{K}(\bv_Z^i, \bv_Z^k)$ and $\mathcal{K}$ is the RBF kernel in our case.
%\PF{Write down the estimator.}
This means especially that for a node $V_i$ we get a vector of estimates
\begin{equation}
\label{eq:score_sample_vector}
\left(\widehat{\partial V_i \log p(\bv^k_Z)}\right)_{k=1, \dots, m}
\end{equation}
i.e. an estimate of the score for every one of the $m$ samples.
Analogously, we get a $m\times d\times d$ tensor for the estimates of $\frac{\partial^2}{\partial V_i \partial V_j} \log p(v)$ and $m$ empirical estimates of the residual $\delta_i(V_Z)$.

In Proposition \ref{prop:score_msep} we showed that 
\[
    \frac{\partial^2}{\partial V_i \partial V_j} \log p(v_Z) = 0 \iff X_i \mindep{\cM_{V}^\Gx} V_j | V_Z\setminus \set{V_i, V_j}.
\]
In the finite sample version, we use a one-sample t-test on the vector of estimated cross-partial derivatives with the null hypothesis that the means is zero.
Due to the central limit theorem, the sample mean follows approximately a Gaussian distribution, regardless of the true distribution of the observations. %\PF{cite this?}
In the pruning steps, we do not conduct a test for every possible subset of the neighbourhood.
Instead, we pick the subset with the minimal mean absolute value of the estimated cross-partial derivatives and conduct the t-test for this set.
I.e. if we consider pruning the edge $V_i - V_j)$ we pick the subset
\begin{displaymath}
    Z =\min_{Z'\subseteq B_j \cup B_i } \operatorname{mean}|\widehat{\delta_{i,j}(V_{Z'})}|.
\end{displaymath}
We then conduct the t-test with the empirical estimates of $\delta_{i,j}(V_{Z})$.

For Proposition \ref{prop:causal_dir_1}, we first prove the following connection to the independence of fitted additive noise models.
\begin{proposition}[Zero MSE implies ANM]
\label{prop:causal_dir_ind_noise}
    Let $Z$ be a set of variable indices such that $V_j\in V_Z$. Then with $R_j(V_Z) := V_j - \mean[V_j | V_{\setminus V_j}]$ we get
    \begin{displaymath}
        \mean[\partial_{V_j} \log p(V_Z) - \mean[\partial_{V_j} \log p(V_Z) | R_j(V_Z)]]^2 = 0 \implies V_j = \mean[V_j | V_{\setminus V_j}] +  R_j(V_Z),
    \end{displaymath}
    with $V_{\setminus V_j}\indep  R_j(V_Z)$. 
    In other words, if the score can be estimated with zero MSE, then $V_j$ can be described by an additive noise model in $V_{\setminus V_j}$ with independent noise. 
\end{proposition}
\begin{proof}
    If the MSE of the final regression is zero, we have by definition
    \begin{align}
        0 &= \mean[\partial_{V_j} \log p(V_Z) - \mean[\partial_{V_j} \log p(V_Z) | R_j(V_Z)]]^2 \\
        &= \int \left(\partial_{V_j} \log p(v_Z)  - \mean[\partial_{V_j} \log p(V_Z) | R_j(v_Z) ]\right)^2  dv_Z,
    \end{align} 
    which implies that $\partial_{V_j} \log p(v_Z) = \mean[\partial_{V_j} \log p(V_Z) | R_j(v_Z)]$ almost everywhere, since this is an integral over non-negative values.
    Denote $q(r) \:= \mean[\partial_{V_j} \log p(V_Z) | r]$.
    %First note that this cannot be if $X_i$ has a child in $\cV$ for the same reason as in the proof of Lemma 1 by \citet{montagna23_nogam}. \PF{Ideally reproduce argument and additionally argue that the identifiablity condition form this proof is a null set in the set of our parametrized functions.}
    By the fundamental theorem of calculus we have
    \begin{align*}
        \log p(v_Z) &= \int \partial_{V_j} \log p(v_Z) dv_j + c(v_{\setminus V_j})\\
        p(v_Z) &= \exp\left(\int \partial_{V_j} \log p(v_Z) dv_j + c(v_{\setminus V_j})\right)\\
        &= \exp\left(\int q(R_j(v_Z)) dv_j + c(v_{\setminus V_j})\right)\\
        &= \exp\left(\int q(v_j - \mean[V_j | v_{\setminus V_j}]) dv_j + c(v_{\setminus V_j})\right)\\
        &= \exp\left(Q(R_j(v_Z)) + c(v_{\setminus V_j})\right)\\
        &= \exp\left(Q(R_j(v_Z)) \right) \cdot \exp\left( c(v_{\setminus V_j})\right),
         %&= \exp\left(\int  h^*(x_i - q^*(\bX_{-i})) dx_i)\right) \cdot \exp\left((c(\bX_{-i})\right)\\
         %&= \exp\left(\int  h^*(\hat N_i) dx_i)\right) \cdot \exp\left((c(\bX_{-i})\right),
    \end{align*}
    for some function $Q$ and some $c(v_{\setminus V_j})$ which is constant in $V_j$ (but may depend on $V_{\setminus V_j}$).
    %Note that the mapping $R_j$ is not injective, since we have assumed that the noise variables have positive densities.
    %Therefore, the left hand factor in the last line
    Since %$R_j(V_Z)$ is deterministic in $V_Z$ and further 
    $V_j$ is deterministic in $V_{\setminus V_j}$ and $R_j(V_Z)$ by definition, we get
    \begin{displaymath}
       p( R_j(v_Z), v_{\setminus V_j}) = p(v_Z) = \exp\left(Q(R_j(v_Z)) \right) \cdot \exp\left( c(v_{\setminus V_j})\right)
    \end{displaymath}
    and further, since this joint distribution can be factorized into two functions of $V_{\setminus V_j}$ and $R_j(V_Z)$ respectively, we get $R_j(V_Z) \indep V_{\setminus V_j}$ %\PF{I think this factorization result exists. Need to double check}.
    Now we have established that we can write
    \begin{displaymath}
        V_j = \mean[V_j | V_{\setminus V_j}] +  R_j(V_Z) ,
    \end{displaymath}
    and $V_{\setminus V_j}\indep R_j(V_Z)$, or in other words, that $V_j$ can be described by an additive noise model in $V_{\setminus V_j}$.
    
\end{proof}
This means, we can use a classical independence test on the regression residual $R_j(V_Z)$ (that we had to estimate anyway) and the regressors $V_{\setminus V_j}$.
If such a test rejects, this also rejects that $\delta_j(V_Z) = 0$.
We used the kernel independence test proposed by \citep{zhang2012kernel} to this end.

As candidate sink for set $Z\subseteq V$, we pick the node $V_i=\min_i \operatorname{mean}(\delta_i(V_Z))$.
%As second \enquote{reference node} 
%\begin{displaymath}
%X_j=\min_{\substack{j\in B_i \\j\neq i}} \operatorname{mean}(\delta_j(X_Z)),
%\end{displaymath}
%where $B_i$ be the set of nodes that have not been $m$-separated from $X_i$ by any test so far as in \cref{algo:scam}.
%If also the baseline independence test does not reject, we do not consider $X_i$ an unconfounded sink.

In the case where we use Proposition \ref{prop:causal_dir_1} to \emph{orient} edges, we only need to decide whether a previously undirected edge $V_i - V_j$ needs to be oriented one way, the other way, or not at all.
Again, here the issue lies in the fact that we need to iterate over possible sets of parents of the nodes.
We pick the subset 
\begin{displaymath}
    Z_j =\min_{Z'\subseteq B_j} \operatorname{mean}(\widehat{\delta_j(V_{Z'}}),
\end{displaymath}
i.e. the set with the lowest MSE.
We then conduct the test with the estimates of $\delta_i(V_{Z_j})$ and $\delta_j(V_{Z_j})$ to check if the edge is pointing to $V_j$.
If there is a directed edge between them, one of the residuals will be independent from the regressors and the other one won't.

Just like \citet{montagna23_nogam} we use a cross-validation scheme to generate the residuals, in order to prevent overfitting.
We split the dataset into several equally sized, disjoint subsamples. For every residual we fit the regression on all subsamples that don't contain the respective target.

%Also, just like in the NoGAM algorithm \citet{montagna23_nogam} we add a pruning step for the directed edges to the end.
%The idea is to use a feature selection method to remove insignificant edges.
%Just like \citet{montagna23_nogam}, we use the CAM-based pruning step proposed by \citet{buhlmann14_cam}, which fits a generalised additive regression model from the parents to a child and test whether one of the additive components is significantly non-zero.
%All parents for which the test rejects this hypothesis are removed.

\subsection{Complexity}\label{app:complexity}
\begin{proposition}{Complexity}
    Let $n$ be the number of samples and $d$ the number of observable nodes.
    \cref{algo:scam} runs in
    \[
    \Omega\left((d^2-d) \cdot (r(n, d) + s(n, d))\right) \quad \text{ and } \quad \mathcal{O}\left(d^2\cdot 2^d (r(n, d) + s(n, d))\right),
    \]
    where $r(n, d)$ is the time required to solve a regression problem and $s(n, d)$ is the time for calculating the score.
    With e.g. kernel-ridge regression and the Stein-estimator, both run  in $\mathcal{O}(n^3)$.
\end{proposition}
\begin{proof}
    \cref{algo:scam} runs its main loop $d$ times.
    It first checks for the existence of an unconfounded sink, which involves solving $2d$ regression problems (including cross-validation prediction) and calculating the score, adding up to $(2d^2 - d)$ regressions and $d$ score evaluations.
    In the worst case, we detect no unconfounded sink and iterate through all subsets of the neighbourhood of a node (which is in the worst case of size $d-1$) and for all other nodes in the neighbourhood we solve $2d$ regression problems and evaluate the score.
    For each subset we calculate two regression functions, the score and calculate the entries in the Hessian of the log-density, i.e. $d\cdot 2^{d}$ regressions, $d\cdot 2^{d-1}$ scores and additionally $2^{d-1}$ Hessians.
    If we are unlucky, this node has a directed outgoing edge and we continue with this node (with the same size of nodes).
    This can happen $d-1$ times.
    So we get $(d^2 - d)\cdot 2^{d}$ regressions and $(d^2 - d)\cdot 2^{d-1}$ scores and Hessians.
    In the final pruning step we calculate for every previously undirected edge (of which there can be $(d^2-d)/2$) a Hessian for all subsets of the neighbourhoods, which can again be $2^{d-1}$ subsets.
    Using the pruning procedure from CAM for the directed edges we also spend at most $\mathcal{O}(nd^3)$ steps.

    In the best case, we always find an unconfounded sink.
    Then our algorithm reduces to NoGAM.

\end{proof}

\section{Experimental details}\label{app:exp_details}
In this section, we present the details of our experiments in terms of synthetic data generation and algorithms hyperparameters.

\subsection{Synthetic data generation}\label{app:data}
In this work, we rely on synthetic data to benchmark AdaScore's finite samples performance. For each dataset, we first sample the ground truth graph and then generate the observations according to the causal graph.

\paragraph{Erd\"{o}s-Renyi graphs.} The ground truth graphs are generated according to the Erd\"{o}s-Renyi model. It allows specifying the number of nodes and the probability of connecting each pair of nodes). In ER graphs, a pair of nodes has the same probability of being connected.
To introduce hidden variables we randomly drop columns from the data matrices.
Whenever the resulting graph contains no hidden confounder (or in the case of nonlinear data also hidden mediators) we reject the choice of hidden variables and sample again.
We similarly reject choices without hidden variables or hidden mediators for the experiments in \cref{fig:experiments-real}.
We always compare against the graph that AdaScore would output in the limit of infinite data.

\paragraph{Nonlinear causal mechanisms.} Nonlinear causal mechanisms are parametrized by a neural network with random weights. We create a fully connected neural network with one hidden layer with $10$ units, Parametric ReLU activation function, followed by one normalizing layer before the final fully connected layer. The weights of the neural network are sampled from a standard Gaussian distribution. This strategy for synthetic data generation is commonly adopted in the literature \citep{montagna2023_assumptions, montagna2023shortcuts, ke2023learning, brouillard2020intervention, lippe2022efficient}.

\paragraph{Linear causal mechanisms.} \looseness-1For the linear mechanisms, we define a simple linear regression model predicting the effects from their causes and noise terms, weighted by randomly sampled coefficients. Coefficients are generated as samples from a Uniform distribution supported in the range $[-3, -0.5] \cup [0.5, 3]$. We don't use too small coefficients to avoid trivial cases of \textit{close to unfaithful} datasets \citep{uhler12_faithfulness}.

\paragraph{Noise terms distribution.} The noise terms are sampled from a Uniform distribution supported between $-2$ and $2$.

Finally, we remark that we standardize the data by their empirical standard deviation. This is known to remove shortcuts that allow finding a correct causal order sorting variables by their marginal variance, as in \textit{varsortability}, described in \citet{reisach2021beware}, or sorting variables by the magnitude of their score $|\partial_{X_i} \log p(X)|$, a phenomenon known as \textit{scoresortability} analyzed by \citet{montagna2023shortcuts}.

% -------------------------------------------
\subsection{AdaScore hyperparameters}\label{app:hyperparams}
For AdaScore, we set the $\alpha$ level for the required hypothesis testing at $0.05$. For the CAM-pruning step, the level is instead set at $0.001$, the default value of the \href{https://www.pywhy.org/dodiscover/dev/index.html}{dodidscover} Python implementation of the method, and commonly found in all papers using CAM-pruning for edge selection \citep{rolland22_score, montagna23_das, montagna23_nogam, buhlmann14_cam}. For the remaining parameters. The regression hyperparameters for the estimation of the residuals are found via cross-validation during inference: tuning is done minimizing the generalization error on
the estimated residuals, without using the performance on the causal graph ground truth. Finally, for the score matching estimation, the regularization coefficients are set to $0.001$.
In the respective experiments, we provided the fact that there are no latent confounders to AdaScore.

\subsection{Random baseline}\label{app:random_baseline}
In our synthetic experiments we also considered a random baseline.
Since we did not want the sparsity of the ground truth graph to influence the performance of the random baseline, we chose the following approach:
we use the Erd\"{o}s-Renyi model, described above, to generate a new graph with the same edge probability and the same number of nodes (including hidden variables) as the ground truth graph. 
We then project this graph onto the actually observed nodes by generating the PAG skeleton over these nodes and directing edges iff they are identifiable by Proposition \ref{prop:causal_dir_1} (using graphical criteria).

\subsection{Compute resources}\label{app:compute}
All experiments have been run on an AWS EC2 instance of type \texttt{m5.12xlarge}.
These machines contain Intel Xeon Platinum 8000 processors with 3.1 GHz and 48 virtual cores as well as 192 GB RAM.
All experiments can be run within a day.

% -------------------------------------------
% -------------------------------------------

\section{Additional Experiments}\label{app:experiments}
In this section, we provide additional experimental results.
All synthetic data has been generated as described in \cref{app:data}.

\subsection{Hypothesis Tests}
Additionally to the plots in \cref{fig:experiments-sparse} (and also \cref{fig:experiments-dense}) we conducted hypothesis tests to see whether the results of AdaScore are stochastically greater or less than the results of the other algorithms.
I.e. we tested the null-hypothesis that $P(A > B) = P(B < A)$, where $A$ denotes the SHD of AdaScore w.r.t. the ground truth graph and $B$ the SHD of one of the baseline algorithms.
In \cref{tab:linear_obs_dense,tab:nonlinear_obs_dense,tab:linear_latent_dense,tab:nonlinear_latent_dense,tab:linear_obs_sparse,tab:nonlinear_obs_sparse,tab:linear_latent_sparse,tab:nonlinear_latent_sparse} we provided $p$-values for the exact Mann-Whitney $U$ test with alternatives that AdaScore is less than the given baseline or greater, respectively.

\begin{table}
\centering
\caption{$p$-values for stochastic ordering: sparse linear fully observable model}
\label{tab:linear_obs_sparse}
\begin{tabular}{l|llll|llll}
\toprule
{} & \multicolumn{4}{c}{less} & \multicolumn{4}{c}{greater} \\
{} &       3 &       5 &       7 &       9 &       3 &       5 &       7 &       9 \\
\midrule
camuv  & 0.61060 & 0.44153 & 0.76942 & 0.79324 & 0.39970 & 0.56901 & 0.23885 & 0.22248 \\
nogam  & 0.60030 & 0.12113 & 0.08691 & 0.03375 & 0.41007 & 0.88947 & 0.92142 & 0.96828 \\
rcd    & 0.61060 & 0.08267 & 0.34901 & 0.68977 & 0.39970 & 0.92142 & 0.66085 & 0.32940 \\
lingam & 0.70844 & 0.79324 & 0.99854 & 1.00000 & 0.30083 & 0.21454 & 0.00177 & 0.00000 \\
random & 0.00000 & 0.00000 & 0.00000 & 0.00000 & 1.00000 & 1.00000 & 1.00000 & 1.00000 \\
\bottomrule
\end{tabular}
\end{table}

\begin{table}
\centering

\caption{$p$-values for stochastic ordering: sparse nonlinear fully observable model}

\label{tab:nonlinear_obs_sparse}
\begin{tabular}{l|llll|llll}
\toprule
{} & \multicolumn{4}{c}{less} & \multicolumn{4}{c}{greater} \\
{} &       3 &       5 &       7 &       9 &       3 &       5 &       7 &       9 \\
\midrule
camuv  & 1.00000 & 1.00000 & 0.99953 & 0.98249 & 0.00000 & 0.00000 & 0.00052 & 0.01876 \\
nogam  & 1.00000 & 0.99995 & 0.99892 & 0.95951 & 0.00000 & 0.00006 & 0.00120 & 0.04554 \\
rcd    & 0.07464 & 0.13831 & 0.00019 & 0.00001 & 0.93279 & 0.87331 & 0.99985 & 0.99999 \\
lingam & 0.97206 & 0.99564 & 0.70844 & 0.65099 & 0.03172 & 0.00516 & 0.30083 & 0.35898 \\
random & 0.60030 & 0.02009 & 0.00005 & 0.00001 & 0.42051 & 0.98249 & 0.99995 & 0.99999 \\
\bottomrule
\end{tabular}
\end{table}

\begin{table}
\centering
\caption{$p$-values for stochastic ordering: sparse linear latent variables model}
\label{tab:linear_latent_sparse}
\begin{tabular}{l|llll|llll}
\toprule
{} & \multicolumn{4}{c}{less} & \multicolumn{4}{c}{greater} \\
{} &       3 &       5 &       7 &       9 &       3 &       5 &       7 &       9 \\
\midrule
camuv  & 0.18443 & 0.60030 & 0.74417 & 0.79324 & 0.82267 & 0.41007 & 0.26455 & 0.22248 \\
nogam  & 0.18443 & 0.23058 & 0.10060 & 0.07858 & 0.82267 & 0.77752 & 0.90412 & 0.92536 \\
rcd    & 0.31976 & 0.62082 & 0.41007 & 0.50533 & 0.69917 & 0.38940 & 0.60030 & 0.50533 \\
lingam & 0.25583 & 0.76942 & 0.97702 & 0.99600 & 0.76115 & 0.23885 & 0.02620 & 0.00436 \\
random & 0.00001 & 0.00000 & 0.00000 & 0.00000 & 0.99999 & 1.00000 & 1.00000 & 1.00000 \\
\bottomrule
\end{tabular}
\end{table}

\begin{table}
\centering

\caption{$p$-values for stochastic ordering: sparse nonlinear latent variables model}
\label{tab:nonlinear_latent_sparse}
\begin{tabular}{l|llll|llll}
\toprule
{} & \multicolumn{4}{c}{less} & \multicolumn{4}{c}{greater} \\
{} &       3 &       5 &       7 &       9 &       3 &       5 &       7 &       9 \\
\midrule
camuv  & 0.99854 & 0.97206 & 0.87887 & 0.76115 & 0.00161 & 0.03172 & 0.13242 & 0.25583 \\
nogam  & 0.34901 & 0.52665 & 0.42051 & 0.28242 & 0.67060 & 0.48400 & 0.58993 & 0.72659 \\
rcd    & 0.96625 & 0.51600 & 0.22248 & 0.00146 & 0.03589 & 0.50533 & 0.79324 & 0.99868 \\
lingam & 0.12113 & 0.13831 & 0.33915 & 0.01054 & 0.88425 & 0.87331 & 0.68024 & 0.99095 \\
random & 0.00010 & 0.00001 & 0.00000 & 0.00000 & 0.99991 & 0.99999 & 1.00000 & 1.00000 \\
\bottomrule
\end{tabular}
\end{table}

\begin{table}
\centering
\caption{$p$-values for stochastic ordering: dense linear fully observable model}
\label{tab:linear_obs_dense}
\begin{tabular}{l|llll|llll}
\toprule
{} & \multicolumn{4}{c}{less} & \multicolumn{4}{c}{greater} \\
{} &       3 &       5 &       7 &       9 &       3 &       5 &       7 &       9 \\
\midrule
camuv  & 0.91309 & 0.28242 & 0.55847 & 0.99823 & 0.09132 & 0.73545 & 0.46271 & 0.00213 \\
nogam  & 0.87331 & 0.12113 & 0.01632 & 0.00367 & 0.13831 & 0.88947 & 0.98479 & 0.99693 \\
rcd    & 0.81557 & 0.07464 & 0.82961 & 0.96828 & 0.19171 & 0.92915 & 0.17733 & 0.03375 \\
lingam & 0.91309 & 0.98585 & 1.00000 & 1.00000 & 0.09132 & 0.01521 & 0.00000 & 0.00000 \\
random & 0.00000 & 0.00000 & 0.00108 & 0.00234 & 1.00000 & 1.00000 & 0.99912 & 0.99806 \\
\bottomrule
\end{tabular}
\end{table}

\begin{table}
\centering
\caption{$p$-values for stochastic ordering: dense nonlinear fully observable model}

\label{tab:nonlinear_obs_dense}
\begin{tabular}{l|llll|llll}
\toprule
{} & \multicolumn{4}{c}{less} & \multicolumn{4}{c}{greater} \\
{} &       3 &       5 &       7 &       9 &       3 &       5 &       7 &       9 \\
\midrule
camuv  & 1.00000 & 0.99999 & 0.99095 & 0.79324 & 0.00000 & 0.00001 & 0.01054 & 0.21454 \\
nogam  & 1.00000 & 1.00000 & 1.00000 & 0.96625 & 0.00000 & 0.00000 & 0.00000 & 0.03813 \\
rcd    & 0.16363 & 0.02298 & 0.00000 & 0.00000 & 0.84937 & 0.97991 & 1.00000 & 1.00000 \\
lingam & 0.94892 & 0.81557 & 0.76115 & 0.48400 & 0.05713 & 0.19171 & 0.24726 & 0.53729 \\
random & 0.97545 & 0.28242 & 0.00010 & 0.00000 & 0.02620 & 0.73545 & 0.99992 & 1.00000 \\
\bottomrule
\end{tabular}
\end{table}

\begin{table}
\centering
\caption{$p$-values for stochastic ordering: dense linear latent variables model}

\label{tab:linear_latent_dense}
\begin{tabular}{l|llll|llll}
\toprule
{} & \multicolumn{4}{c}{less} & \multicolumn{4}{c}{greater} \\
{} &       3 &       5 &       7 &       9 &       3 &       5 &       7 &       9 \\
\midrule
camuv  & 0.24726 & 0.11053 & 0.90868 & 0.99902 & 0.76115 & 0.89452 & 0.10060 & 0.00120 \\
nogam  & 0.05404 & 0.18443 & 0.09132 & 0.00977 & 0.94892 & 0.82961 & 0.91309 & 0.99163 \\
rcd    & 0.48400 & 0.13242 & 0.99912 & 0.99928 & 0.52665 & 0.87887 & 0.00108 & 0.00088 \\
lingam & 0.19171 & 0.90412 & 0.99948 & 0.99987 & 0.81557 & 0.10060 & 0.00058 & 0.00015 \\
random & 0.01316 & 0.00000 & 0.00027 & 0.25583 & 0.98777 & 1.00000 & 0.99976 & 0.75274 \\
\bottomrule
\end{tabular}
\end{table}

\begin{table}
\centering
\caption{$p$-values for stochastic ordering: dense nonlinear latent variables model}

\label{tab:nonlinear_latent_dense}
\begin{tabular}{l|llll|llll}
\toprule
{} & \multicolumn{4}{c}{less} & \multicolumn{4}{c}{greater} \\
{} &       3 &       5 &       7 &       9 &       3 &       5 &       7 &       9 \\
\midrule
camuv  & 0.48400 & 0.89940 & 0.97991 & 0.97702 & 0.53729 & 0.10548 & 0.02298 & 0.02620 \\
nogam  & 0.00052 & 0.16363 & 0.37918 & 0.10060 & 0.99953 & 0.84937 & 0.64102 & 0.90412 \\
rcd    & 0.71758 & 0.32940 & 0.98368 & 0.36903 & 0.29156 & 0.68977 & 0.01751 & 0.65099 \\
lingam & 0.00016 & 0.00108 & 0.06371 & 0.00257 & 0.99987 & 0.99902 & 0.93965 & 0.99766 \\
random & 0.00072 & 0.00177 & 0.13831 & 0.00003 & 0.99935 & 0.99854 & 0.87331 & 0.99997 \\
\bottomrule
\end{tabular}
\end{table}

\subsection{Non-additive mechanisms}
In \cref{fig:experiments-sparse} we have demonstrated the performance of our proposed method on data generated by linear SCMs and non-linear SCMs with additive noise.
But Proposition \ref{prop:score_msep} also holds for \emph{any} faithful distribution generated by an acyclic model.
Thus, we employed as mechanism a neural network-based approach similar to the non-linear mechanism described in \cref{app:exp_details}.
Instead of adding the noise term, we feed it as additional input into the neural network. Results in this setting are reported in \cref{fig:f1-experiments-sparse}.
As neither the mixed graph mode of AdaScore nor any of the baseline algorithms has theoretical guarantees for the orientation of edges in this scenario, we report the $F_1$-score (popular in classification problems) w.r.t. to the existence of an edge, regardless of orientation. Our experiments show that AdaScore can, in general, correctly recover the graph's skeleton in all the scenarios, with an $F_1$ score median between $1$ and $\sim0.75$, respectively for small and large numbers of nodes.

On the other hand, if we let AdaScore output a PAG, we can apply the ordinary orientation rules of FCI \citep{spirtes2000_cauastion}.
In \cref{fig:fci_experiments} we compare the vanilla FCI algorithm with kernel independence test \citep{zhang2012kernel} against AdaScore on dense and sparse ground truth graphs.

\begin{figure}[ht]
\centering
%\subfigure[]{%
    \includegraphics[width=0.55\textwidth]{main_text_plots/legend.pdf}\\
%}
\vspace{1.2em}
\subfigure[Fully observable model\label{fig:observable_non_additive}]{%
    \includegraphics[width=0.4\textwidth]{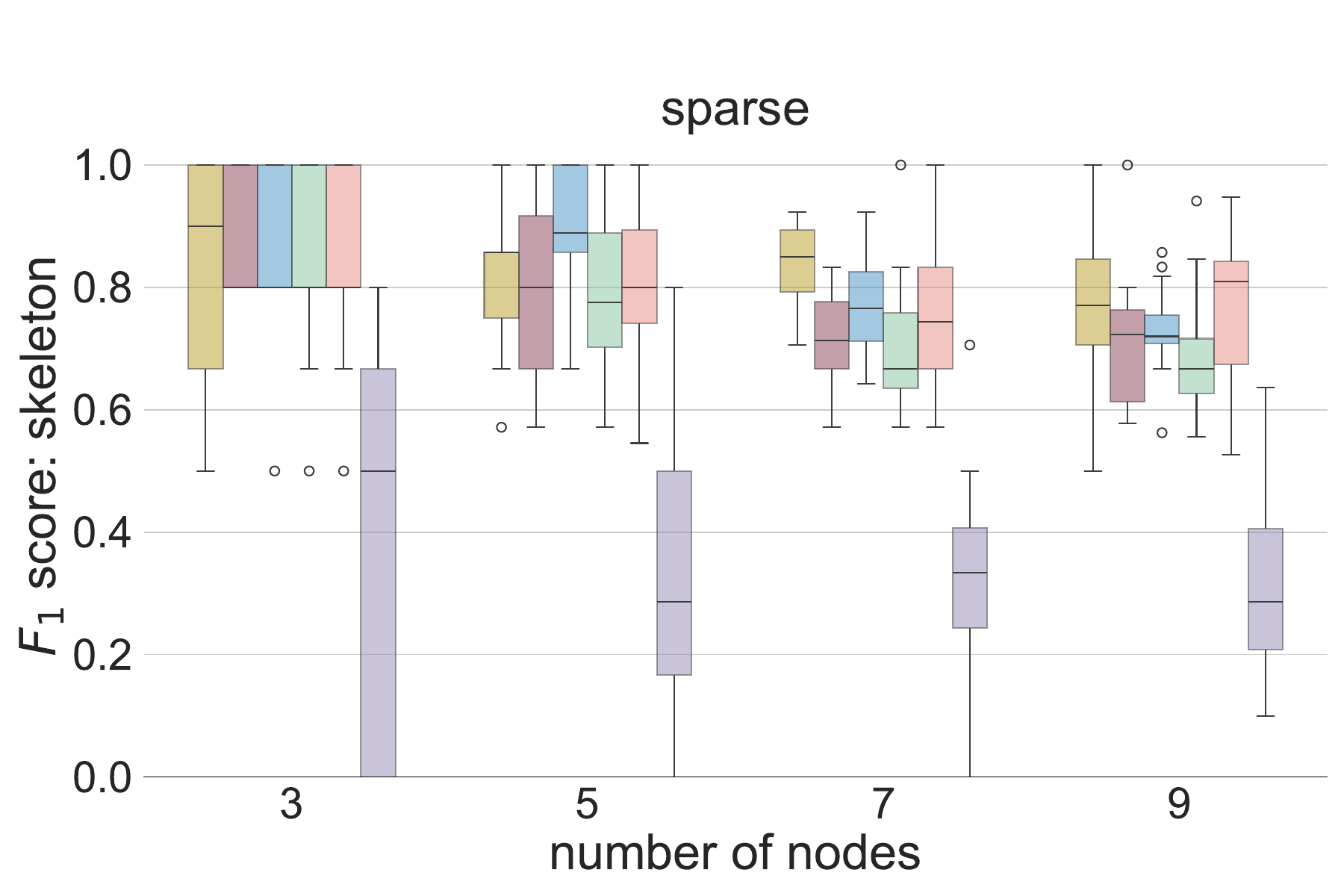}
    \includegraphics[width=0.4\textwidth]{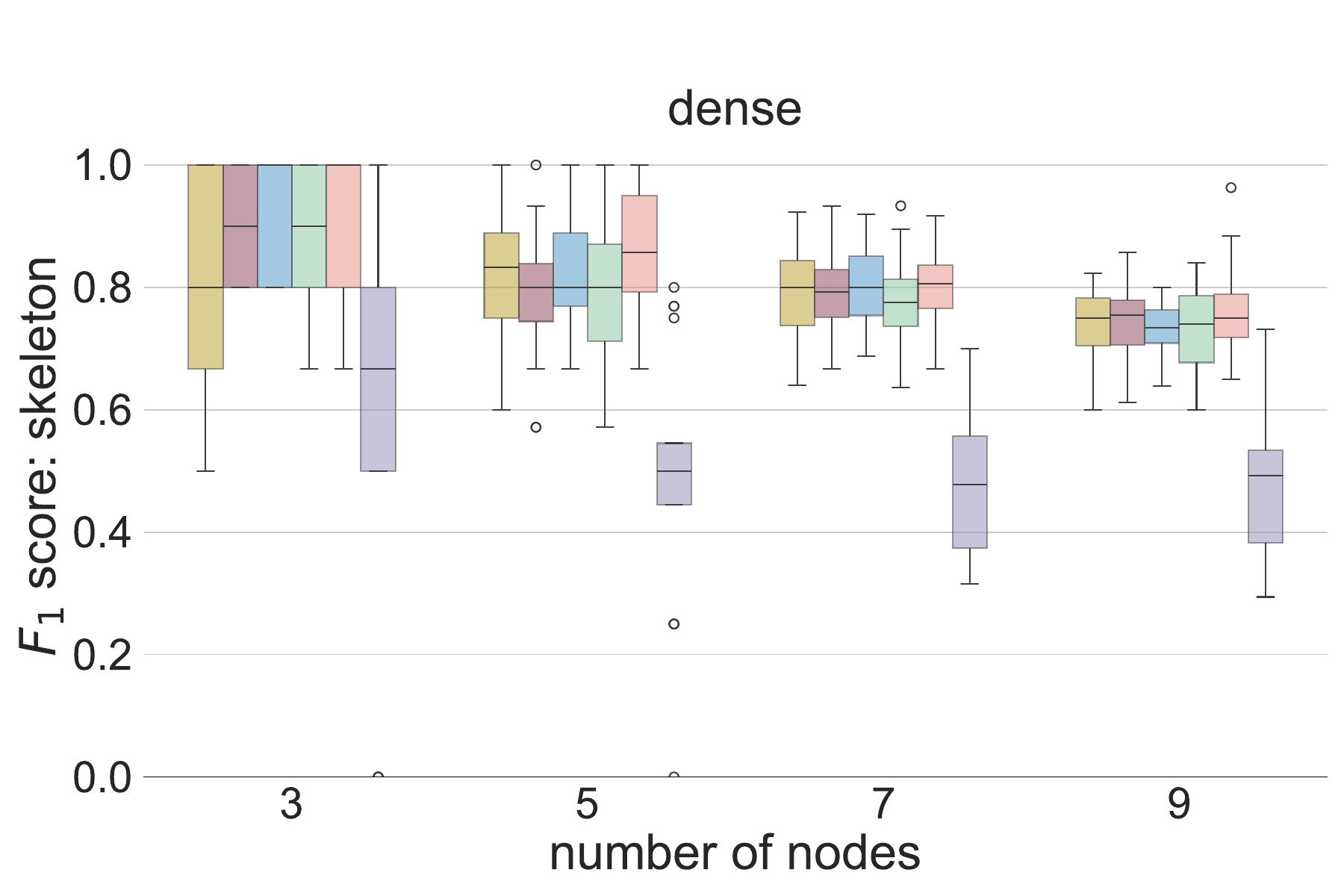}
}
\vspace{1.2em}
\subfigure[Latent variables model\label{fig:latent_non_additive}]{%
    \includegraphics[width=0.4\textwidth]{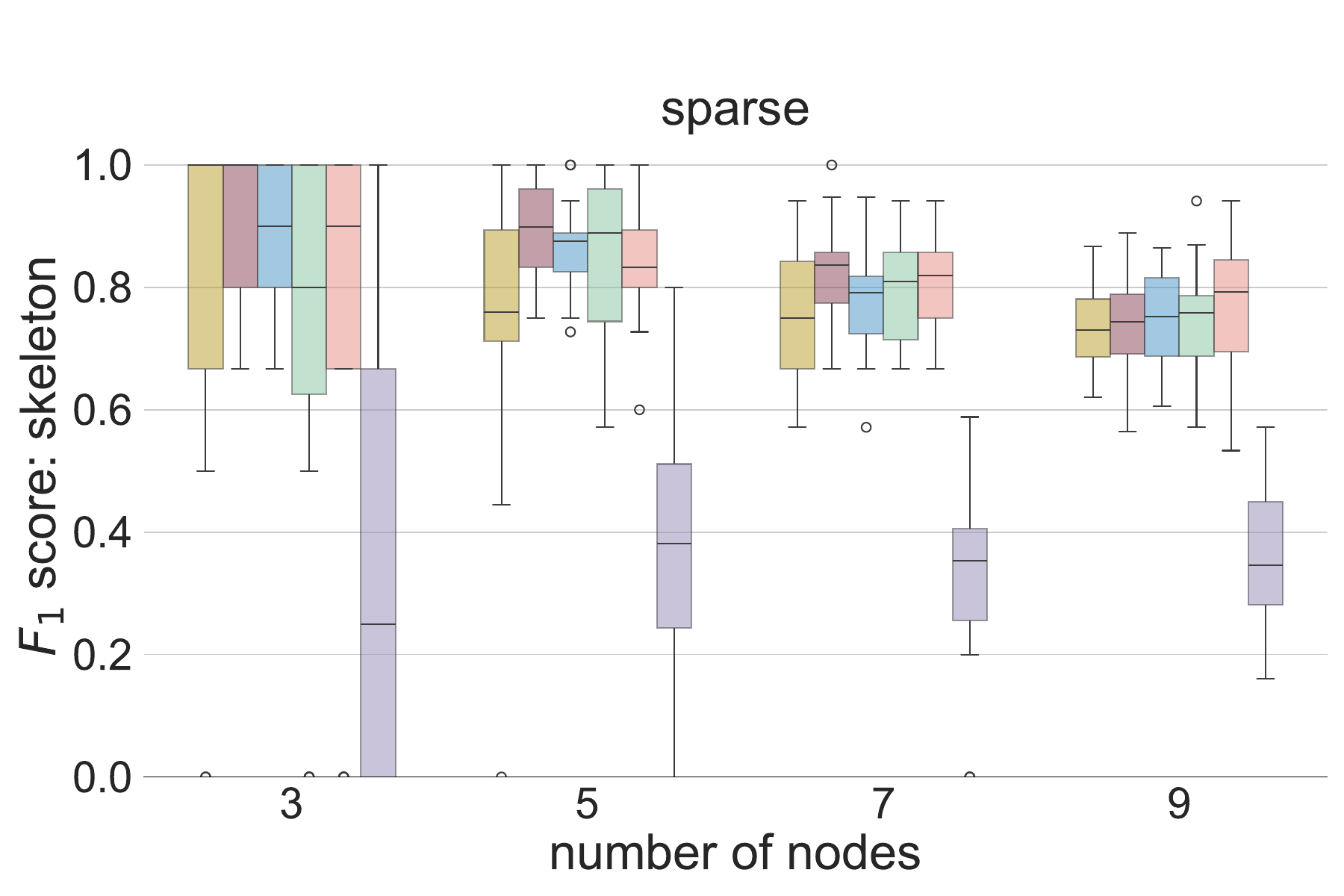}
    \includegraphics[width=0.4\textwidth]{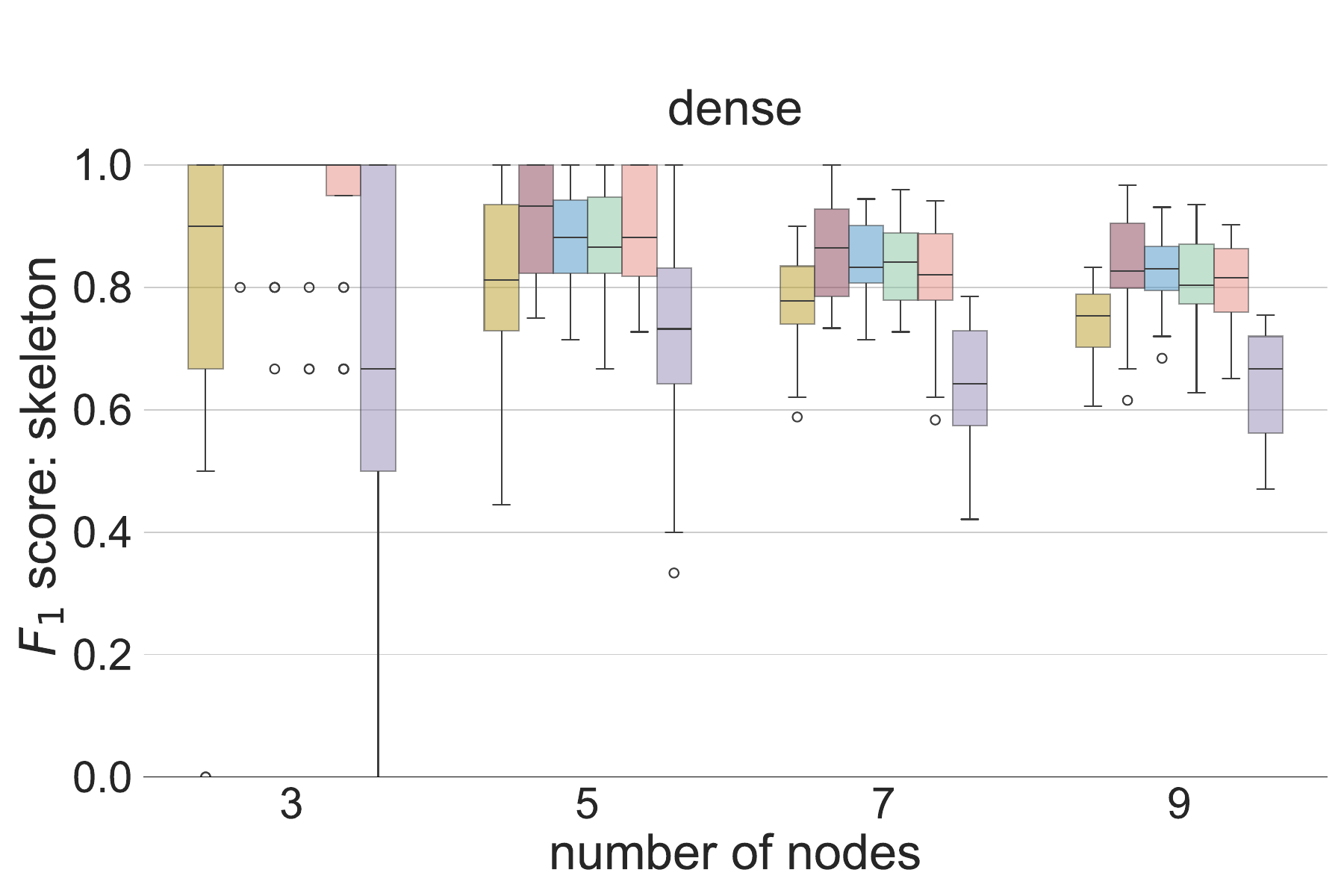}
}
\caption{Empirical results for non-additive causal mechanisms on sparse and dense graphs with different numbers of nodes, on fully observable (no hidden variables) and latent variable models. We report the $F_1$ score w.r.t. the existence of edges (the higher, the better).}
\label{fig:f1-experiments-sparse}
\end{figure}

\begin{figure}[ht]
\centering
%\subfigure[]{%
    \includegraphics[width=.2\textwidth]{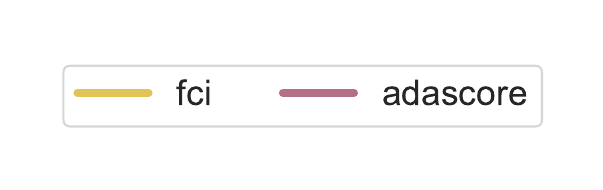}\\
%}
%\subfigure[Fully observable model\label{fig:observable_non_additive}]{%
    \includegraphics[width=0.4\textwidth]{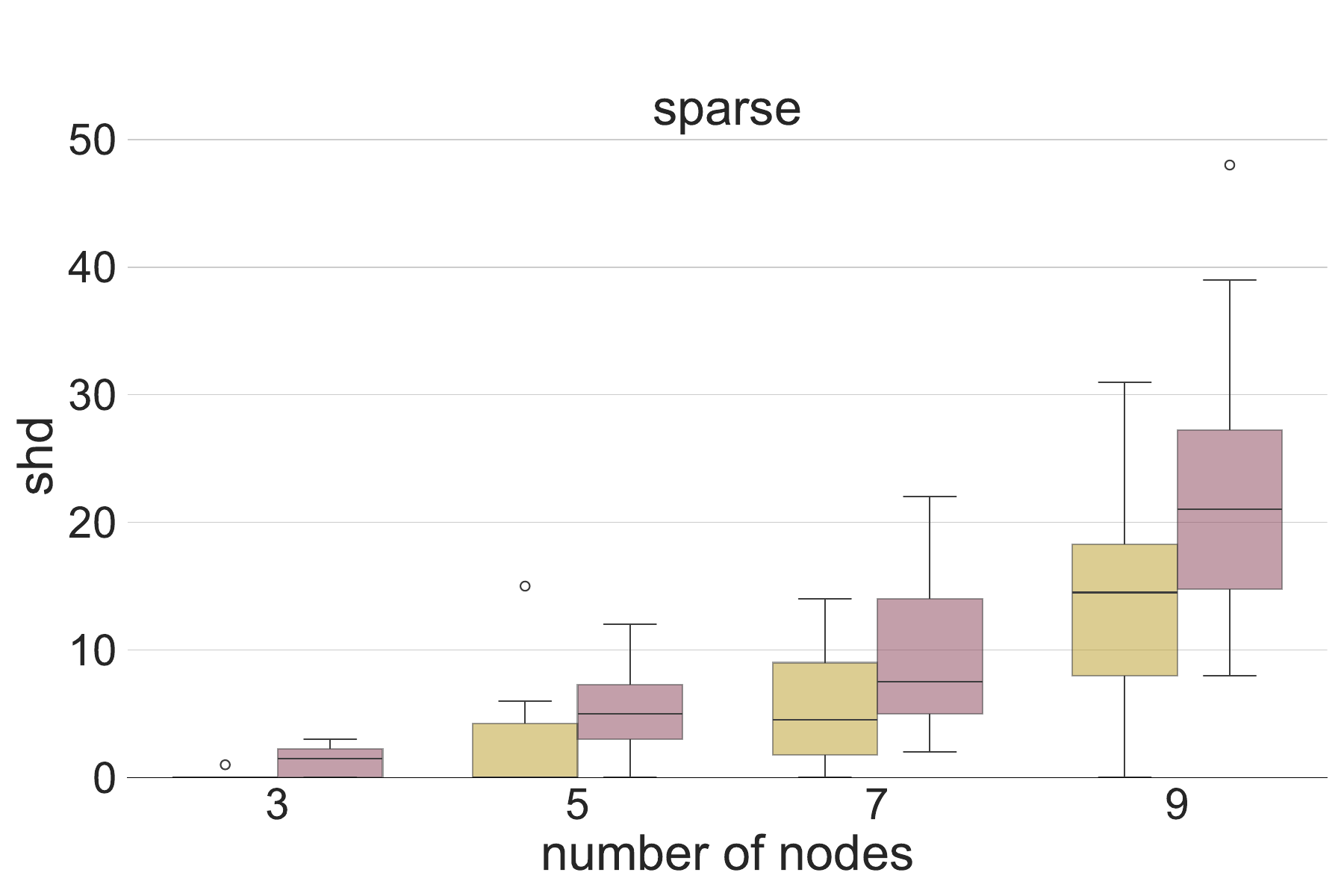}
    \includegraphics[width=0.4\textwidth]{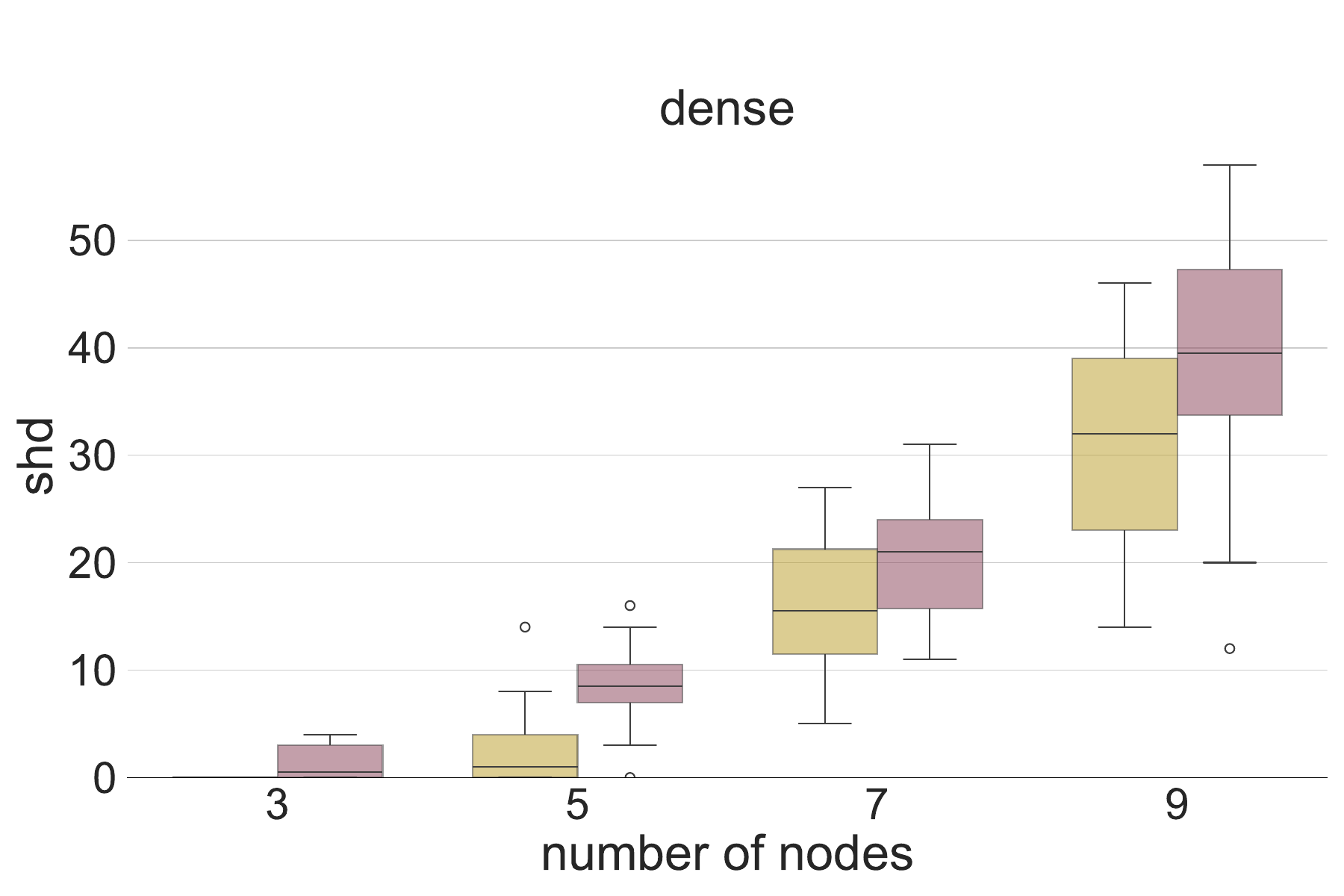}
%}
\caption{Results for the vanilla FCI algorithm with kernel independence test versus PAG output of AdaScore. We used non-additive causal mechanisms on sparse (left) and dense graphs (right) with different numbers of nodes. We report the SHD.}
\label{fig:fci_experiments}
\end{figure}

% Sparsity
\subsection{Dense graphs}\label{app:sparse_exp}
In this section, we present the experiments on dense Erd\"{o}s-Renyi graphs where each pair of nodes is connected by an edge with probability $0.3$. The results are illustrated in \cref{fig:experiments-dense}. For dense graphs, recovery results are similar to the sparse case, with AdaScore generally providing comparable performance to the other methods. 

\begin{figure}
\centering
    %\subfigure{%
    \includegraphics[width=0.55\textwidth]{main_text_plots/legend.pdf}\\
    %}
    \vspace{1.2em} 
    \subfigure[Fully observable model\label{fig:dense_observable}]{%
        \includegraphics[width=0.8\textwidth]{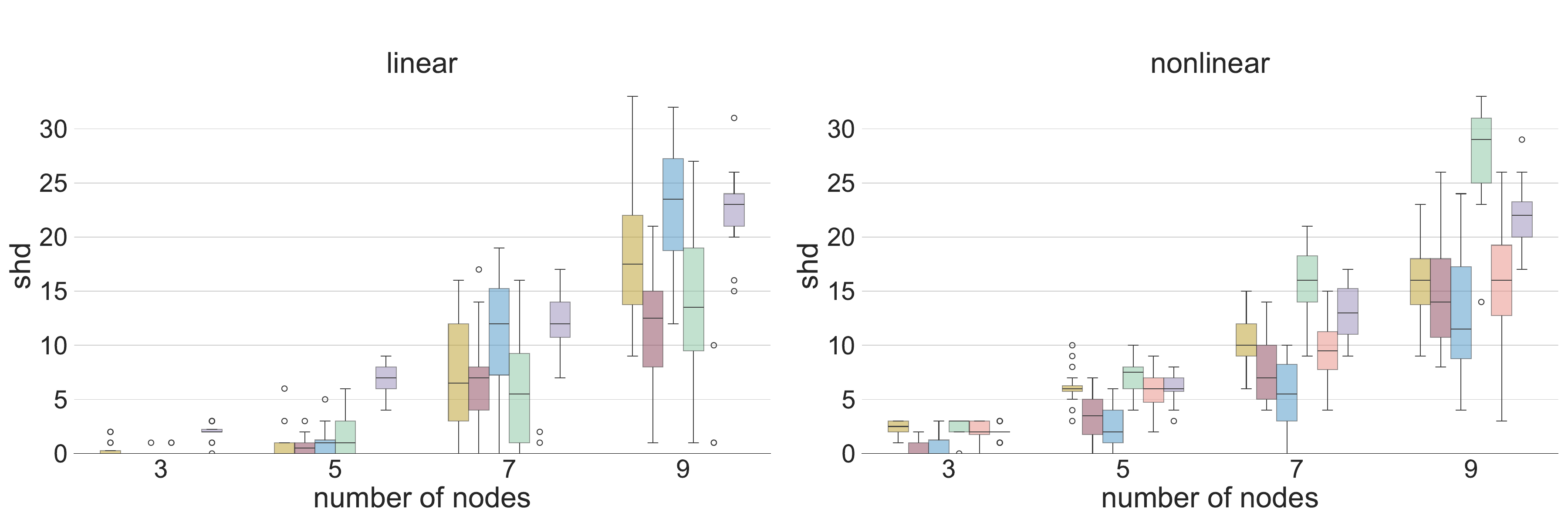}
    }
    \vspace{1.2em}
    \subfigure[Latent variables model\label{fig:dense_latent}]{%
        \includegraphics[width=0.8\textwidth]{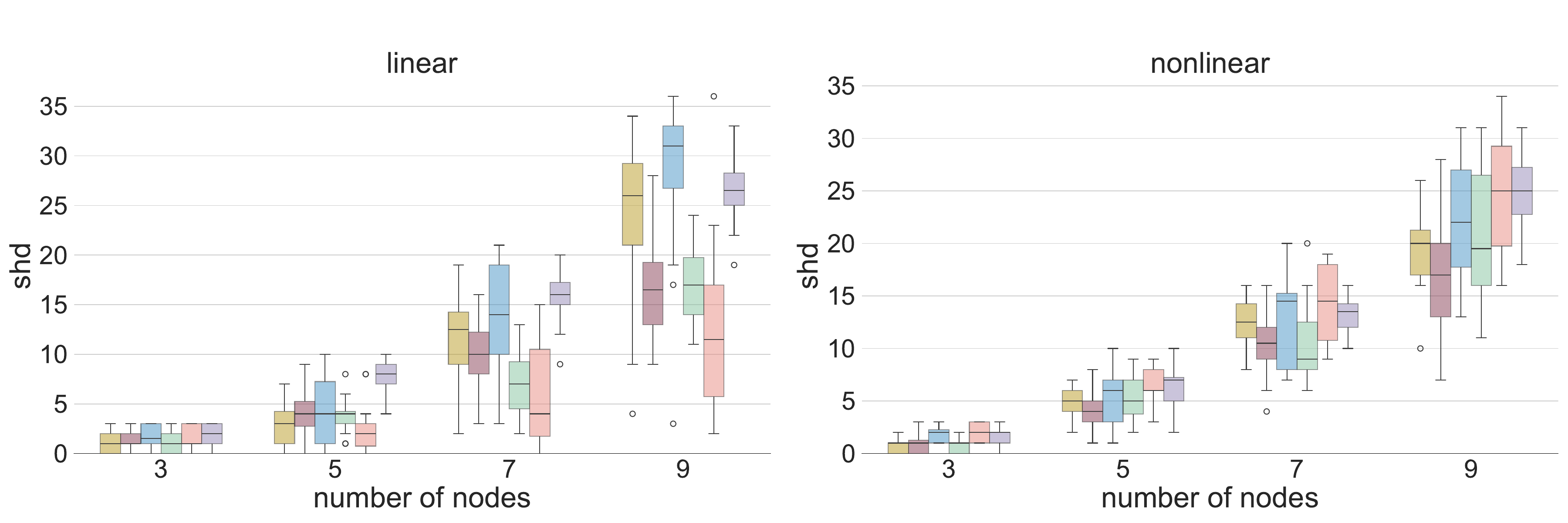}
    }
\caption{\looseness-1Empirical results on dense graphs with different numbers of nodes, on fully observable (no hidden variables) and latent variable models. We report the SHD accuracy (lower is better). %We note that Adascore is better than or comparable to other methods in the nonlinear settings, whereas its performance decreases on larger linear latent variable models.
}
\label{fig:experiments-dense}
\end{figure}

\subsection{$F_1$ scores}
The following plots show the $F_1$ score as an additional metric for the previously discussed experiments.
Since the $F_1$ score is only applicable to binary decisions, we calculate it with respect to the binary classification of whether there is an identifiable direct edge in the ground truth graph or not in \cref{fig:direct_edge_f1_dense,fig:direct_edge_f1_sparse}, or whether there is an edge that is not identifiable via Proposition \ref{prop:causal_dir_1} respectively in \cref{fig:bi_edge_f1}.
\begin{figure}[ht]
\centering
%\subfigure[]{%
    \includegraphics[width=0.55\textwidth]{main_text_plots/legend.pdf}\\
%}
\vspace{1.2em}
\subfigure[Fully observable model\label{fig:direct_edge_f1_dense_obs}]{%
    \includegraphics[width=0.8\textwidth]{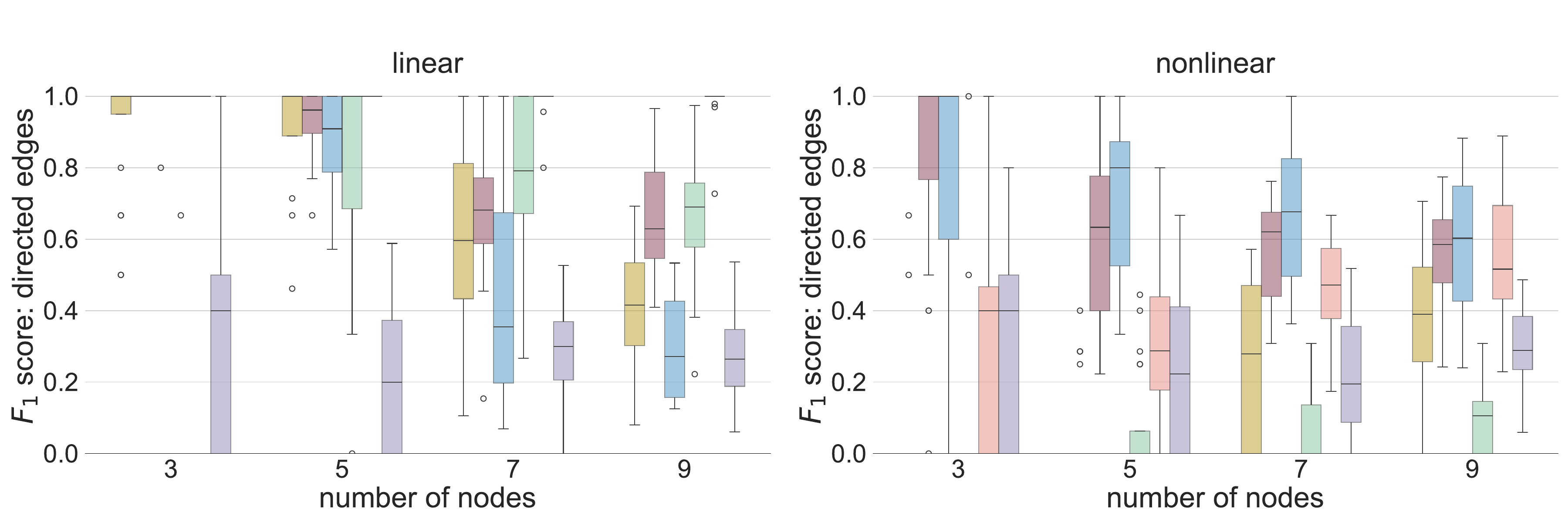}
}
\vspace{1.2em}
\subfigure[Latent variables model\label{fig:direct_edge_f1_dense_latent}]{%
    \includegraphics[width=0.8\textwidth]{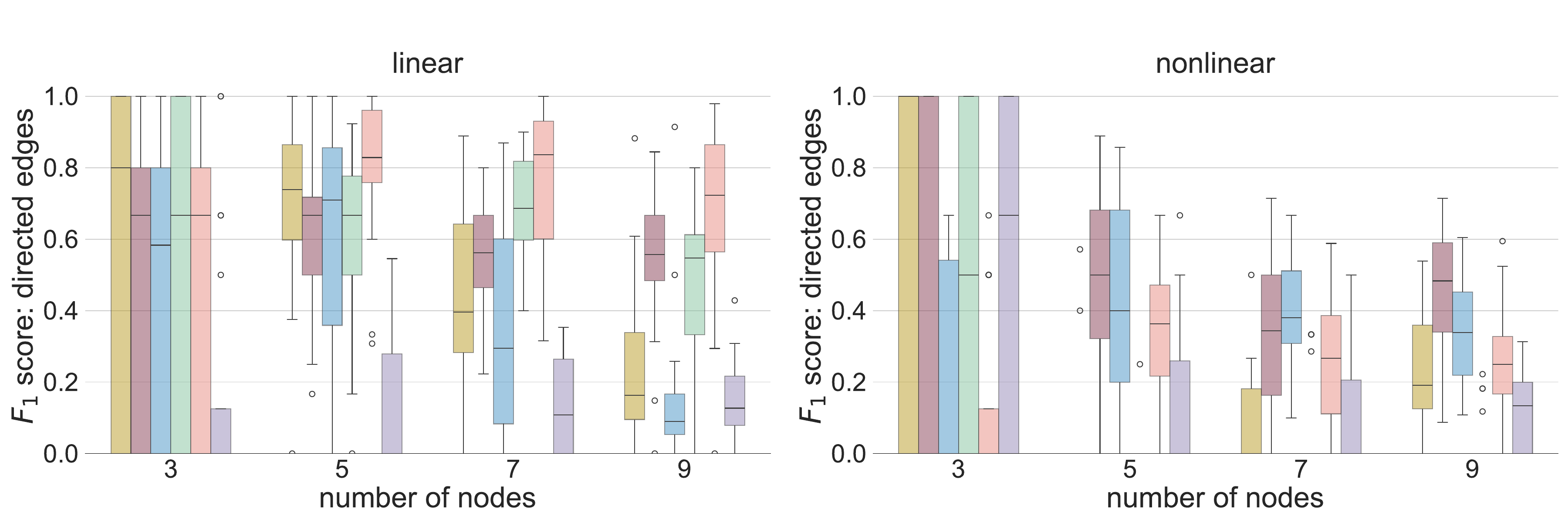}
}
\caption{Empirical results on dense graphs with different numbers of nodes, on fully observable (no hidden variables) and latent variable models. We report the $F_1$ score w.r.t. to the binary decision of whether there is an identifiable direct edge or not (the higher, the better).}
\label{fig:direct_edge_f1_dense}
\end{figure}

\begin{figure}[ht]
\centering
%\subfigure[]{%
    \includegraphics[width=0.55\textwidth]{main_text_plots/legend.pdf}\\
%}
\vspace{1.2em}
\subfigure[Fully observable model\label{fig:direct_edge_f1_sparse_obs}]{%
    \includegraphics[width=0.8\textwidth]{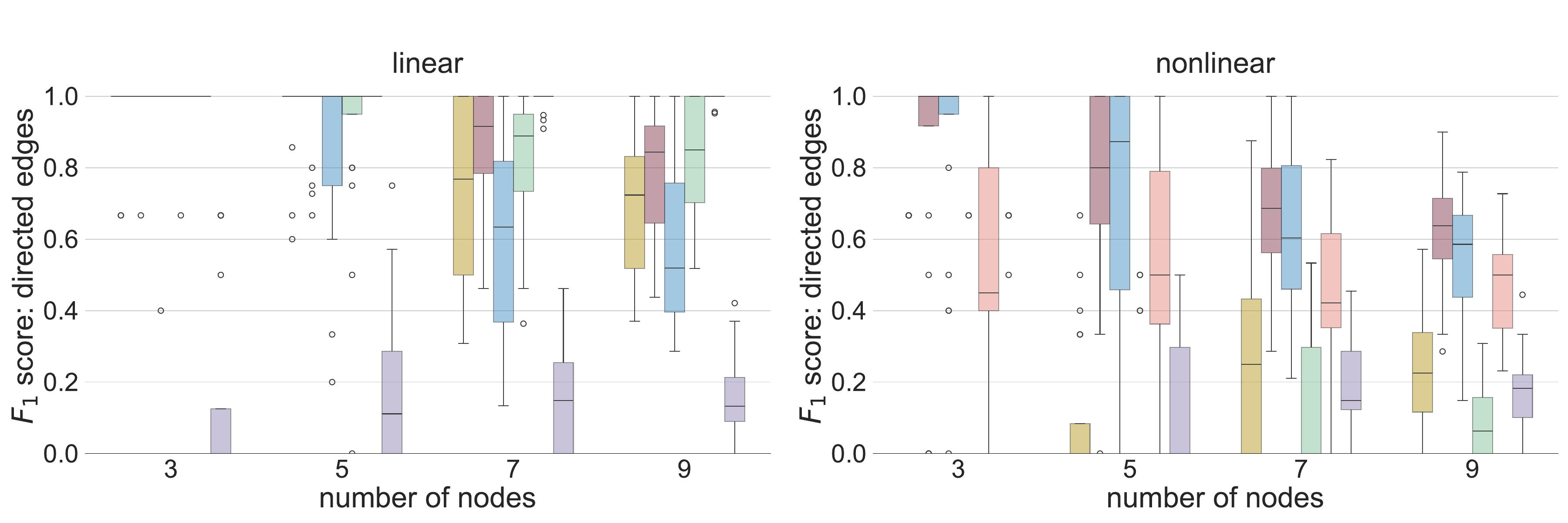}
}
\vspace{1.2em}
\subfigure[Latent variables model\label{fig:direct_edge_f1_sparse_latent}]{%
    \includegraphics[width=0.8\textwidth]{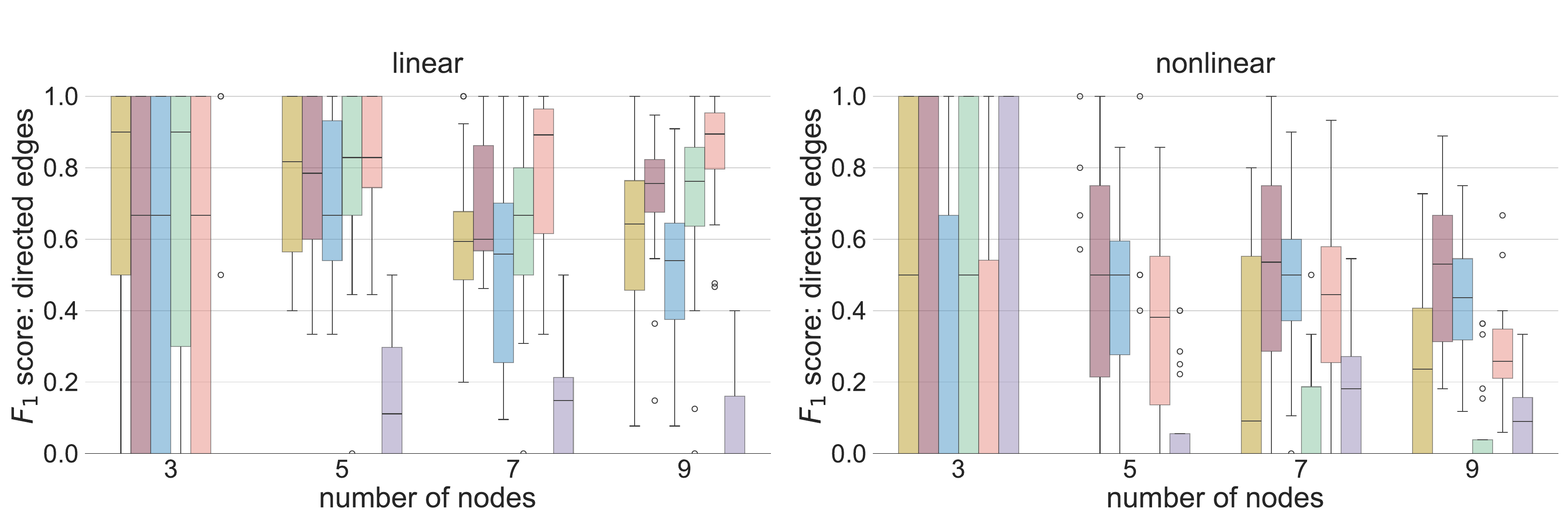}
}
\caption{Empirical results on sparse graphs with different numbers of nodes, on fully observable (no hidden variables) and latent variable models. We report the $F_1$ score w.r.t. to the binary decision of whether there is an identifiable direct edge or not (the higher, the better).}
\label{fig:direct_edge_f1_sparse}
\end{figure}

\begin{figure}[ht]
\centering
%\subfigure[]{%
    \includegraphics[width=0.55\textwidth]{main_text_plots/legend.pdf}\\
%}
\vspace{1.2em}
\subfigure[Dense model \label{fig:bi_edge_f1_dense}]{%
    \includegraphics[width=0.8\textwidth]{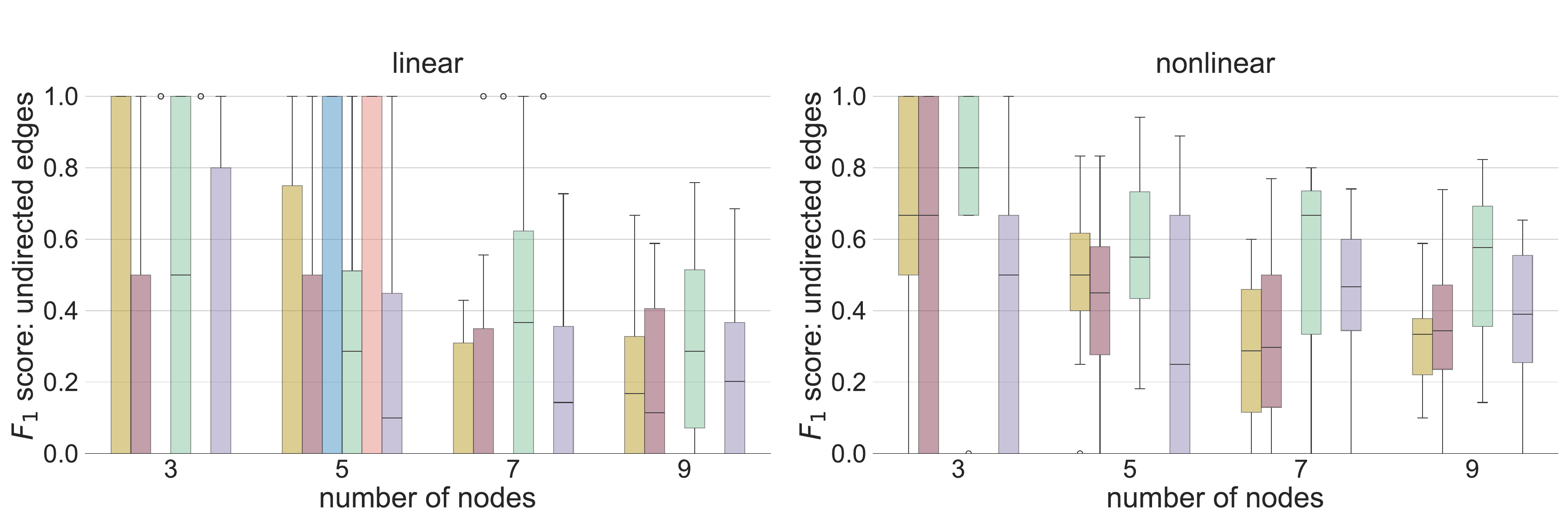}
}
\vspace{1.2em}
\subfigure[Sparse model\label{fig:bi_edge_f1_sparse_latent}]{%
    \includegraphics[width=0.8\textwidth]{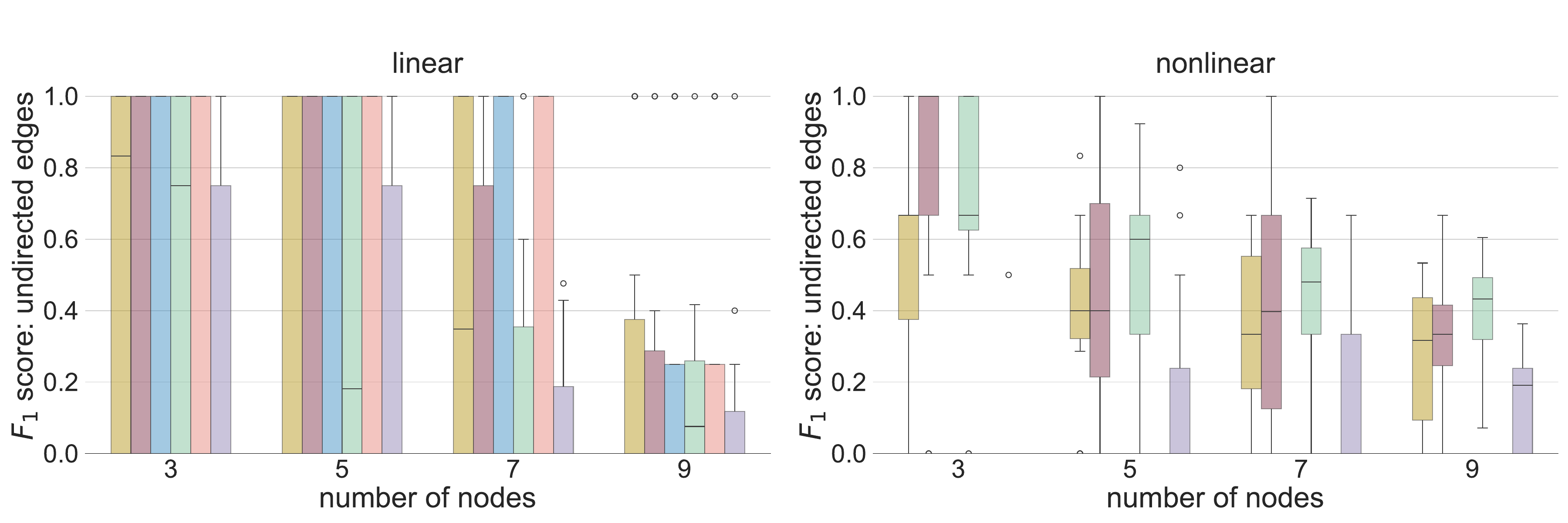}
}
\caption{Empirical results on sparse graphs with different numbers of nodes, on fully observable (no hidden variables) and latent variable models. We report the $F_1$ score w.r.t. to the binary decision of whether there is an unidentifiable edge or not (the higher, the better).}
\label{fig:bi_edge_f1}
\end{figure}

\subsection{Increasing number of samples}
In the following series of plots we demonstrate the scaling behaviour of our method w.r.t. to the number of samples.
\Cref{fig:experiments-samples-dense} shows results with edge probability 0.5 and \cref{fig:experiments-samples-sparse} with 0.3. All graphs contain seven observable nodes.
As before we observe that AdaScore performs comparably to other methods.
%E.g. in \cref{fig:sample_sparse_observable,fig:sample_dense_latent} we can see that the median error AdaScore improves with additional samples and in all plots we see that no other algorithm seems to gain an advantage over AdaScore with increasing sample size.
\begin{figure}[ht]
\centering
%\subfigure[]{%
    \includegraphics[width=0.55\textwidth]{main_text_plots/legend.pdf}\\
%}
\vspace{1.2em}
\subfigure[Fully observable model\label{fig:sample_sparse_observable}]{%
    \includegraphics[width=0.8\textwidth]{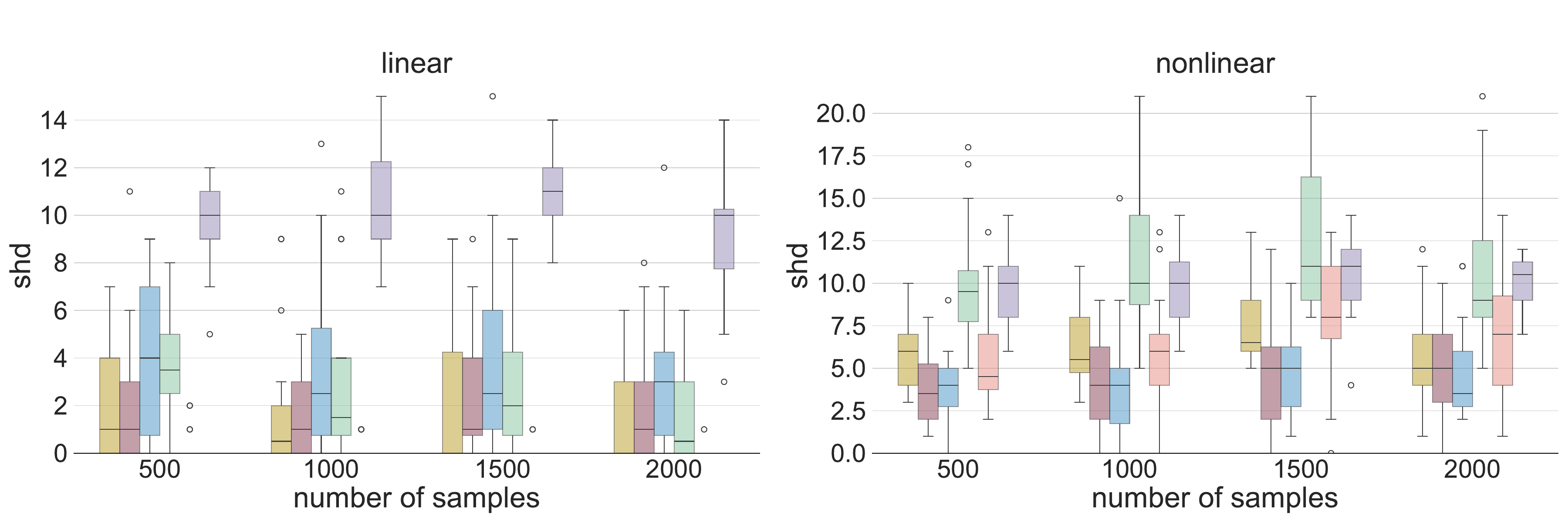}
}
\vspace{1.2em}
\subfigure[Latent variables model\label{fig:sample_sparse_latent}]{%
    \includegraphics[width=0.8\textwidth]{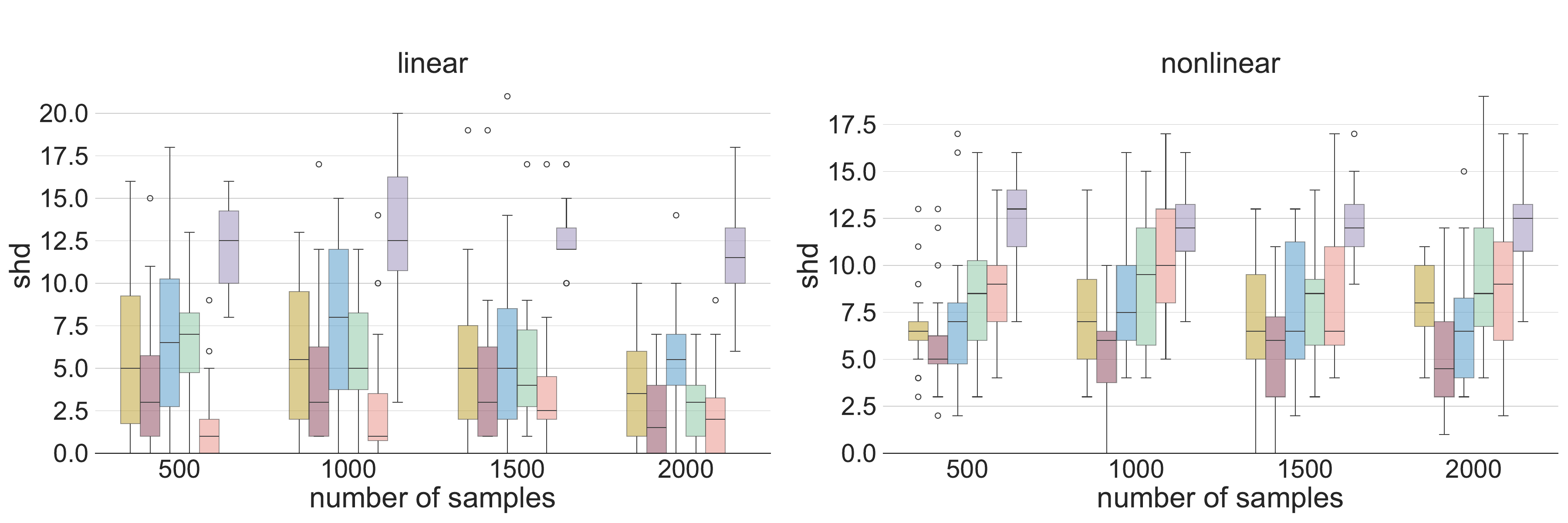}
}
\caption{Empirical results on sparse graphs with different numbers of samples and seven nodes, on fully observable (no hidden variables) and latent variable models. We report the SHD accuracy (the lower, the better).}
\label{fig:experiments-samples-sparse}
\end{figure}

\begin{figure}[ht]
\centering
%\subfigure[]{%
    \includegraphics[width=0.55\textwidth]{main_text_plots/legend.pdf}\\
%}
\vspace{1.2em}
\subfigure[Fully observable model\label{fig:sample_dense_observable}]{%
    \includegraphics[width=0.8\textwidth]{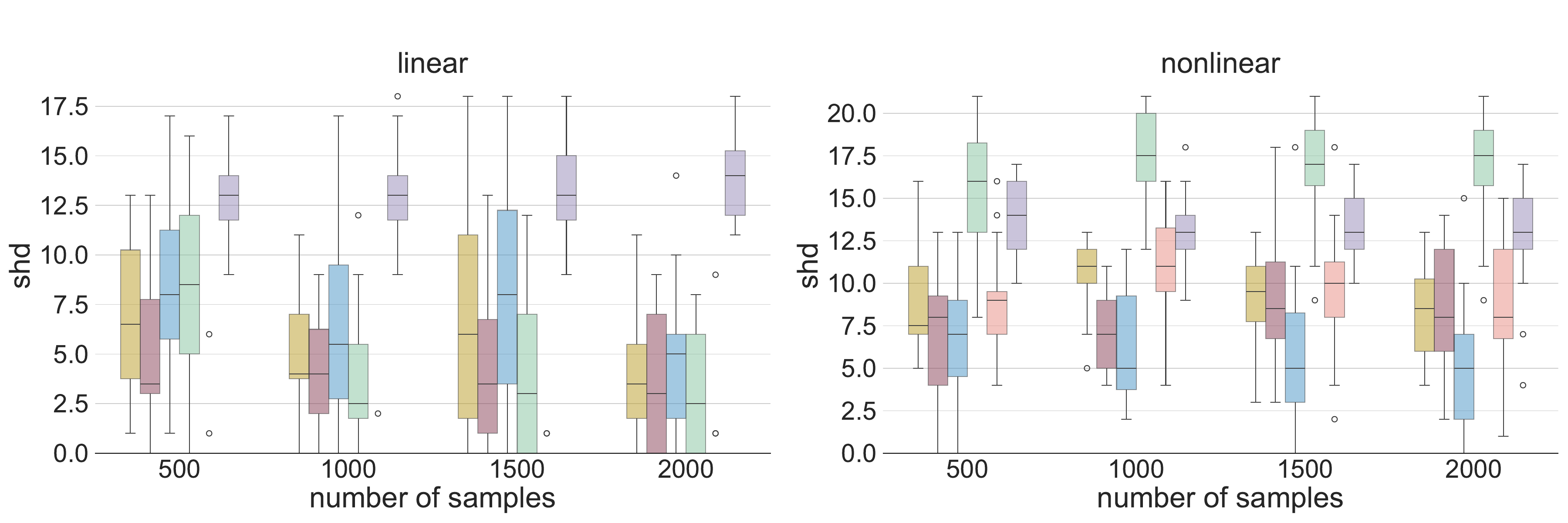}
}
\vspace{1.2em}
\subfigure[Latent variables model\label{fig:sample_dense_latent}]{%
    \includegraphics[width=0.8\textwidth]{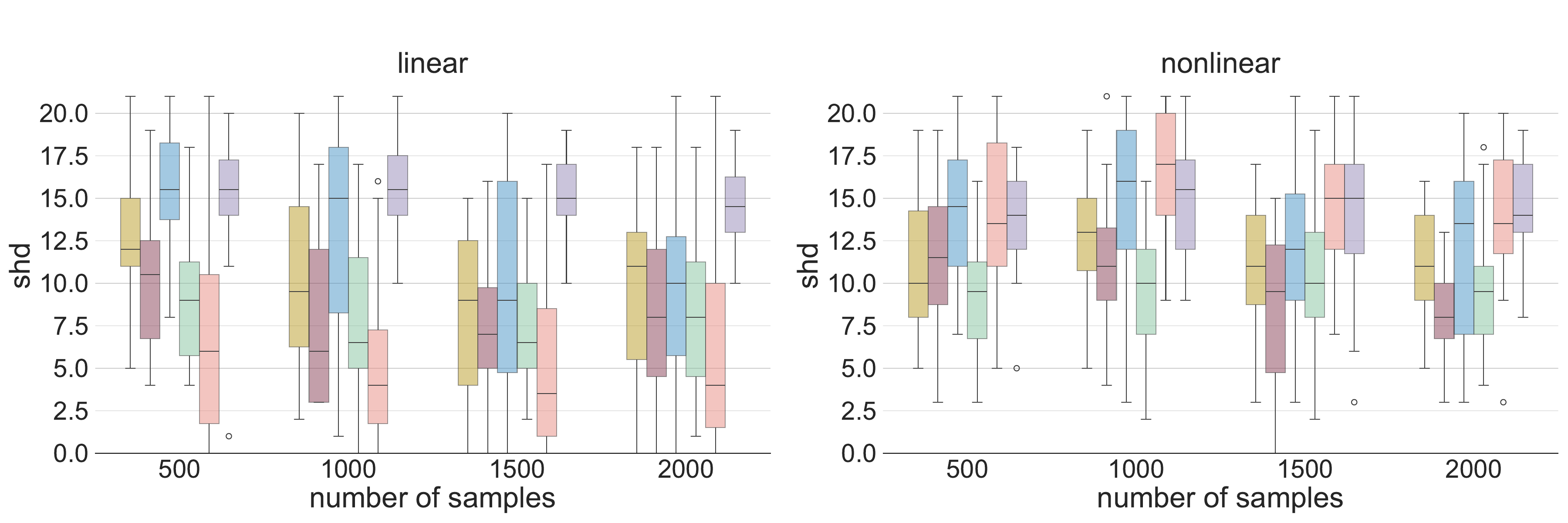}
}
\caption{Empirical results on dense graphs with different numbers of samples and seven nodes, on fully observable (no hidden variables) and latent variable models. We report the SHD accuracy (the lower, the better).}
\label{fig:experiments-samples-dense}
\end{figure}

\subsection{Runtimes}\label{app:experiments_time}
In Figures \ref{fig:time_dense} to \ref{fig:time_sparse_samples} we have plotted  the runtimes of the benchmarked methods in different settings.

\begin{figure}[ht]
\centering
%\subfigure[]{%
    \includegraphics[width=0.55\textwidth]{main_text_plots/legend.pdf}\\
%}
\vspace{1.2em}
\subfigure[Fully observable model\label{fig:time_observable_dense}]{%
    \includegraphics[width=0.8\textwidth]{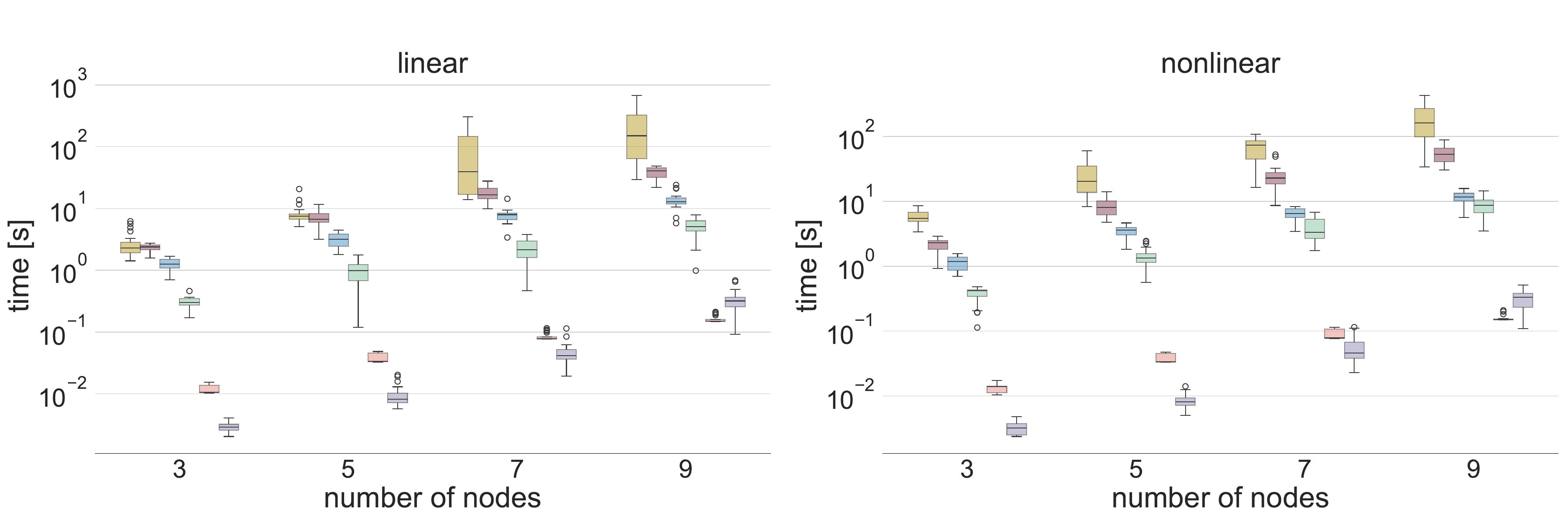}
}
\vspace{1.2em}
\subfigure[Latent variables model\label{fig:time_latent_dense}]{%
    \includegraphics[width=0.8\textwidth]{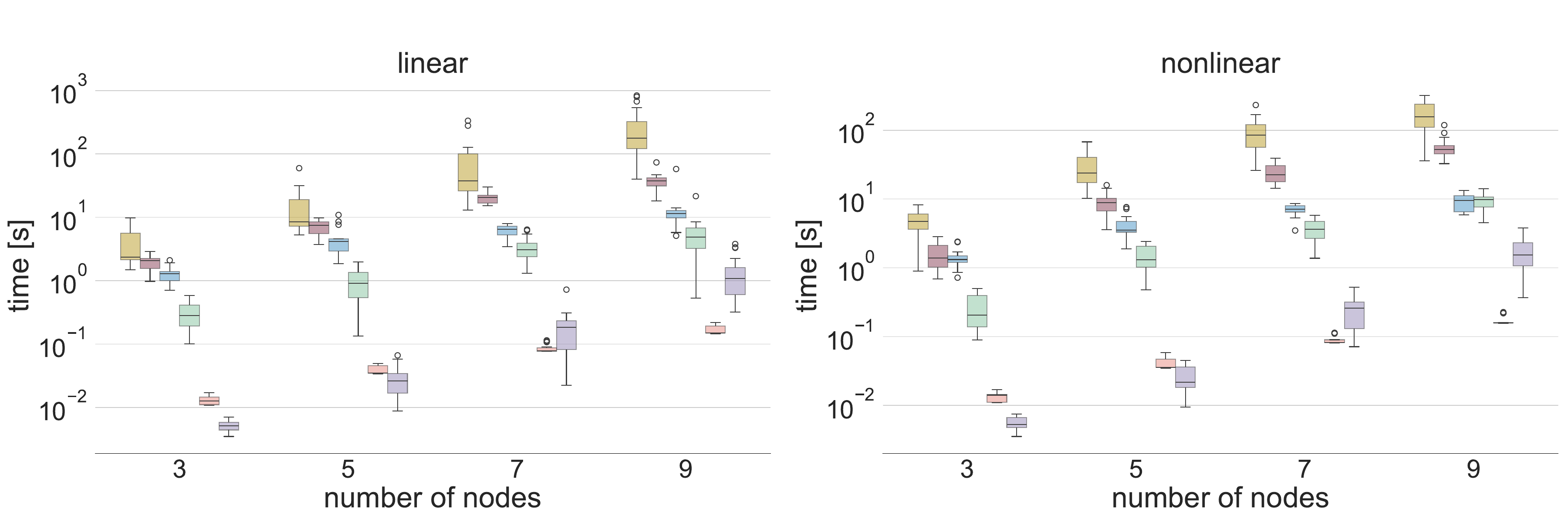}
}
\caption{Runtime in seconds on dense graphs with different numbers of nodes, on fully observable (no hidden variables) and latent variable models.}
\label{fig:time_dense}
\end{figure}

\begin{figure}[ht]
\centering
%\subfigure[]{%
    \includegraphics[width=0.55\textwidth]{main_text_plots/legend.pdf}\\
%}
\vspace{1.2em}
\subfigure[Fully observable model\label{fig:time_observable_sparse}]{%
    \includegraphics[width=0.8\textwidth]{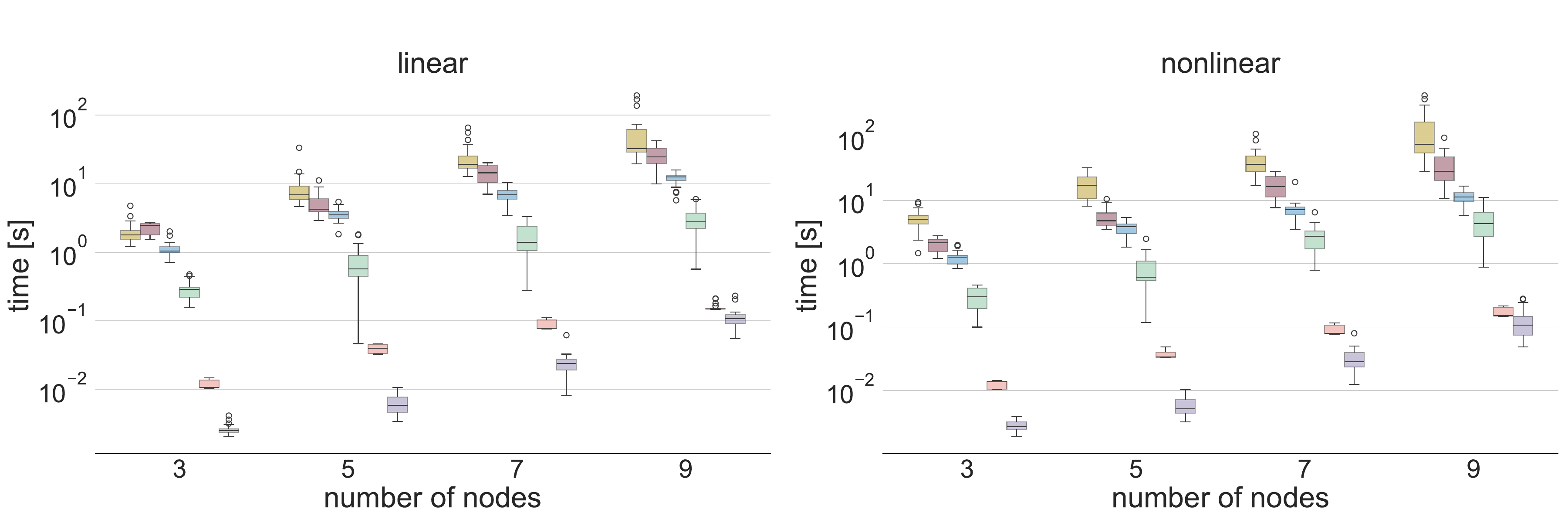}
}
\vspace{1.2em}
\subfigure[Latent variables model\label{fig:time_latent_sparse}]{%
    \includegraphics[width=0.8\textwidth]{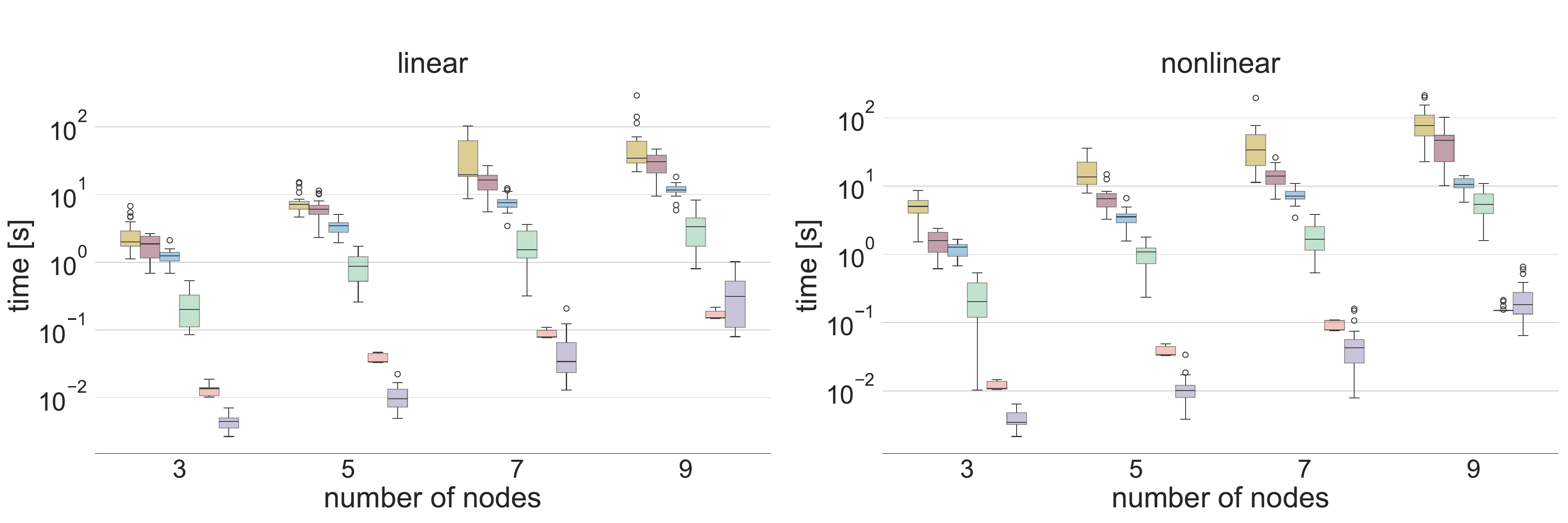}
}
\caption{Runtime in seconds on sparse graphs with different numbers of nodes, on fully observable (no hidden variables) and latent variable models.}
\label{fig:time_sparse}
\end{figure}

\begin{figure}[ht]
\centering
%\subfigure[]{%
    \includegraphics[width=0.55\textwidth]{main_text_plots/legend.pdf}\\
%}
\vspace{1.2em}
\subfigure[Fully observable model\label{fig:time_observable_dense_samples}]{%
    \includegraphics[width=0.8\textwidth]{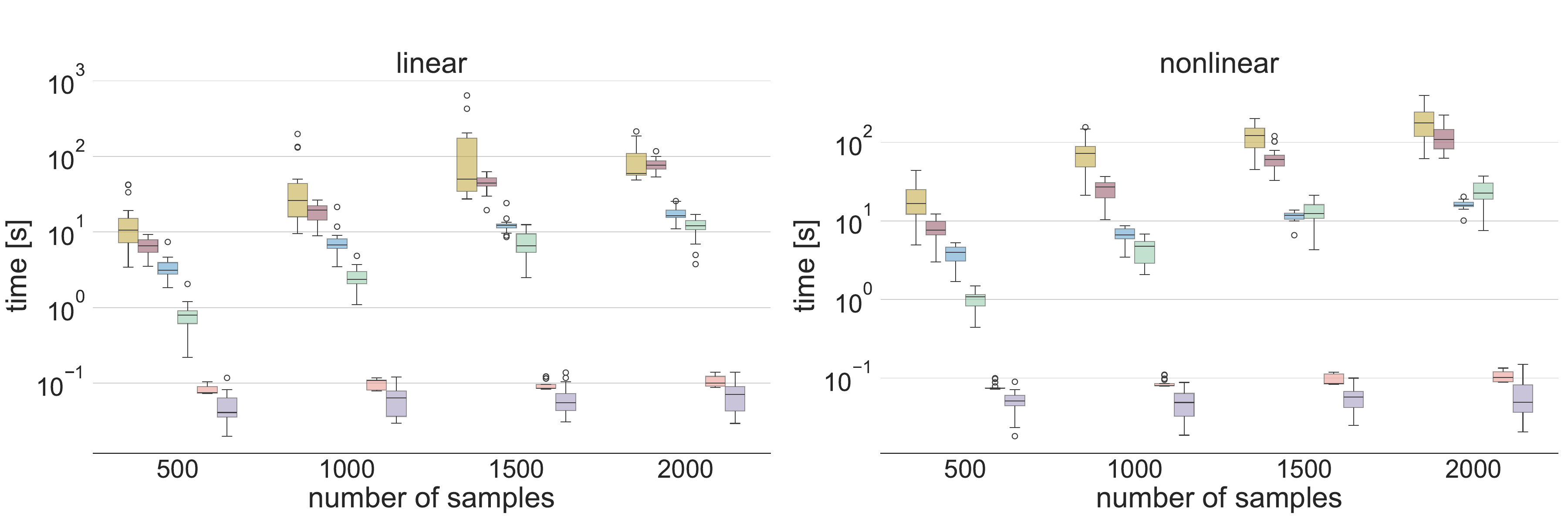}
}
\vspace{1.2em}
\subfigure[Latent variables model\label{fig:time_latent_dense_samples}]{%
    \includegraphics[width=0.8\textwidth]{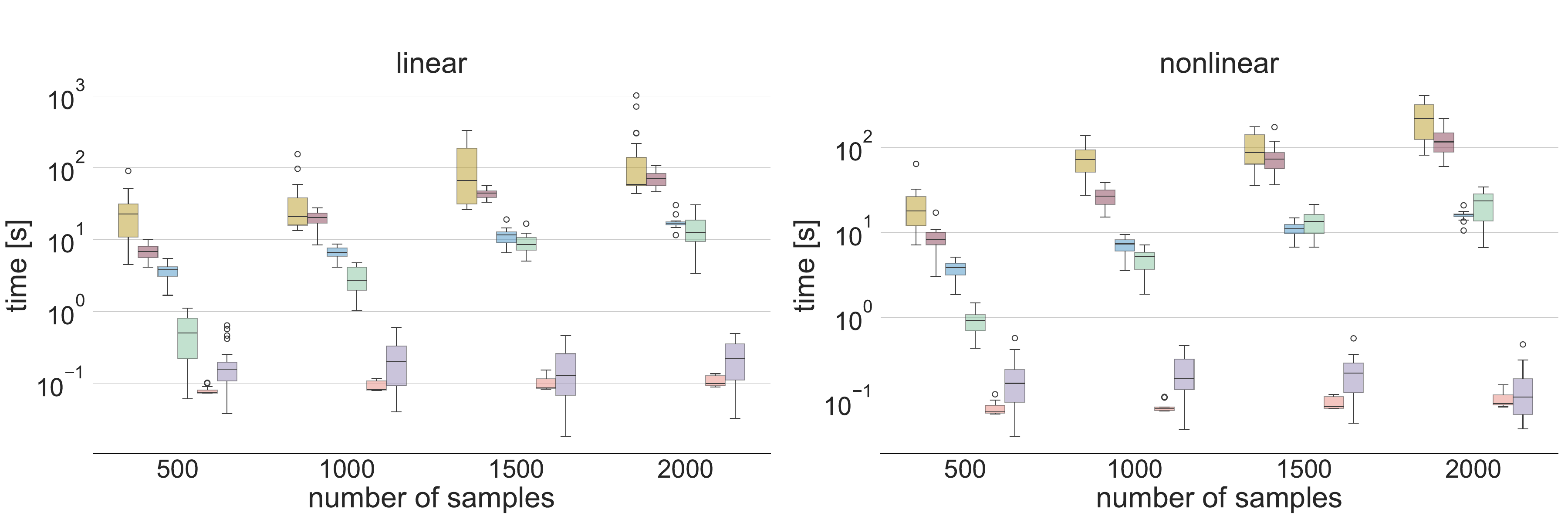}
}
\caption{Runtime in seconds on dense graphs with different numbers of samples, on fully observable (no hidden variables) and latent variable models.}
\label{fig:time_dense_samples}
\end{figure}

\begin{figure}[ht]
\centering
%\subfigure[]{%
    \includegraphics[width=0.55\textwidth]{main_text_plots/legend.pdf}\\
%}
\vspace{1.2em}
\subfigure[Fully observable model\label{fig:time_observable_sparse_samples}]{%
    \includegraphics[width=0.8\textwidth]{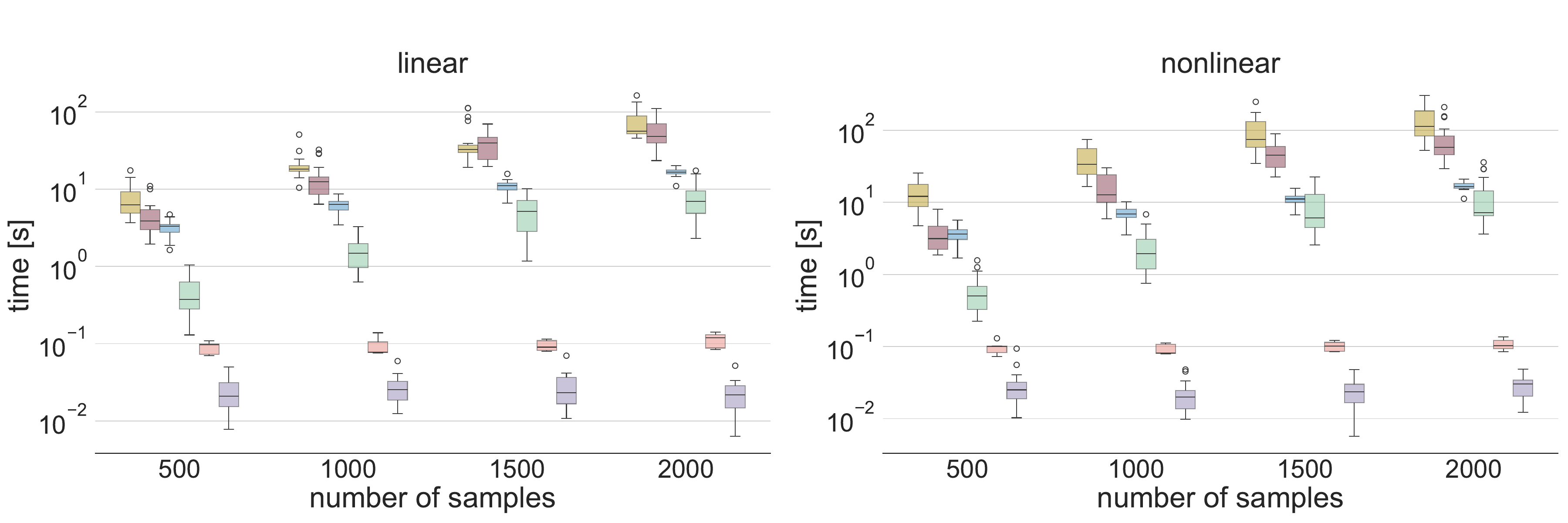}
}
\vspace{1.2em}
\subfigure[Latent variables model\label{fig:time_latent_sparse_samples}]{%
    \includegraphics[width=0.8\textwidth]{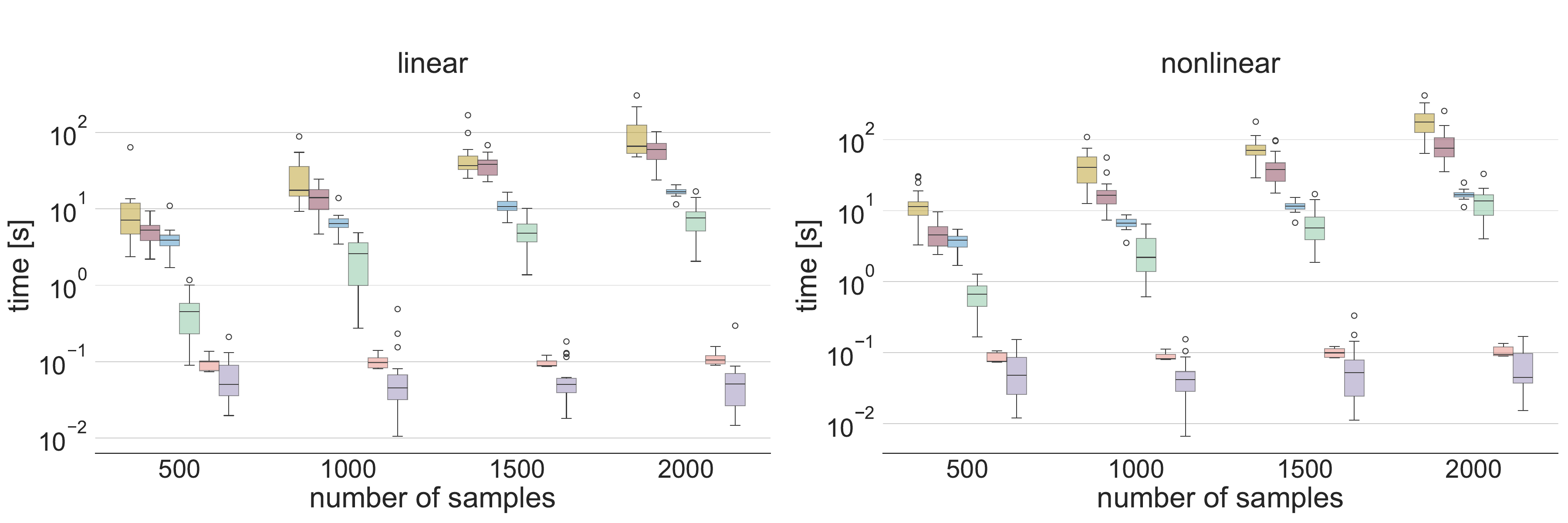}
}
\caption{Runtime in seconds on sparse graphs with different numbers of samples, on fully observable (no hidden variables) and latent variable models.}
\label{fig:time_sparse_samples}
\end{figure}

\subsection{Limitations}\label{app:limitations}
In this section, we remark the limitations of our empirical study. It is well known that causal discovery lacks meaningful, multivariate benchmark datasets with known ground truth. For this reason, it is common to rely on synthetically generated datasets. We believe that results on synthetic graphs should be taken with care, as there is no strong reason to believe that they should mirror the benchmarked algorithms' behaviors in real-world settings, where often there is no prior knowledge about the structural causal model underlying available observations.

\end{document}